\pdfoutput=1
\documentclass[letterpaper]{article}

\usepackage[caption=false]{subfig}
\usepackage{grffile}
\usepackage[title]{appendix}
\usepackage{algorithm, algcompatible, algpseudocode, algorithmicx}
\usepackage{amsmath, amsthm, amssymb, mathtools, cases}
\usepackage{setspace}
\usepackage[export]{adjustbox}
\usepackage{multirow}

\newtheorem{thm}{Theorem}
\newtheorem{lemma}{Lemma}
\newtheorem{remark}{Remark}

\newcommand{\eqnref}[1]{(\ref{#1})}

\makeatletter
\def\hlinewd#1{%
	\noalign{\ifnum0=`}\fi\hrule \@height #1 \futurelet
	\reserved@a\@xhline}
\newcommand{\hthickline}{\hlinewd{1.2pt}}

\makeatother

\algnewcommand\algorithmicforeach{\textbf{for each}}
\algdef{S}[FOR]{ForEach}[1]{\algorithmicforeach\ #1\ \algorithmicdo}

\renewcommand{\COMMENT}[2][.7\linewidth]{%
	\leavevmode\hfill\makebox[#1][l]{//~#2}}
\algnewcommand\RETURN{\State \textbf{return} }

\newcommand\CONDITION[2]%
{\begin{tabular}[t]{@{}l@{}l@{}}
		#1&#2
	\end{tabular}%
}
\algdef{SE}[WHILE]{While}{EndWhile}[1]%
{\algorithmicwhile\ \CONDITION{#1}{\ \algorithmicdo}}%
{\algorithmicend\ \algorithmicwhile}
\algdef{SE}[FOR]{For}{EndFor}[1]%
{\algorithmicfor\ \CONDITION{#1}{\ \algorithmicdo}}%
{\algorithmicend\ \algorithmicfor}
\algdef{S}[FOR]{ForAll}[1]%
{\algorithmicforall\ \CONDITION{#1}{\ \algorithmicdo}}
\algdef{SE}[REPEAT]{Repeat}{Until}{\algorithmicrepeat}[1]%
{\algorithmicuntil\ \CONDITION{#1}{}}
\algdef{SE}[IF]{If}{EndIf}[1]%
{\algorithmicif\ \CONDITION{#1}{\ \algorithmicthen}}%
{\algorithmicend\ \algorithmicif}%
\algdef{C}[IF]{IF}{ElsIf}[1]%
{\algorithmicelse\ \algorithmicif\ \CONDITION{#1}{\ \algorithmicthen}}

\algdef{SE}[DOWHILE]{Do}{doWhile}{\algorithmicdo}[1]{\algorithmicwhile\ #1}%
\newcommand\Algphase[1]{%
	\vspace*{-.7\baselineskip}\Statex\hspace*{\dimexpr-\algorithmicindent-2pt\relax}\rule{\textwidth}{0.4pt}%
	\Statex\hspace*{-\algorithmicindent}\textbf{#1}%
	\vspace*{-.7\baselineskip}\Statex\hspace*{\dimexpr-\algorithmicindent-2pt\relax}\rule{\textwidth}{0.4pt}%
}

\title{Sampling-Based Tour Generation of Arbitrarily Oriented Dubins Sensor Platforms}

\author{
	Doo-Hyun Cho\thanks{Ph.D. Candidate, Department of Aerospace Engineering.
		E-mail: dhcho@lics.kaist.ac.kr}\\
	{\normalsize\itshape
		KAIST, Daejeon, 34141, Republic of Korea}\\
	{
		Dae-Sung Jang\thanks{Computer Scientist, Intelligent Systems Division.
			E-mail: dsjang@kau.ac.kr}
	}\\
	{\normalsize\itshape
		Korea Aerospace University, 10540, Republic of Korea}\\
	Han-Lim Choi\thanks{Associate Professor, Department of Aerospace Engineering.
		E-mail: hanlimc@kaist.ac.kr}\\
	{\normalsize\itshape
		KAIST, Daejeon, 34141, Republic of Korea}
}

\begin{document}
	
	\maketitle
	
	\begin{abstract}
		This paper describes a formulation and develops a novel procedure for a fleet of unmanned aerial vehicles (UAVs) from the perspective of remotely executable tasks.
		In a complex mission environment, the characteristics of vehicles can be different in terms of sensing capability, range, direction, or the motion constraints.
		The purpose of this paper is to find a set of paths that minimizes the sum of costs while every task region is visited exactly once under certain reasonable assumptions.
		The heterogeneous multi-UAV path planning problem is formulated as a generalized, heterogeneous, multiple depot traveling salesmen problem (GHMDATSP), which is a variant of the traveling salesman problem.
		The proposed transformation procedure changes an instance of the GHMDATSP into a format of an Asymmetric, Traveling Salesman Problem (ATSP) to obtain tours for which the total cost of a fleet of vehicles is minimized.
		The instance of the ATSP is solved using the Lin-Kernighan-Helsgaun heuristic, and the result is inversely transformed to the GHMDATSP-formatted instance to obtain a set of tours.
		An additional local optimization based path refinement process helps obtain a high-quality solution.
		Numerical experiments investigate and confirm for the validity and applicability of the proposed procedure.
	\end{abstract}
	
	\newpage
	\section*{Nomenclature}\label{nomenclature}
	\noindent
	\begin{table*}[htbp]
		\centering
		\begin{adjustbox}{max width=\textwidth}
			\begin{tabular}{@{}lcl@{}}
				$ c_{s,s'}, c(V^{k}_{t_1,i_1},V^{k}_{t_2,i_2}) $ &=& Cost of an edge from $ s $ to $ s' $ and from $ V^{k}_{t_1,i_1} $ to $ V^{k}_{t_2,i_2} $, respectively.\\
				$ \textrm{Cost}^{k} $ &=& Total tour cost of vehicle $ k $.\\
				$ D $ &=& Set of depots and terminals. $ D = \{(n+1)^1, \cdots, (n+1)^m, (n+2)^1, \cdots, (n+2)^m\} $.\\
				$ E $ &=& Set of all directed edges.\\
				$ g $ &=& Gravitational acceleration.\\
				$ i, k $ &=& Index of a sample node and a vehicle, respectively.\\
				$ K $ &=& Set of vehicles.\\
				$ l_{\textrm{max}}^{k} $ &=& Maximum load factor of vehicle $ k $.\\
				$ m,n $ &=& Number of vehicles and tasks, respectively.\\
				$ (n+1)^k,(n+2)^k $ &=& Vehicle $ k $'s depot and terminal cluster, respectively.\\
				$ N_{t} $ &=& Neighborhood of task $ t $.\\
				$ r_{min}, r^{k} $ &=& Minimum turning radius, for vehicle $ k $.\\
				$ r_{sen} $ &=& Radius of sensor footprint.\\
				$ s $ &=& Sample node.\\
				$ S^{\textrm{NIN}}_{t} $ &=& Set of sample nodes necessarily intersecting task $ t $.\\
				$ t $ &=& Index of a task.\\
				$ T $ &=& Set of tasks. $ T = \{1, \cdots, n\} $.\\
				$ T^{\textrm{NIN}}_{s} $ &=& Set of tasks necessarily intersected by sample node $ s $.\\
				$ u^k $ &=& Control input to change $ \theta_k $ of vehicle $ k $.\\
				$ v^k $ &=& Speed of vehicle $ k $.\\
				$ V $ &=& Set of all sample nodes.\\
				$ V^k, V_t, V^{k}_{t} $ &=& Clusters. Set of sample nodes in vehicle $ k $, in task $ t $, and in vehicle $ k $ and task $ t $, respectively. \\
				$ V^{k}_{t,i} $ &=& $ i^{\textrm{th}} $ sample node in cluster $ V^{k}_{t}, \quad V^{k}_{t,i} \in V^{k}_{t} \quad \forall k\in K, t\in \{T\cup D\} $.\\
				$ V^{NIN}(s), V^{NIN}(s)[p] $ &=& Set of necessarily intersected virtual nodes by sample node $ s $ and its $ p^{th} $ element.\\
				$ x_{s,s'} $ &=& Indicator variable associated with edge $ (s,s') $ that the edge is included in the solution.\\
				$ x^k, y^k $ &=& Coordinates of the position of vehicle $ k $.\\
				$ \dot{x^k}, \dot{y^k} $ &=& Coordinates' derivative of vehicle $ k $.\\
				$ \gamma^k $ &=& Normalization constant of vehicle $ k $.\\
				$ \theta^k, \dot{\theta^k} $ &=& Heading angle and its derivative of vehicle $ k $.
			\end{tabular}
		\end{adjustbox}
	\end{table*}%

	\section{Introduction}\label{sec:intro}
	
	In the field of surveillance and reconnaissance, the ability of unmanned aerial vehicles (UAVs) to perform tasks has recently been greatly enhanced through mass production and miniaturization of sensor systems and improvement of the accuracy of data obtained by measurement\cite{colomina2014unmanned, pajares2015overview}.
	Particularly, as the technologies related to the UAV, such as computing power and battery capacity, have rapidly developed \cite{valavanis2014handbook}, the frequency of UAVs used in practical situations has increased more than ever before.
	Despite the technological advances, however, the limited dynamics and maximum operating distance of a UAV are still a major constraint on mission performance \cite{cai2014survey}.
	Much research has been done in the past decades on path planning of UAVs to efficiently perform the given tasks within the given constraints \cite{cho2016informative, Kim2016cubature, Choi2015informative, ceccarelli2007micro, bortoff2000path}.
	
	In the traditional UAV maneuvering scheme when a human operator designates a specific location within a mission area, the operator moves it to a designated location according to the given control law.
	In terms of the decision-making hierarchy \cite{cummings2007automation}, in the conventional methodology, the mission scheduling of when and where to go is done by a human operator.
	Automation of the decision making is a major research effort, and one of the representative methods is to create a sequence for how to visit the area according to the given mission purpose.
	For fixed-wing UAV path planning, sequencing problems of visiting task location have been modeled as the traveling salesman problem (TSP) for a Dubins vehicle \cite{savla2008traveling, le2012dubins}.
	
	In this study, we assume that a sensor or communication system mounted on a vehicle can perform observations, measurements, or communications for a given mission area at a distance from a given location (e.g., optical image measurement, radar/lidar applications, etc.) \cite{yuan2007optimal, liu2013path}.
	In particular, in the area of the sensor network and telecommunication, a problem with the above assumption has been seen as a \textit{data mule} problem \cite{shah2003data, yuan2007optimal, liu2013path, chang2015artificial}.
	The impact of the above assumption in constructing the problem of path planning is quite significant because it can be interpreted as accomplishing the task by simply passing near the given task, not visiting it exactly \cite{herwitz2004imaging}.
	For determining the order of locations to visit, a problem can be modeled using a variant of the TSP called the Generalized TSP (GTSP) or TSP with Neighborhoods (TSPN) \cite{srivastava1969generalized, arkin1994approximation, dumitrescu2001approximation, de2005tsp, behdani2014integer}.
	In the GTSP, clusters are generated for each task to be visited, where a cluster is a set of sample nodes in a task-capable area.
	When the created path visits one or more of the nodes in the cluster for every given task, the mission is configured as completed.	
	However, the GTSP-based formulation does not include the dynamics of a given vehicle, in other words, point-to-point visits in an instance are approximated by straight line segments. 
	Since the maximum flight path curvature of a fixed-wing UAV is constrained at a given speed, its path planning problem has been often modeled as the Dubins TSPN (DTSPN) \cite{isaacs2013dubins}.
	
	The problem described in this paper assumes that multiple aerial vehicles perform the entire mission to visit around the designated locations instead of directly visiting the points (Generalized).
	In addition, the moving costs in different directions between two locations and the costs of the same task location pair for each vehicle can be different (Asymmetric, Heterogeneous), and the initial position of a vehicle is called a depot which can be located in a different place from the others (Multiple Depot).
	Instead of approaching the problem in a continuous domain, we sample multiple feasible states for each task-vehicle pair to move the problem into the discrete domain.
	Then the generated graph of the problem falls into the class of sampling-based roadmap \cite{choset2005principles, lavalle2006planning, obermeyer2012sampling}, and there are two main advantages; the complexity of the problem is reduced, and the state-of-the-art solvers or heuristics for the sampling-based roadmap can be used to solve the problem in interest.
	Therefore, the problem addressed in this paper can be simplified using a variant of TSP, which is called the Generalized Heterogeneous Multiple Depot Asymmetric TSP (GHMDATSP).
	
	Among similar problems addressed in the literature, the Generalized Multiple Depot TSP (GMDTSP) \cite{sundar2016generalized} is the most similar to the GHMDATSP but it still easier due to fewer constraints.
	In addition, the transformation method proposed in \cite{oberlin2009transformation} converts a graph of the Heterogeneous Multiple Depot TSP (HMDTSP) into the form of the ATSP, but the HMDTSP does not address the problem of handling the remote tasks because the vehicle must pass through the designated locations with a certain heading.
	One of the main contributions of this paper is proposing a novel procedure to overcome the above issues and obtain a high-quality solution using the GHMDATSP.	
	The procedure can be divided into two major parts: i) solving the GHMDATSP, and ii) the path refinement process.
	The first part belongs to the category of the transformation method, and it is mainly based on the Noon-Bean transformation \cite{noon1993efficient} and Oberlin's work \cite{oberlin2010today}.
	The problem that is transformed into the form of the ATSP is calculated by the famous heuristic technique called the Lin-Kernighan-Helsgaun heuristic (LKH) \cite{helsgaun2000effective} among several heuristics that solve the ATSP problem. 
	Then the inverse transformation helps obtain the solution in the form of the GHMDATSP.
	In the second part, the solution from the above part is refined by local optimization.	
	While preliminary ideas and results were addressed in our previous research \cite{jang2016optimal, cho2018heterogeneous}, this paper presents the mixed-integer linear programming formulation and theoretical results about the relationship between the solutions in the form of the ATSP and the GHMDATSP, respectively.
	In addition, this paper extends the concept called the `Necessarily Intersecting Neighborhood' \cite{jang2016optimal} from the omni-directional sensor to more general situations where the sensor has directionality, and utilizes it for the problem of interest.
	
	The rest of the paper is organized as follows.
	In Section \ref{sec:prob_form}, the notations are introduced and the problem is mathematically formulated as a mixed-integer linear program.
	Detailed explanations on the necessarily intersecting node and region are provided in Section \ref{sec:nin}.
	A procedure of obtaining a solution of the multi-vehicle path planning problem is described in Section \ref{sec:proc}.
	Computational results are provided and discussed in Section \ref{sec:sim}.
	Finally, Section \ref{sec:conc} discusses the conclusion of this study.

	\section{Problem Formulation}\label{sec:prob_form}

	In the case of a fixed-wing type UAV, each vehicle follows a curvature constrained Dubins vehicle dynamics.
	The mathematical model of Dubins vehicle dynamics is shown in the following set of equations \eqnref{eq:dubins}.
	\begin{subequations}\label{eq:dubins}
		\begin{align}
		\label{eq:dubins1} \dot{x}^k \quad = \quad & v^k\cos\theta^k \\
		\label{eq:dubins2} \dot{y}^k \quad = \quad & v^k\sin\theta^k \\
		\label{eq:dubins3} \dot{\theta}^k \quad = \quad & \dfrac{\gamma^k}{v^k}u^k
		\end{align}
	\end{subequations}	
	
	For every parameter, the superscript $ k \in K$ $ (K=\{1, 2, \cdots, m\}) $ denotes the index of a vehicle where $ m $ is the number of vehicles.
	$ x^k $ and $ y^k $ designate the coordinates of the position of vehicle $ k $;
	$ v^k $ is the speed of vehicle $ k $ which is fixed for the duration of the entire simulation scenarios;
	and $ \theta^k $ is the heading angle of the vehicle $ k $.
	Each vehicle has a control input $ u^k $ to change $ \theta^k $, and $ \gamma^k $ is a normalization constant that makes the range of $ u^k $ as $ [-1, 1] $.
	The value of $ u^k $ in this paper is set as $ -1 $, $ 0 $, or $ 1 $ to make the length of the Dubins path between an arbitrary sample node pair as small as possible.
	
	A graph of the GHMDATSP is defined as a set of sample nodes belonging to a set of tasks $ T = \{1,2,\cdots,n\} $ and a set of the depot and terminal nodes (the initial and final locations of vehicles) $ D = \{(n+1)^1,\cdots,(n+1)^m,(n+2)^1,\cdots,(n+2)^m\} $, along with the edges connecting the sample nodes.
	In $ D $, $ n+1 $ denotes a depot and $ n+2 $ is a terminal.
	Basically the TSP problems with depots assume that vehicles depart from and finally returns to their given depot.
	In this study, however, vehicles finally return to one of the terminal nodes which is different from the sample nodes in depot clusters; this makes it possible to arbitrarily set the end point of each vehicle.	
	The cardinality of $ T $ and $ D $ is denoted as $ |T|=n $ and $ |D|=2m $, and it is assumed that $ n \geq 1 $.
	We let each vehicle have a separate depot.
	If multiple vehicles are located in a single depot, then the positions of depots for relevant vehicles are repeated in the same place.
	
	Let $ V^k_t $ be a set of sample nodes for vehicle $ k $ to visit task $ t $, and $ G=(V,E) $ be a directed graph where $ V = \bigcup_{k\in K, t\in T\cup D} V^k_t $ is the union of sample node sets of the tasks, depot, and terminal for each vehicle.
	We define the $ i^{\textrm{th}} $ sample node in cluster $ V^k_t $ as $ V_{t,i}^{k} \in V^k_t $, or simply $ s \in V^k_t $ for simplicity if it does not compromise the understanding of the context.
	Also for a simplicity, a set of all sample nodes relevant to a vehicle $ k $ is $ V^k = \bigcup_{t\in T\cup D}  V^k_t $, and a set of all sample nodes relevant to a task is $ V_t = \bigcup_{k\in K}  V^k_t $.
	For any $ s \in V $ or $ S \subset V $, $ t(s) $ and $ t(S) $ are a task or a set of tasks to which the sample node $ s $ or the set $ S $ belong, respectively.	
	For any sample nodes $ s, s' \in V $, $ x_{s,s'} $ is defined as a binary decision variable, and the value equals 1 if the tour has an edge from $ s $ to $ s' $, and $ y_s $ is a binary decision variable which equals to 1 when the tour visits a sample node $ s $.
	The cost for each edge in edge set $ E $ of graph $ G $ are calculated before obtaining the solution of the GHMDATSP, let $ c_{s,s'} $ be the cost of from sample node $ s $ to $ s' $.
	The cost can be computed using the parameters of the vehicle based on the Dubins model.
	
	In this study, we use the objective function as Eq. \eqref{eq:obj} which minimizes the total cost (or total required time), which is one of the most commonly used in the aerospace and robotics fields.
	We associate with each edge $ (s,s') $ a variable $ x_{s,s'} $ which indicates that the edge is included in the solution.
	Inspired by the models for the variants of the generalized TSP in \cite{sundar2017algorithms, cho2018heterogeneous},the following is the formulation of the GHMDATSP using a mixed integer linear programming:
	\begin{alignat}{3}
		\MoveEqLeft[3] \text{Minimize: }\notag\\
		& \qquad\qquad\qquad\qquad\quad \label{eq:obj} \sum_{k \in K} \left(\sum_{s,s'\in V^k}c_{s,s'}\cdot x_{s,s'}\right) \\[1ex]
		\MoveEqLeft[3] \text{subject to}\notag\\
		& \label{eq:vtassign2} \sum_{s \in V_t} y_s = 1  & \forall t \in T \\
		& \label{eq:vtassign3} \sum_{s \in V^k_t} y_s = 1    &	\forall k \in K, t \in D \\
		& \label{eq:degree1} \sum_{s' \in V^k_{T\cup D\setminus \{t(s)\} }} x_{s',s} = y_s \qquad \qquad & \forall k \in K, s \in V^k \\
		& \label{eq:degree2} \sum_{s' \in V^k_{T\cup D\setminus \{t(s)\} }} x_{s,s'} = y_s \qquad \qquad & \forall k \in K, s \in V^k \\
		& \label{eq:subelim} \sum_{s' \in V\setminus S} x_{s',s} + x_{s,s'} \geq 2y_s \qquad & \forall S\subseteq \bigcup_{t\in T}V_t, s\in S
	\end{alignat}
	Constraints \eqref{eq:vtassign2} ensure that each task is visited at one of the nodes in its cluster.
	Constraints \eqref{eq:vtassign3} imply that one of the sample nodes in the depot or the terminal cluster for each vehicle has to be chosen.
	Equations in \eqref{eq:degree1} and \eqref{eq:degree2} are the degree constraints; if sample node $ s $ of vehicle $ k $ is chosen to be visited and the value of $ y_s $ is 1, the in-degree and out-degree of the node are 1, respectively.
	In other words, one of the edges towards sample node $ s $ and one of the outward edges from node $ s $ should be selected.
	Constraints in \eqref{eq:subelim} are the subtour elimination constraints; they prevent the generation of subtours in the solution for each vehicle.

	\section{Necessarily Intersecting Node and Region}\label{sec:nin}
	
	To solve the path planning problem, suppose a sampling-based roadmap is created in the form of the GTSP or its variants.
	The GTSP or GTSP-like graph can be converted to the form of ATSP using Noon-Bean transformation or other methods similar to it, and then exact or high-performance heuristic TSP solvers can be used to solve the instance.
	However, the generated roadmap basically discretizes the path planning problem which is in the continuous domain, and the performance of the above approach might be worse when the tasks are located quite close to each other.
	To relieve this issue, a concept called the \textit{Intersecting Node} and the \textit{Necessarily Intersecting Neighborhood} have been proposed \cite{isaacs2013dubins, jang2016optimal}.
	The intersecting node takes into account whether the footprint of a sensor includes other tasks when the vehicle is located on the sample node of the roadmap, and the necessarily intersecting node is an extension of the intersecting node which uses the nonholonomic and curvature-constrained vehicle dynamics as well as the footprint.
	Previous studies only handle the omni-directional sensors, so we extended the necessarily intersecting node to be able to handle the cases when the sensors are designed to look at the front or right sides based on the direction of the vehicle's heading.
	The shape of the footprint is assumed to be a circle with the radius $ r_{sen} $.

	\begin{figure*}[]
		\centering
		\captionsetup{justification=centering}
		\begin{minipage}{\linewidth}\centering
			\subfloat[3D view.]{\label{fig:ninOmni1}\includegraphics[frame,height=6cm]{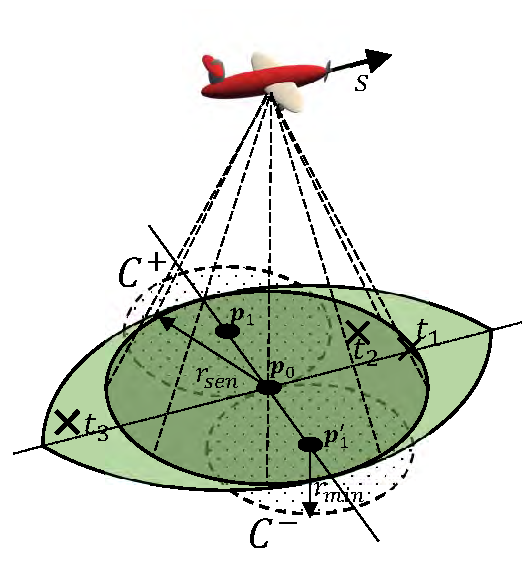}}\hspace{0.2cm}
			\subfloat[Top view.]{\label{fig:ninOmni2}\includegraphics[frame,height=6cm]{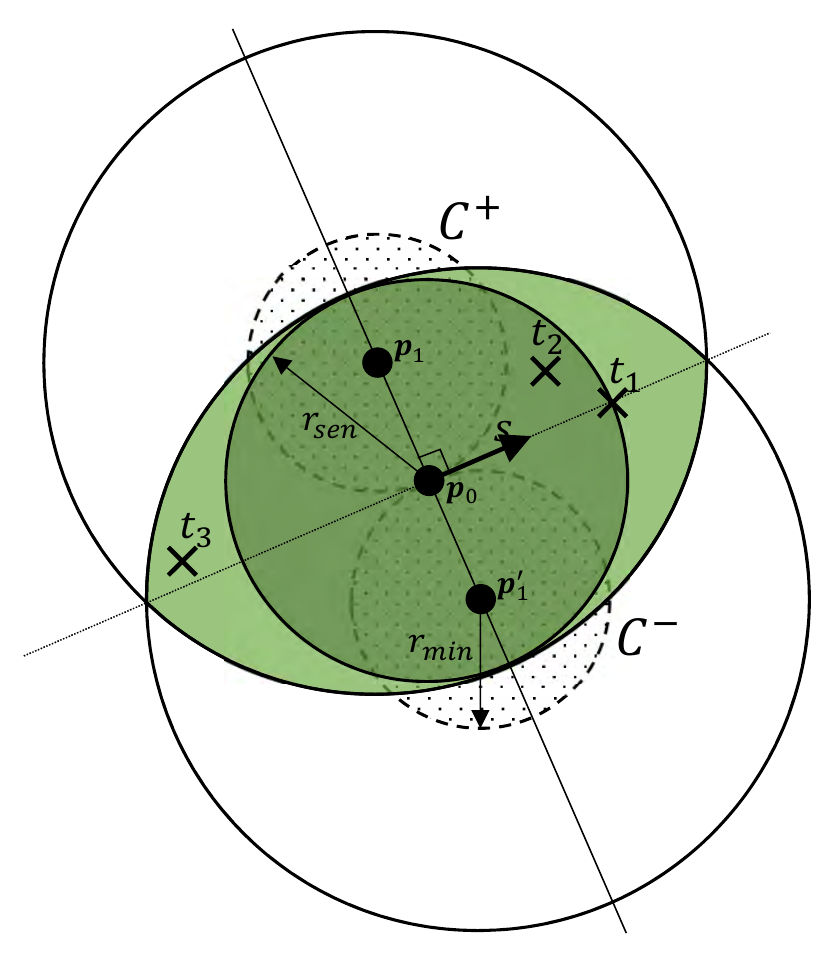}}
		\end{minipage}	
		\caption{An example of the necessarily intersecting nodes. A vehicle is passing through three tasks using an omni-directional sensor.}
		\label{fig:ninOmni}
	\end{figure*}
	
	\begin{figure*}[]
		\centering
		\captionsetup{justification=centering}
		\begin{minipage}{\linewidth}\centering
			\subfloat[Tour of an instance with 5 tasks and 2 vehicles, omni-directional sensors.]{\label{fig:ninPath1}\includegraphics[frame,width=0.495\linewidth]{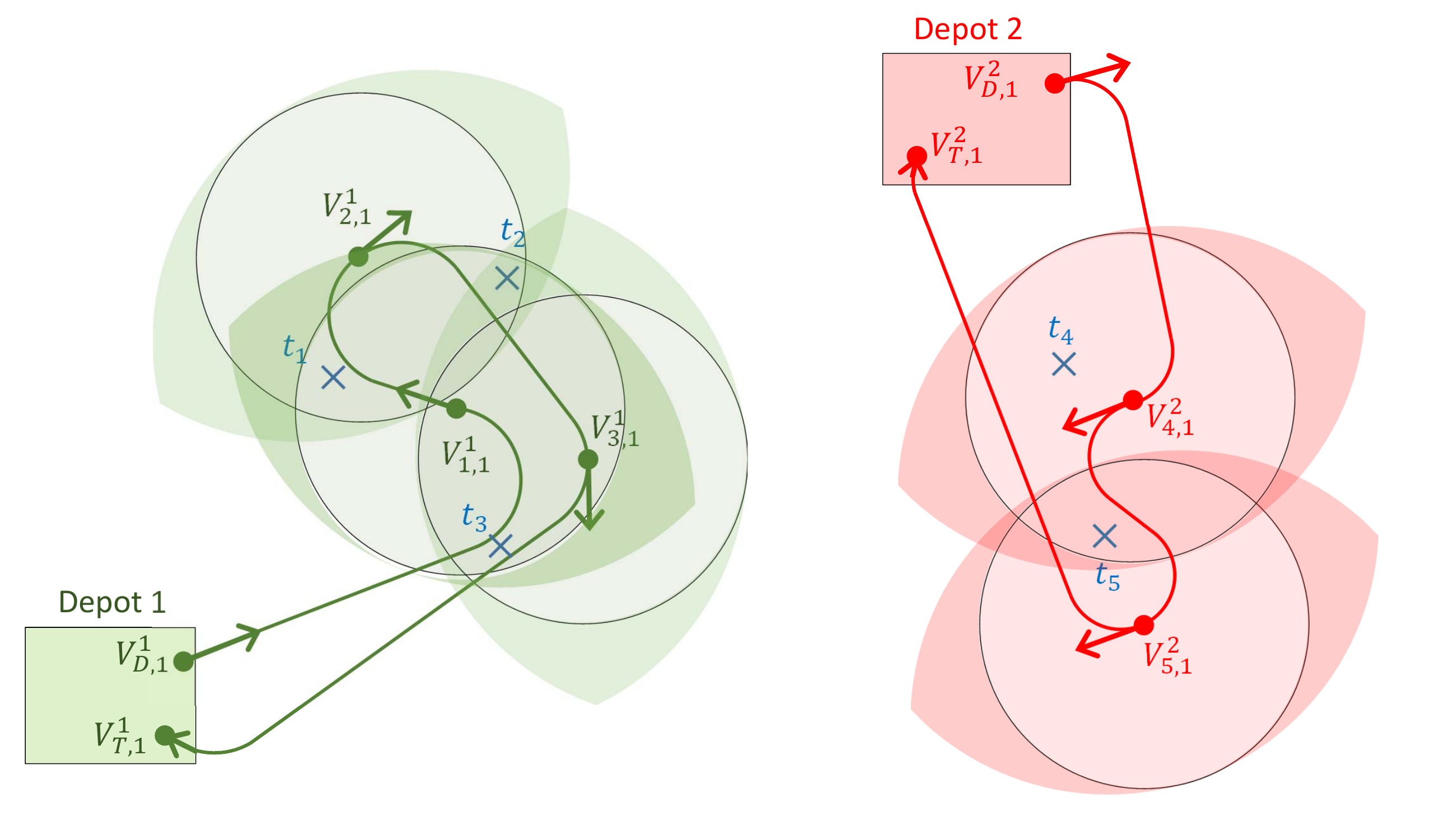}}
			\subfloat[Tour of an instance using necessarily intersecting nodes.]{\label{fig:ninPath2}\includegraphics[frame,width=0.495\linewidth]{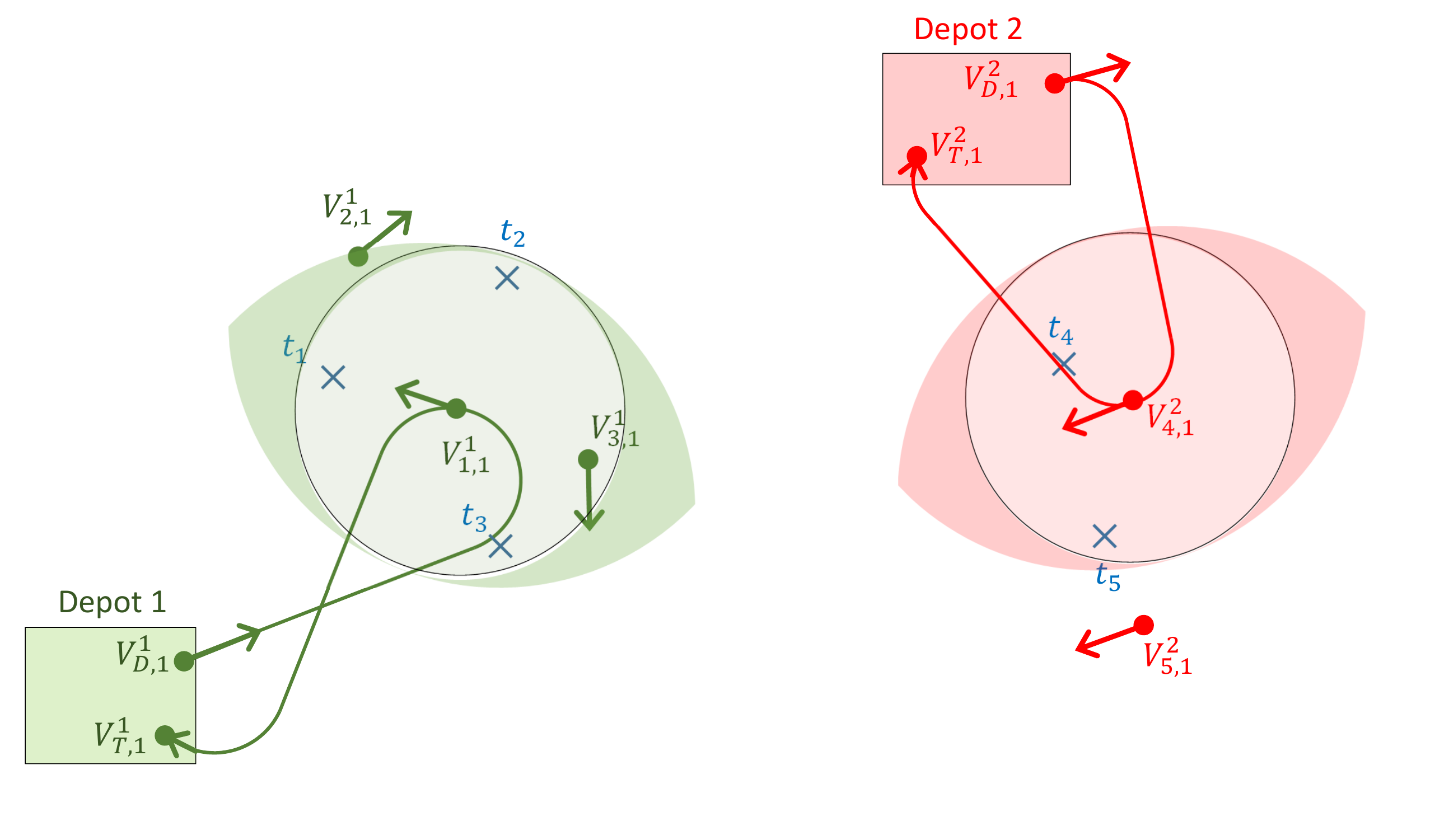}}	
		\end{minipage}	
		\caption{Tours of a GHMDATSP instance, omni-directional sensor.}
		\label{fig:ninPathOmni}
	\end{figure*}

	In Figure \ref{fig:ninOmni}, sample node $ s $ is one of the nodes visiting task $ t_1 $; simply, $ s $ belongs to the cluster $ V_{t1} $ of task $ t_1 $.
	Because the distance between node $ \textbf{p}_0 $ and task $ t_2 $ is less than $ r_{sen} $ in Figure \ref{fig:ninOmni}, the footprint (dark green area) contains $ t_2 $ when the vehicle is located on the sample node $ s $.
	Thus, if the position of any task $ t $ belongs to the sensor footprint when the vehicle is on the sample node $ s $, task $ t $ is called the intersecting node of sample node $ s $, and the dark green colored footprint area of sample node $ s $ is called the intersecting region.
	Furthermore, the path of the vehicle before and after the sample node $ s $ must exist in the non-dotted area between the circles $ C^+ $ and $ C^- $ for a short distance because the vehicle's dynamics is nonholonomic and curvature-constrained.
	The circles $ C^+ $ and $ C^- $ tangent at $ \textbf{p}_0 $ denote maximally steered paths from sample node $ s $, where the radius of each circle is $ r_{min} $ and the center is $ \textbf{p}_1 $ and $ \textbf{p}_1' $, respectively.
	Therefore, task $ t_3 $ is included in the footprint before the vehicle passes through sample node $ s $, and this node is called the necessarily intersecting node.
	The union of the light green and dark green colored areas is called the necessarily intersecting region.
	Because the intersecting region belongs to the necessarily intersecting region, all the task nodes in Figure \ref{fig:ninOmni} are the necessarily intersecting nodes of sample node $ s $.
	
	Suppose five tasks and two vehicles with omni-directional sensors are given in the instance as shown in Figure \ref{fig:ninPathOmni}.
	In the instance, the motion constraints of each vehicle follow the Dubins model, and each vehicle-task cluster has one sample node.
	Given the minimum turning radius of the vehicles, the optimal set of paths to visit the sample node of every task is as shown in Figure \ref{fig:ninPath1}.
	Since the necessarily intersecting region of sample node $ V^1_{1,1} $ includes both tasks $ t_2 $ and $ t_3 $ as well as $ t_1 $ in the figure, tasks $ t_2 $ and $ t_3 $ are the necessarily intersecting nodes of sample node $ V^1_{1,1} $.
	Likewise, since the necessarily intersecting region of the sample node $ V^2_{4,1} $ contains both the tasks $ t_4 $ and $ t_5 $, task $ t_5 $ is the necessarily intersecting node of the sample node $ V^2_{4,1} $.
	As a result, every task in the instance can be measured by the set of paths that vehicle 1 visits $ V^1_{1,1} $ and vehicle 2 visits $ V^2_{4,1} $ as shown in Figure \ref{fig:ninPath2}.

	\begin{figure*}[]
		\centering
		\captionsetup{justification=centering}
		\begin{minipage}{\linewidth}\centering
			\subfloat[3D view.]{\label{fig:ninFront1}\includegraphics[frame,height=5cm]{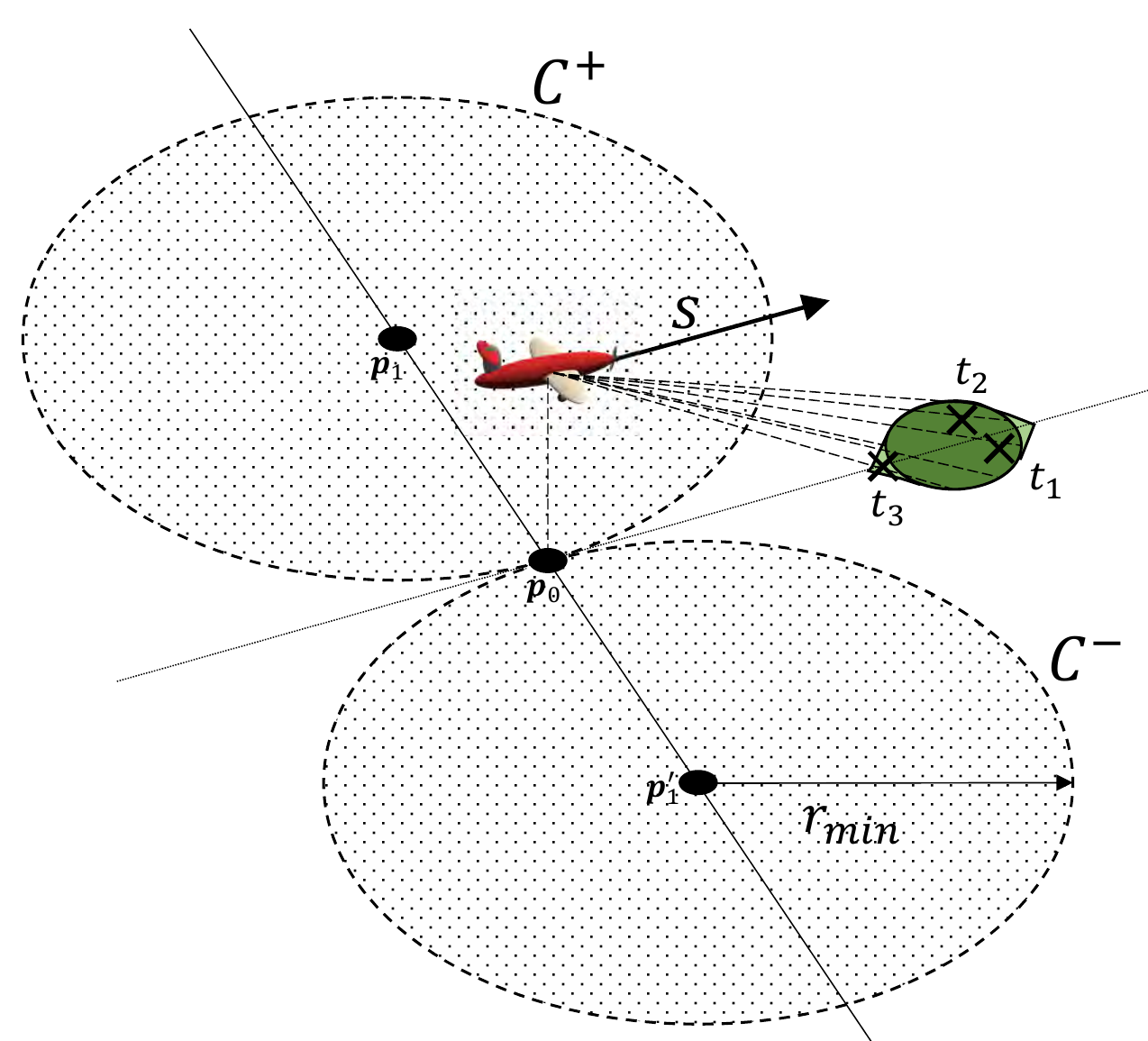}}\hspace{0.2cm}
			\subfloat[Top view.]{\label{fig:ninFront2}\includegraphics[frame,height=5cm]{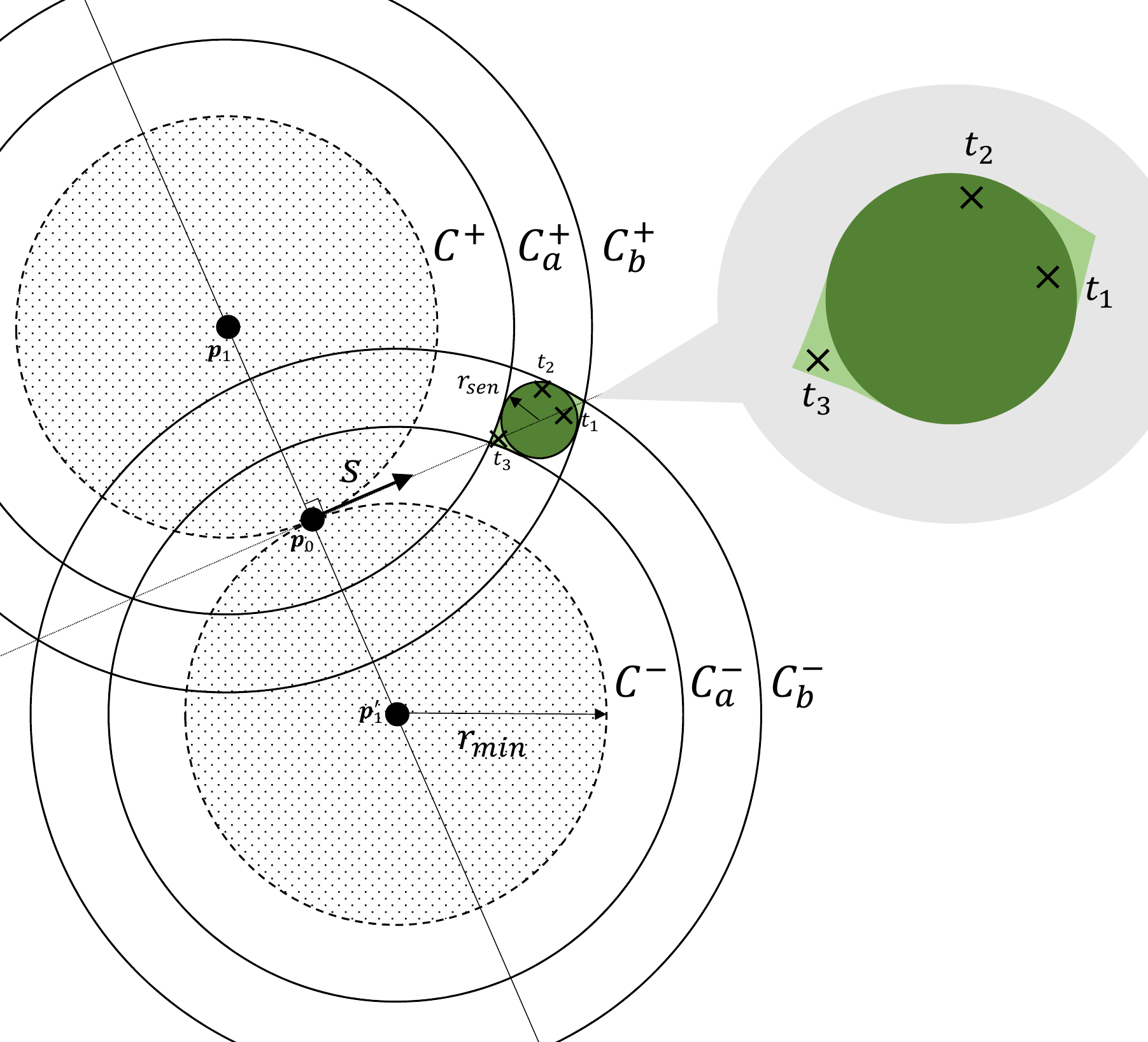}}
		\end{minipage}	
		\caption{An example of the necessarily intersecting nodes. A vehicle is passing through three tasks using a forward sensor}
		\label{fig:ninFront}
	\end{figure*}

	\begin{figure*}[]
		\centering
		\captionsetup{justification=centering}
		\begin{minipage}{\linewidth}\centering
			\subfloat[Tour of an instance with 5 tasks and 2 vehicles, frontward sensors.]{\label{fig:ninPath3}\includegraphics[frame,width=0.495\linewidth]{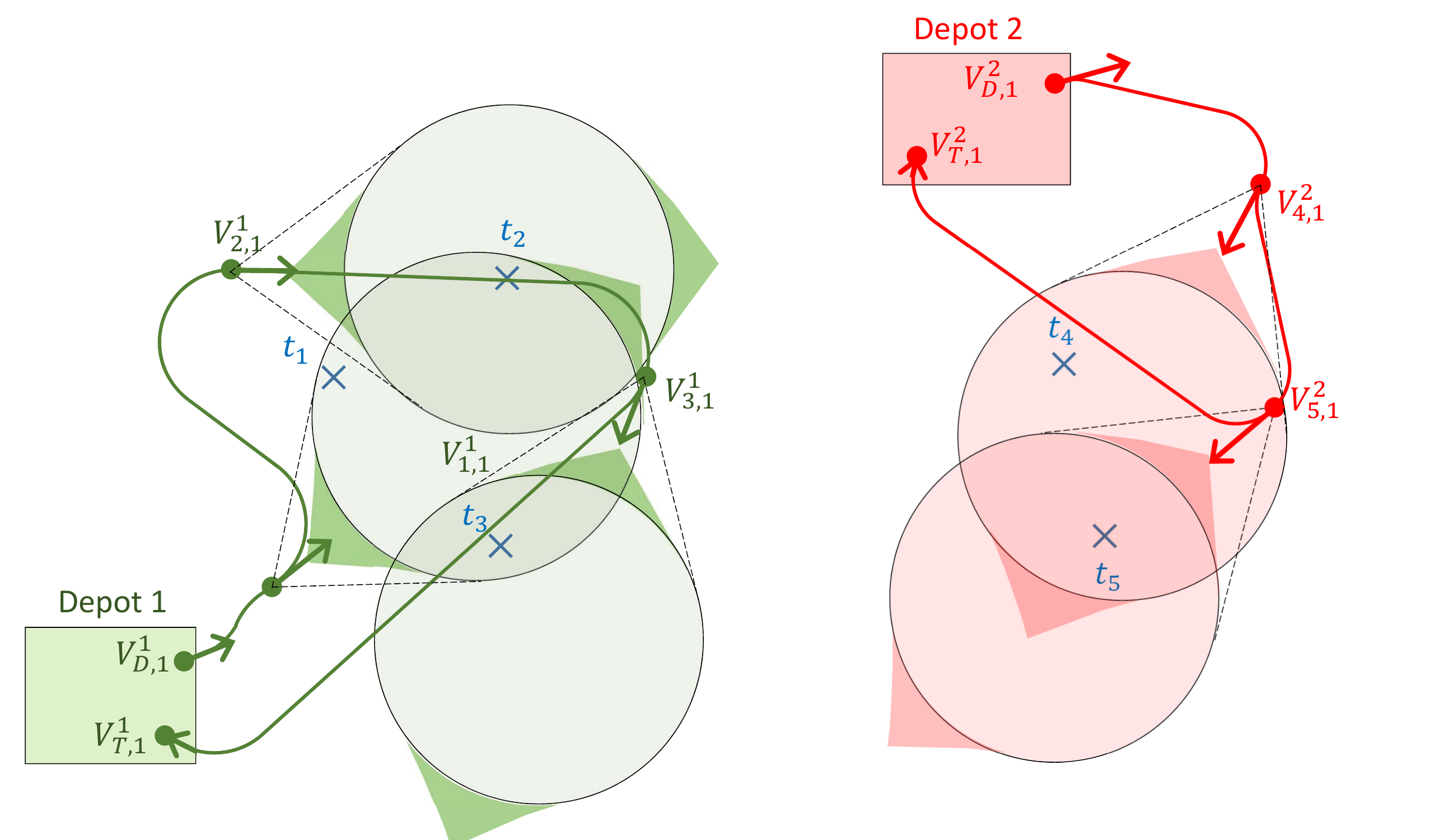}}
			\subfloat[Tour of an instance using necessarily intersecting nodes.]{\label{fig:ninPath4}\includegraphics[frame,width=0.495\linewidth]{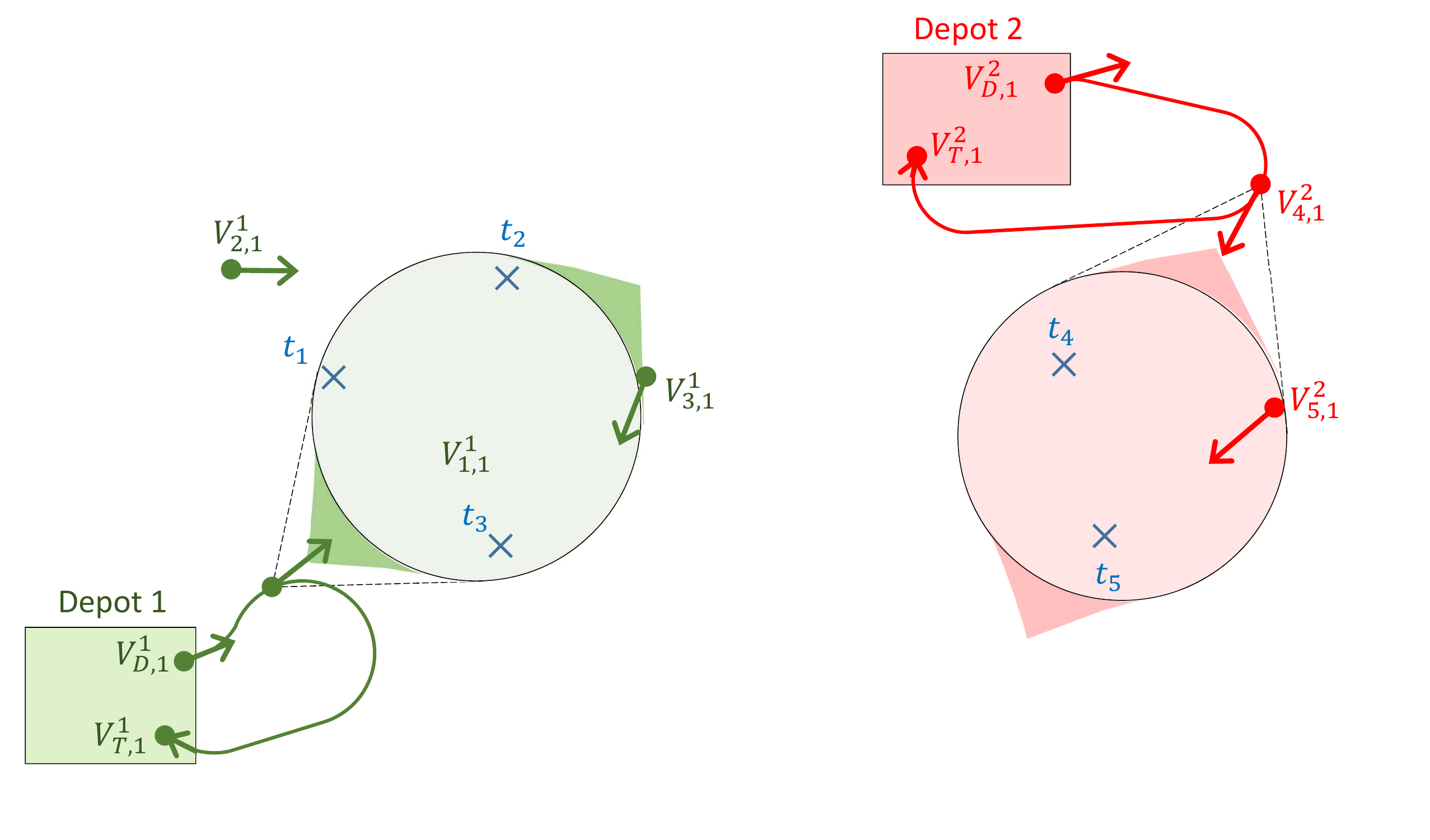}}	
		\end{minipage}	
		\caption{Tours of a GHMDATSP instance, frontward sensor.}
		\label{fig:ninPathFront}
	\end{figure*}
	
	\begin{figure*}[]
		\centering
		\captionsetup{justification=centering}
		\begin{minipage}{\linewidth}\centering
			\subfloat[3D view.]{\label{fig:ninRight1}\includegraphics[frame,height=6cm]{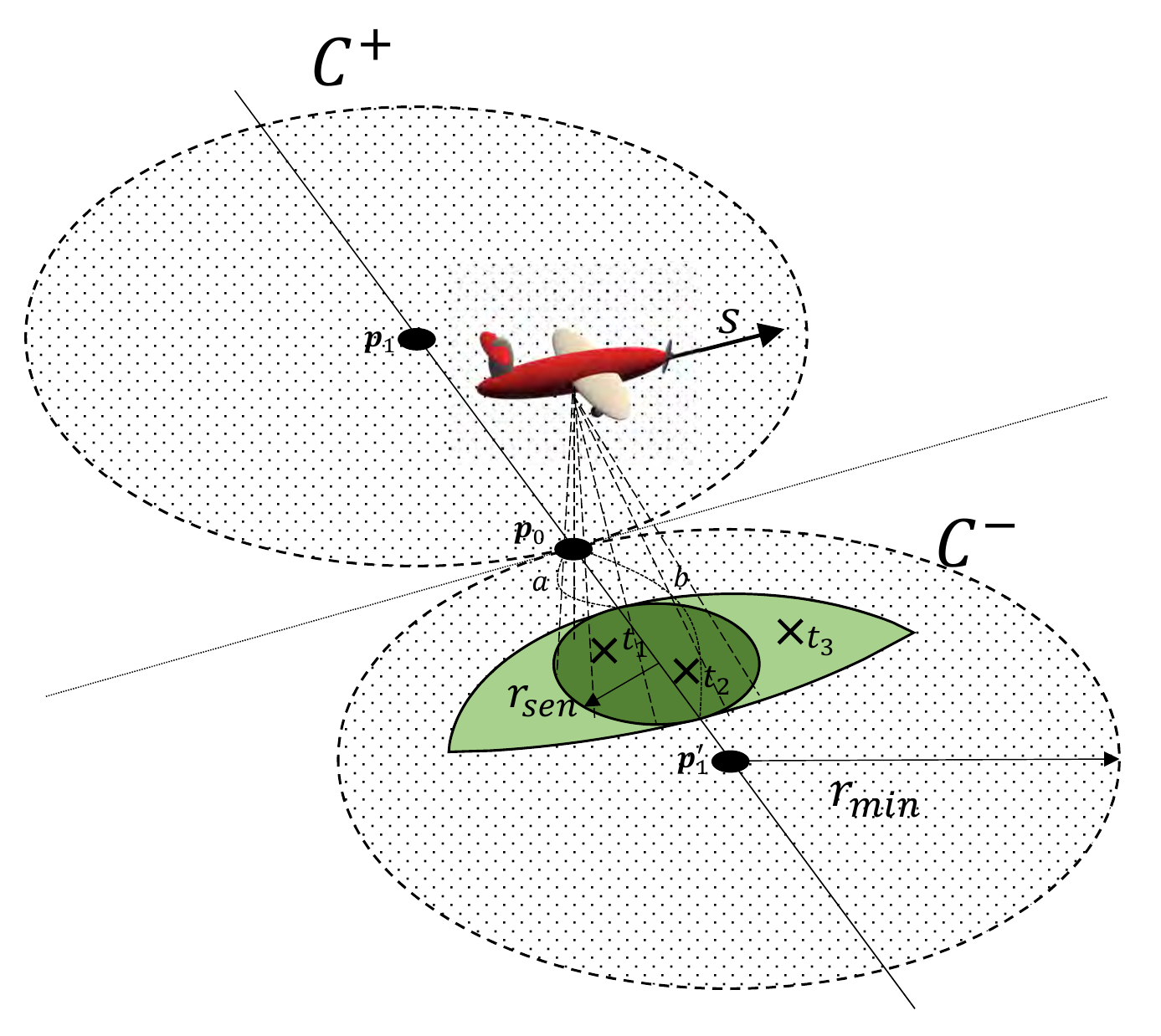}}\hspace{0.2cm}
			\subfloat[Top view.]{\label{fig:ninRight2}\includegraphics[frame,height=6cm]{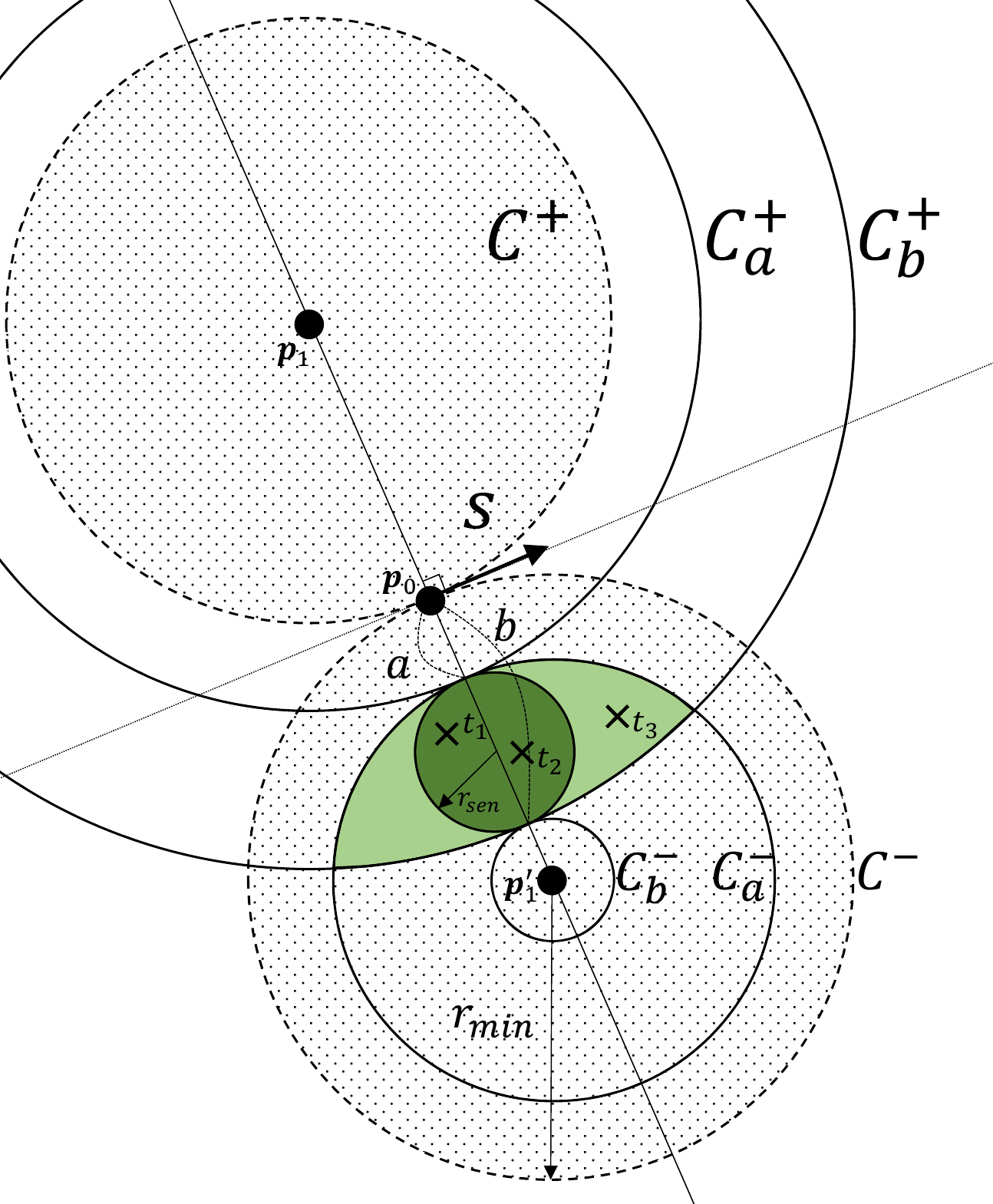}}	
		\end{minipage}	
		\caption{An example of the necessarily intersecting nodes. A vehicle is passing through three tasks using a rightward sensor.}
		\label{fig:ninRight}
	\end{figure*}
	
	\begin{figure*}[]
		\centering
		\captionsetup{justification=centering}
		\begin{minipage}{\linewidth}\centering
			\subfloat[Tour of an instance with 5 tasks and 2 vehicles, rightward sensors.]{\label{fig:ninPath5}\includegraphics[frame,width=0.495\linewidth]{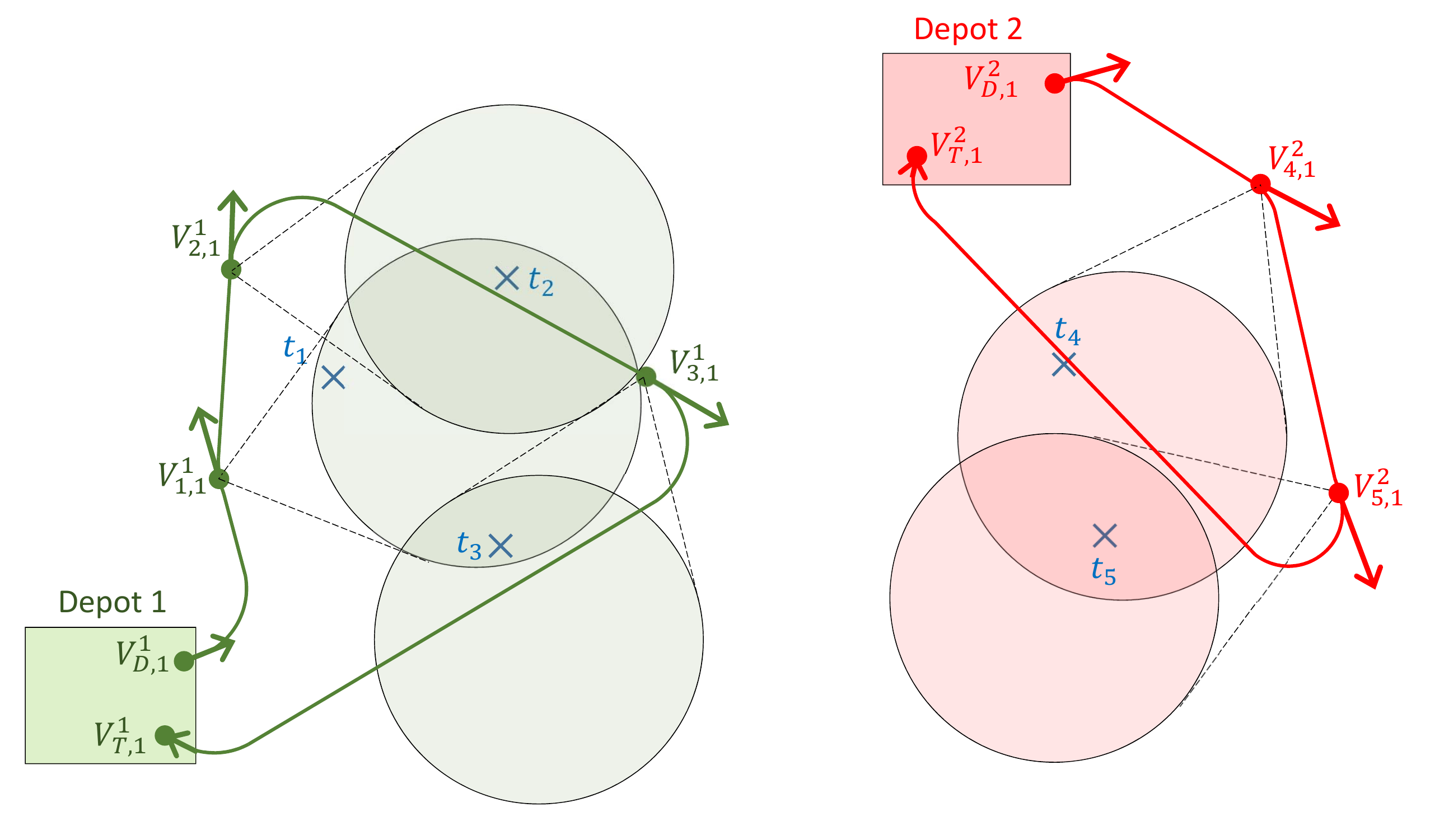}}
			\subfloat[Tour of an instance using necessarily intersecting nodes.]{\label{fig:ninPath6}\includegraphics[frame,width=0.495\linewidth]{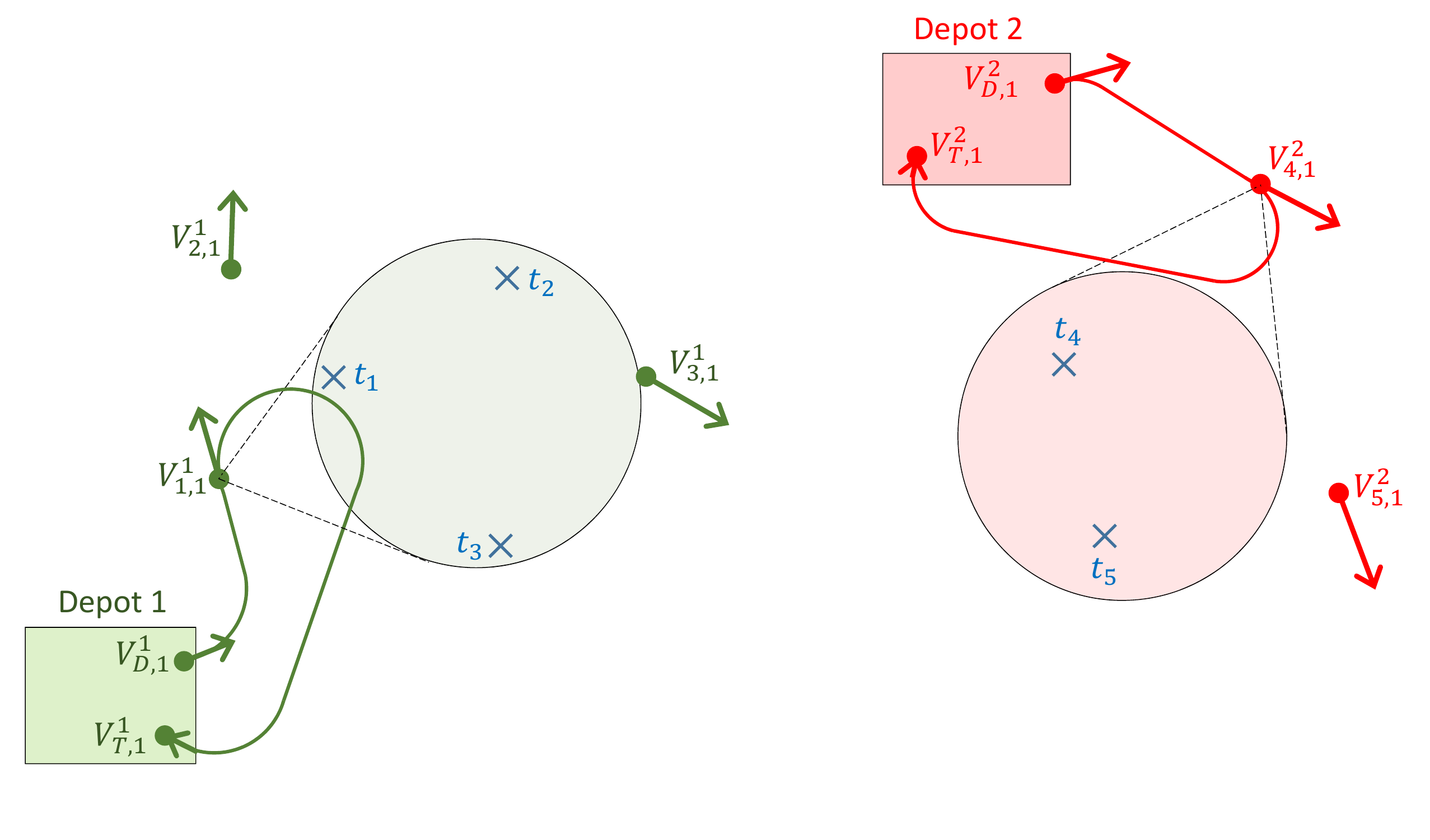}}	
		\end{minipage}	
		\caption{Tours of a GHMDATSP instance, rightward sensor.}
		\label{fig:ninPathRight}
	\end{figure*}

	Figures \ref{fig:ninFront} and \ref{fig:ninRight} show the situation in which the directional sensor looks ahead and right, respectively, with respect to the heading of the vehicle. 
	Similar to the situation of Figure \ref{fig:ninOmni}, sample node $ s $ belongs to the cluster $ V_{t1} $ of task $ t_1 $, and $ t_2 $ is located in the footprint of the sensor (dark green) when the vehicle follows the position and heading angle of sample node $ s $.
	In addition, the sensing area that is inevitably covered due to the vehicle dynamics is colored light green.
	For convenience, the sensing area of the vehicle when it passes through sample node $ s $ is zoomed in at the top right of Figure \ref{fig:ninFront2}.
	If the vehicle turns left along $ C^+ $, the footprint is generated along the donut-like area between the circles $ C^+_a $ and $ C^+_b $, whereas the footprint is generated along the donut-like area between the circles $ C^-_a $ and $ C^-_b $ if the vehicle turns right along $ C^- $.
	The superscripts + and - of circle $ C $ mean that a circle is centered on points $ \textbf{p}_1 $ and $ \textbf{p}_1' $, respectively, and the subscripts $ a $ and $ b $ mean that the radius of a circle is $ r_{sen} + a $ and $ r_{sen} + b $, respectively. 
	In an exceptional case, when the sensor is facing to the right, the radii of $ C^+_a $ and $ C^+_b $ are $ r_{sen} + a $ and $ r_{sen} + b $, respectively.
	In Figures \ref{fig:ninFront} and \ref{fig:ninRight}, the colored part is the intersection area for all cases where the vehicle moves with LL, LR, RL, and RR (L: left, R: right) turns before and after passing sample node $ s $, which is the necessarily intersecting region of sample node $ s $.
	If $ r_{min} $ is less than $ b $ in Figure \ref{fig:ninRight}, the additional intersecting region (light green color) caused by the vehicle dynamics becomes smaller than the footprint of the sensor, thus the necessarily intersecting region becomes the same as the footprint of the sensor.
	Refer to Appendix \ref{app:nir} for the relationship of the parameters and the necessarily intersecting region in detail when the sensor looks forward.
	
	As in Figure \ref{fig:ninPathOmni}, suppose that there is a simple instance of GHMDATSP in each of Figures \ref{fig:ninPathFront} and \ref{fig:ninPathRight} with five tasks, two Dubins vehicles, and one sample node for each vehicle-task cluster.
	The frontward sensor is mounted on each vehicle in Figure \ref{fig:ninPathFront} and the rightward sensor is mounted in Figure\ref{fig:ninPathRight}.
	Given the turning rate of each vehicle, the sets of optimal paths to visit the sample node of every task are shown in Figures \ref{fig:ninFront1} and \ref{fig:ninRight1}, respectively.
	In Figures \ref{fig:ninPathFront} and \ref{fig:ninPathRight}, tasks $ t_2 $ and $ t_3 $ are necessarily intersecting nodes of sample node $ V^1_{1,1} $ and task $ t_5 $ is a necessarily intersecting node of sample node $ V^2_{4,1} $.
	As a result, all the tasks in each instance can be measured through the set of paths that vehicle 1 visits $ V^1_{1,1} $ and vehicle 2 visits $ V^2_{4,1} $ as shown in Figures \ref{fig:ninPath4} and \ref{fig:ninPath6}.

	\subsection{Necessarily Intersecting Region for Arbitrary Footprints}

	\begin{figure*}[t]
		\centering
		\captionsetup{justification=centering}
		\begin{minipage}{\linewidth}\centering
			\subfloat[3D view.]{\label{fig:ninArbi1}\includegraphics[frame,height=5cm]{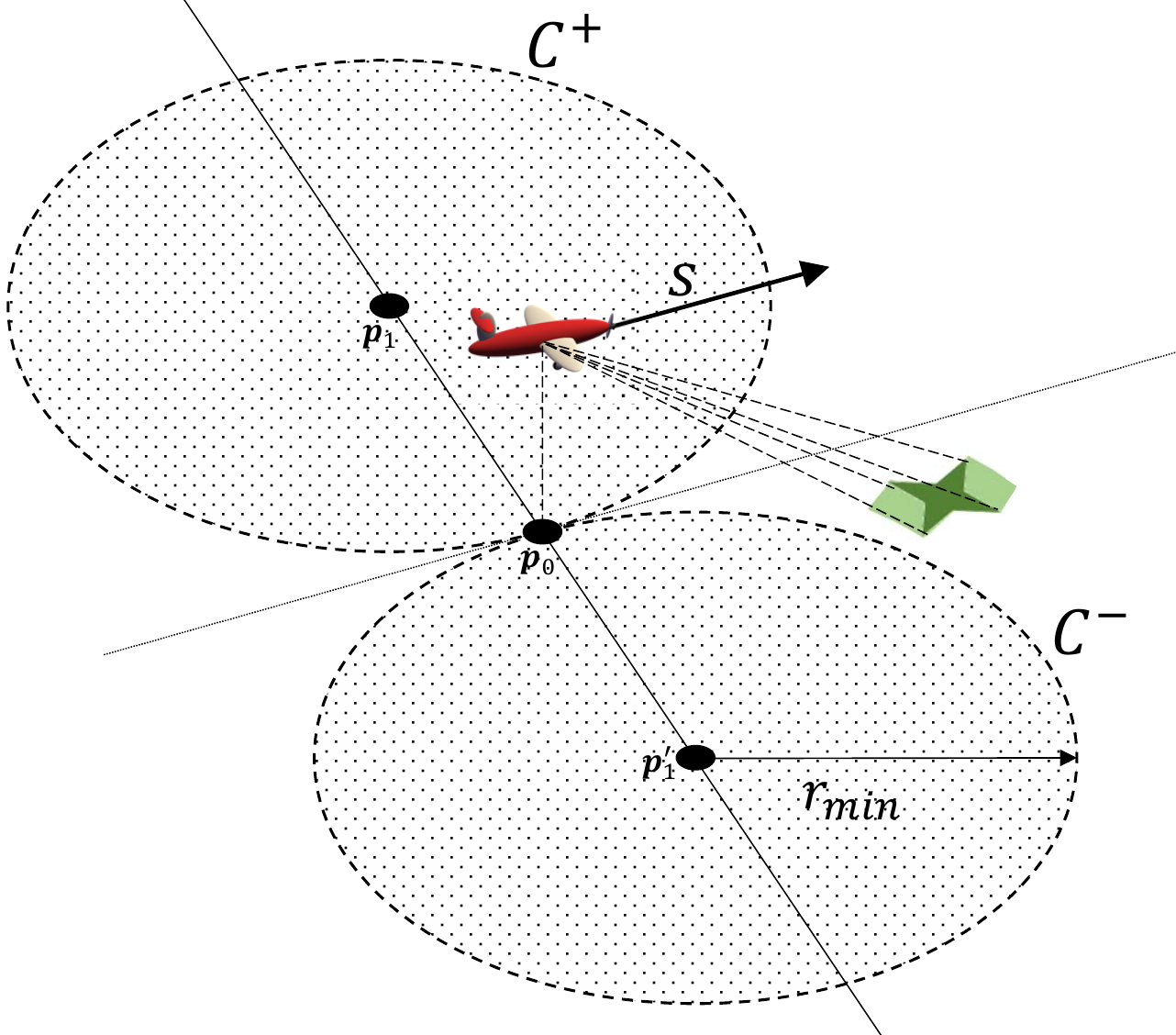}}\hspace{0.2cm}
			\subfloat[Top view.]{\label{fig:ninArbi2}\includegraphics[frame,height=5cm]{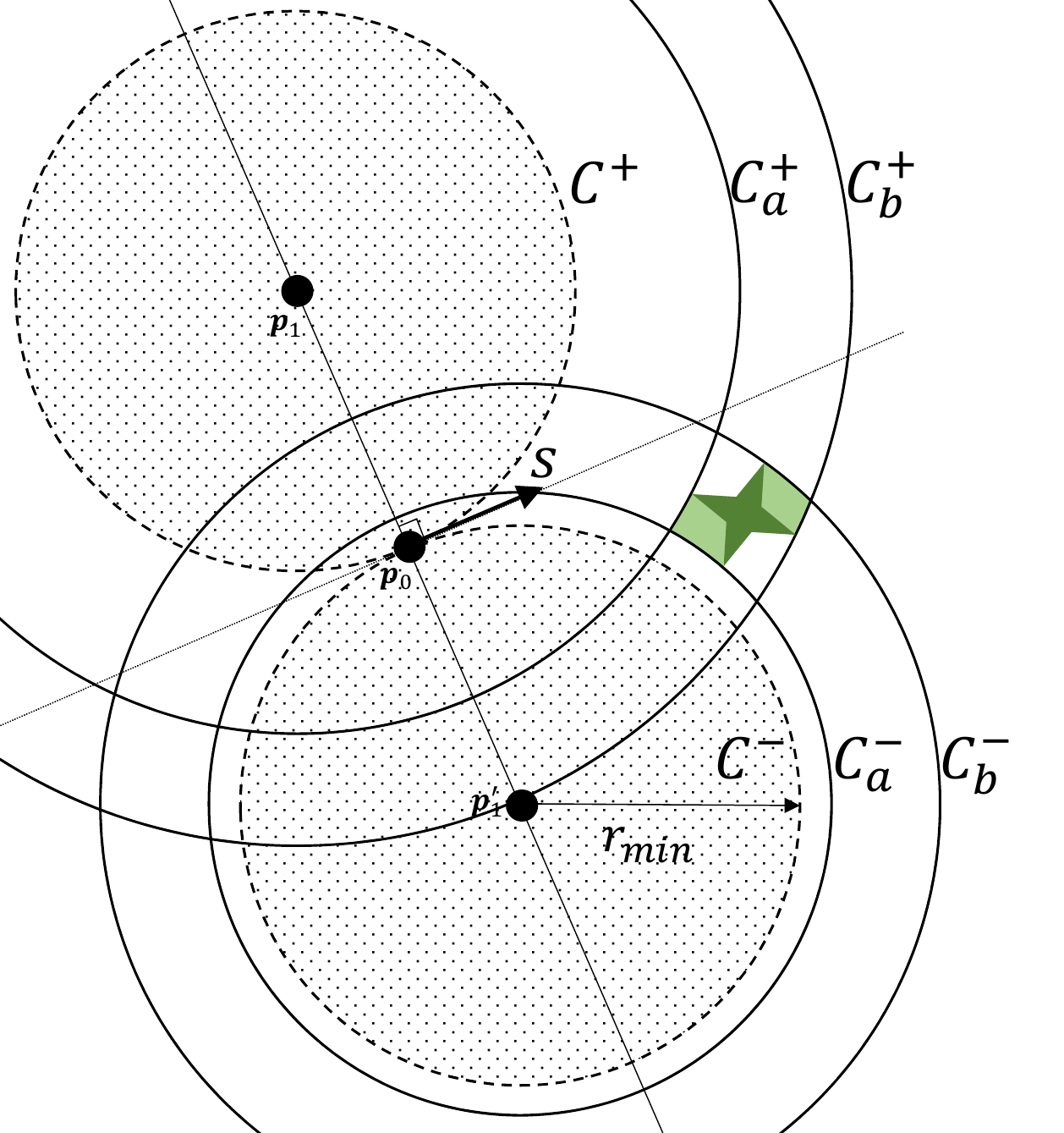}}
		\end{minipage}	
		\subfloat[An intersection to compute the necessarily intersecting region.]{\label{fig:ninArbi3}\includegraphics[height=4cm]{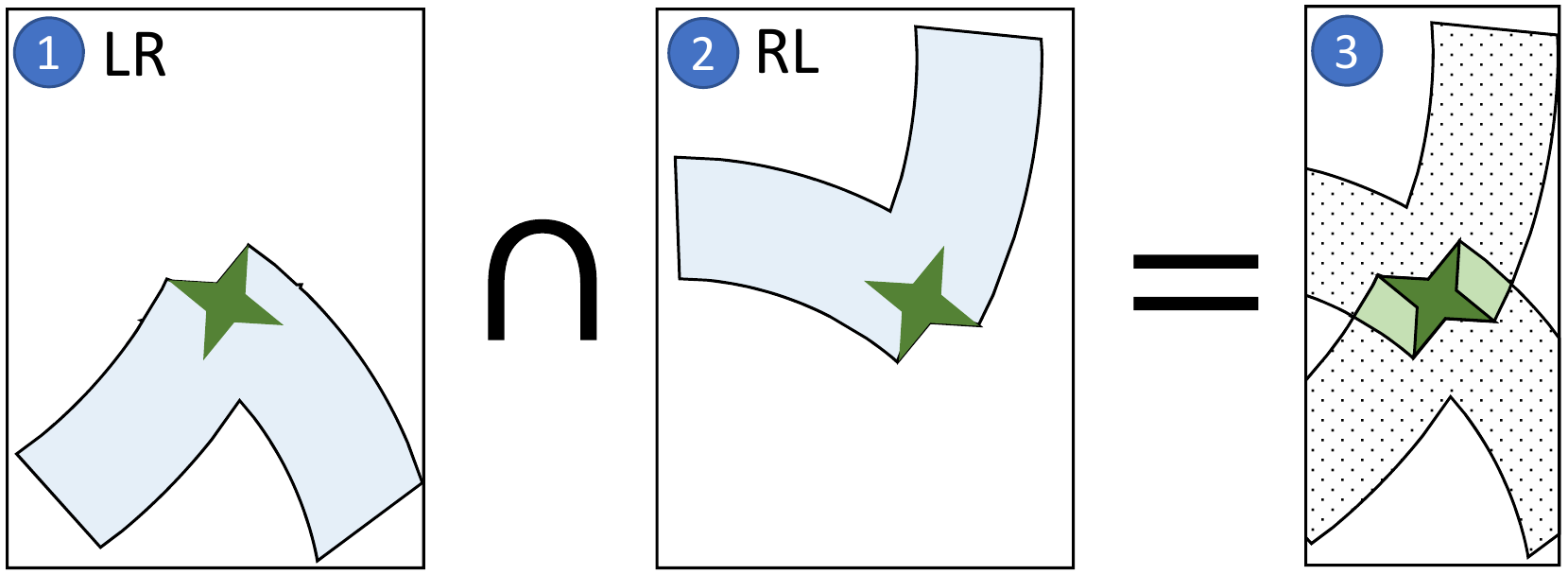}}
		\caption{Necessarily intersecting region for a footprint with arbitrary shape.}
		\label{fig:ninArbi}
	\end{figure*}
	
	This section describes the way of computing necessarily intersecting region for arbitrarily shaped and located footprints.
	In the example in Figure \ref{fig:ninArbi}, it is assumed that the shape of the footprint is similar to a four-point star which is quite non-convex, and the position is in front of the vehicle but not at the exact front.
	The Dubins vehicle has three maneuvers: left turn (L), right turn (R), and straight (S), therefore there are nine cases of the Dubins path: L-(L,R,S), R-(L,R,S), and S-(L,R,S), if the path is combined before and after the certain point.
	Necessarily intersecting region is the intersection of the trace of the footprint for all of the above cases, but the result is the same as the intersection of both cases LR and RL.
	The trace of the footprint along the trajectories LR and RL is shaded in sky blue in Figure \ref{fig:ninArbi3}, and the intersection of each area is indicated by the colored area on the right side of the figure.
	Let the points $ \textbf{p}_1 $ and $ \textbf{p}_1' $ the center of left-turning and right-turning circles, respectively.
	The procedure of obtaining the necessarily intersecting region for sample node $ s $ is as follows.
	
	\begin{enumerate}
		\item Obtain circles centered at $ \textbf{p}_1 $ with radii set to the minimum and maximum distance between $ \textbf{p}_1 $ and the footprint (Circles $ C^+_a $ and $ C^+_b $ in Figure \ref{fig:ninArbi2}).
		\item Obtain circles centered at $ \textbf{p}'_1 $ with radii set to the minimum and maximum distance between $ \textbf{p}'_1 $ and the footprint (Circles $ C^-_a $ and $ C^-_b $ in Figure \ref{fig:ninArbi2}).
		\item The left-turning trace of the footprint can be obtained through circles $ C^+_a $ and $ C^+_b $, and the right-turning trace of the footprint can be obtained through circles $ C^-_a $ and $ C^-_b $. Obtain the traces of the footprint L-R and R-L based on sample node $ s $ and calculate their intersection.
	\end{enumerate}

	\section{Procedure}\label{sec:proc}

	The procedure to obtain a solution for the multi-vehicle path planning problem in this study is as follows:
	\begin{enumerate}
		\item Given the scenario, construct an instance of the GHMDATSP.
		\item Transform the GHMDATSP instance into the form of the ATSP.
		\item Solve the ATSP instance and transform the solution into the form of the GHMDATSP.
		\item Locally optimize each state in the continuous domain.
	\end{enumerate}
	
	The details for each step are described below.

	\subsection{Constructing the GHMDATSP instance}\label{subsec:proc1}
	
	An instance of the GHMDATSP can be expressed as a union of the directed graph for each vehicle.
	To constitute the GHMDATSP instance, inputs must be given: for example, information of vehicles and tasks.
	More specifically, about the vehicle, information about the starting position, speed, and the minimum radius of rotation for each vehicle is required.
	For a task, the radius of the neighboring area and the number of sample nodes including an entering information for one of the vehicles are required.
	Each sample node holds information about the entry position and entering angle of the vehicle.
	When creating sample nodes, the quasi-random Halton sequence is used to prevent the position and direction of the sample node from being located on a specific side on a neighboring circle.
	The sample nodes are created under the following constraints: the position of the sample nodes are on the boundary of the neighboring circle of each task; the entering direction is towards the neighboring circle, but not directly towards the center of the circle (exact task position). 
	Based on the given inputs, the paths between all of the different sample nodes except the sample nodes belonging to the same task are represented as a directed edge with weights as its moving cost (flight time).
	The moving cost of each vehicle's path is calculated as a minimum Dubins path given the turning radius and velocity.
	
	\subsection{Transformation from GHMDATSP to ATSP with Necessarily Intersecting Nodes}\label{subsec:proc2}
	
	To express that a task is finished when at least one of the sample nodes in each task has been visited, we add the zero-cost edges between the sample nodes in the same task, depot, or terminal cluster.
	The GTSP can be transformed into the form of the ATSP using the Noon-Bean transformation \cite{noon1993efficient}, and we apply a similar method to the GHMDATSP.
	
	\begin{figure*}[]
		\centering
		\captionsetup{justification=centering}
		\begin{minipage}{\linewidth}\centering
			\subfloat[Sample nodes for vehicle \#1.]{\label{fig:trans1}\includegraphics[frame,width=.495\linewidth]{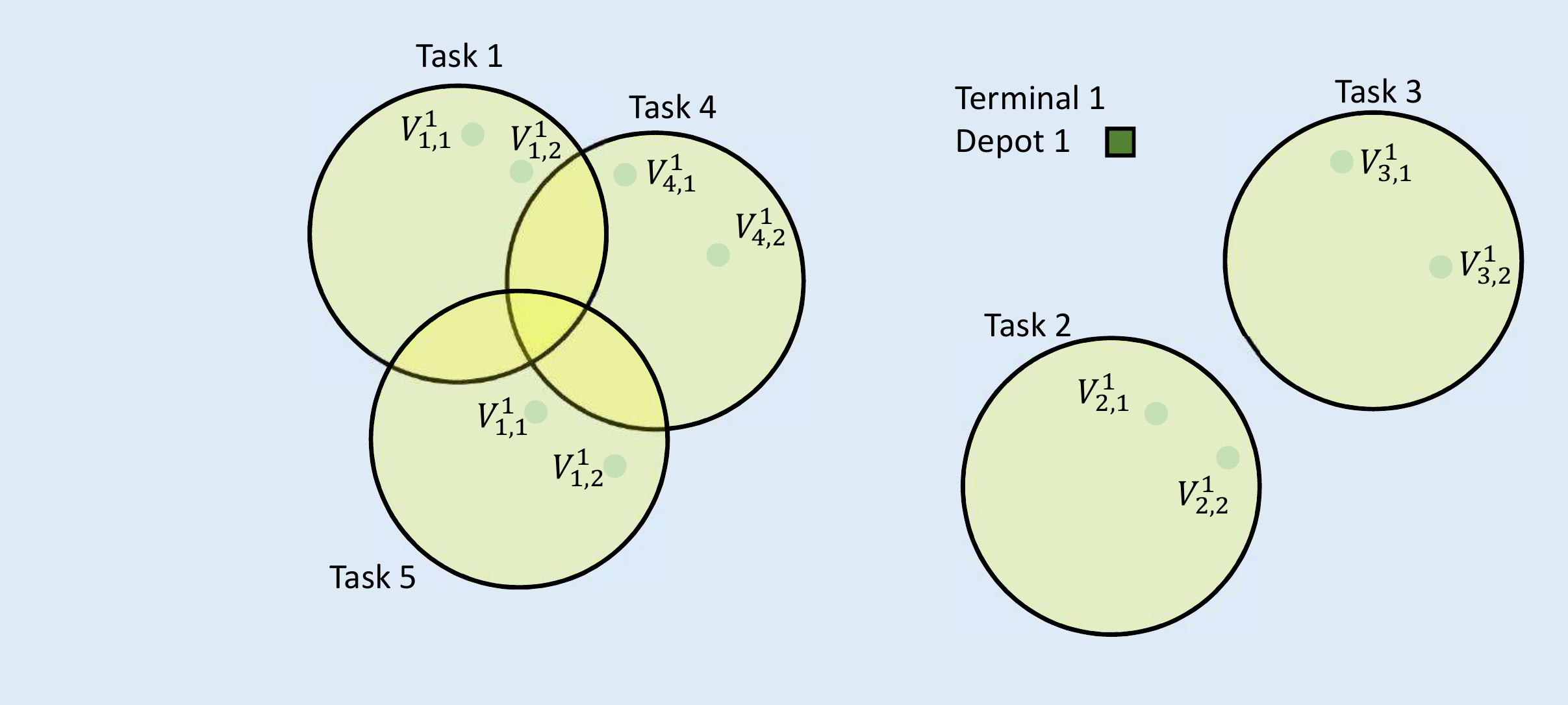}}
			\subfloat[Sample nodes for vehicle \#2.]{\label{fig:trans2}\includegraphics[frame,width=.495\linewidth]{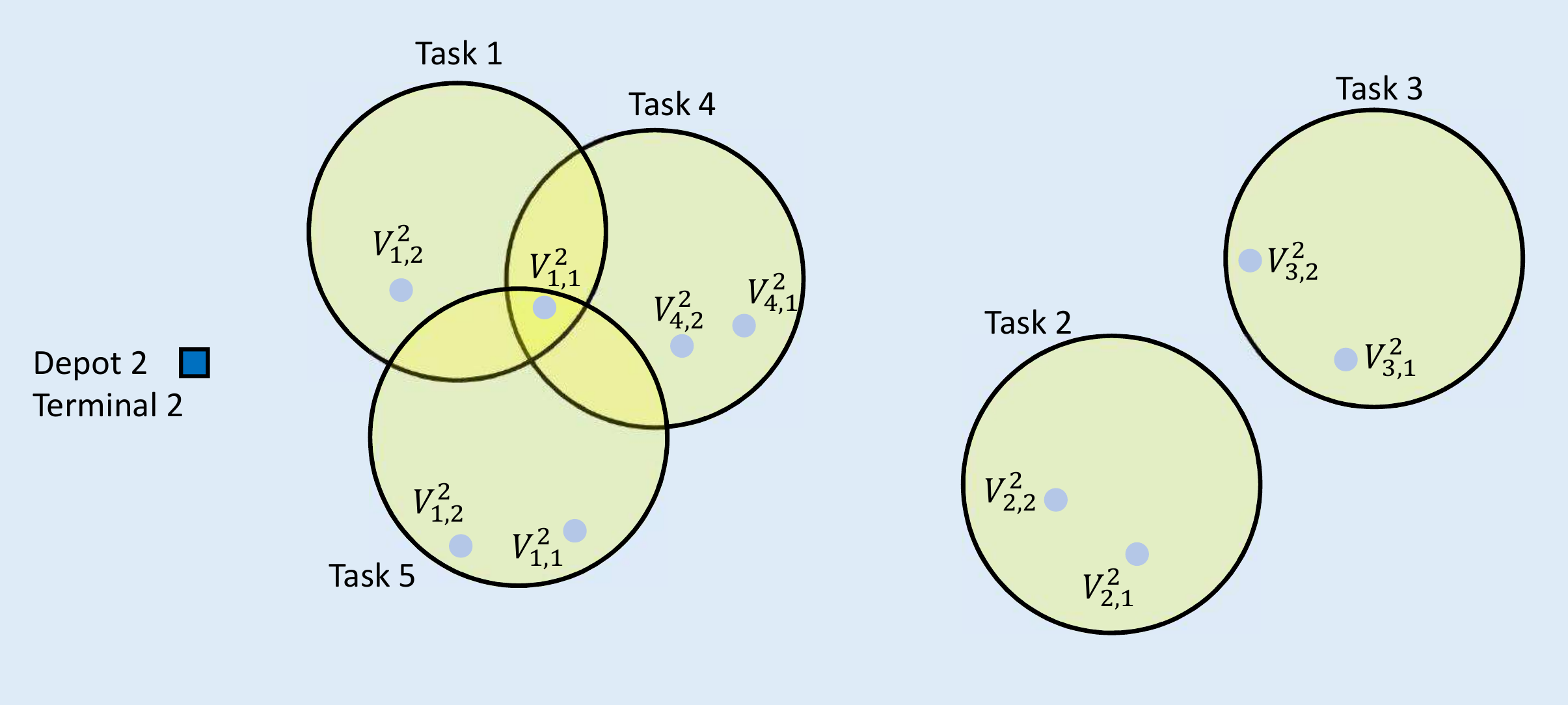}}
		\end{minipage}
		\begin{minipage}{\linewidth}\centering
			\subfloat[All sample nodes in the scenario. Depots and terminals also have multiple samples, but are skipped for simplicity.]{\label{fig:trans3}\includegraphics[frame,page=3,width=.495\linewidth]{transFigs3}}
			\subfloat[Feasible solution for the scenario, without the necessarily intersecting nodes.]{\label{fig:trans4}\includegraphics[frame,width=.495\linewidth]{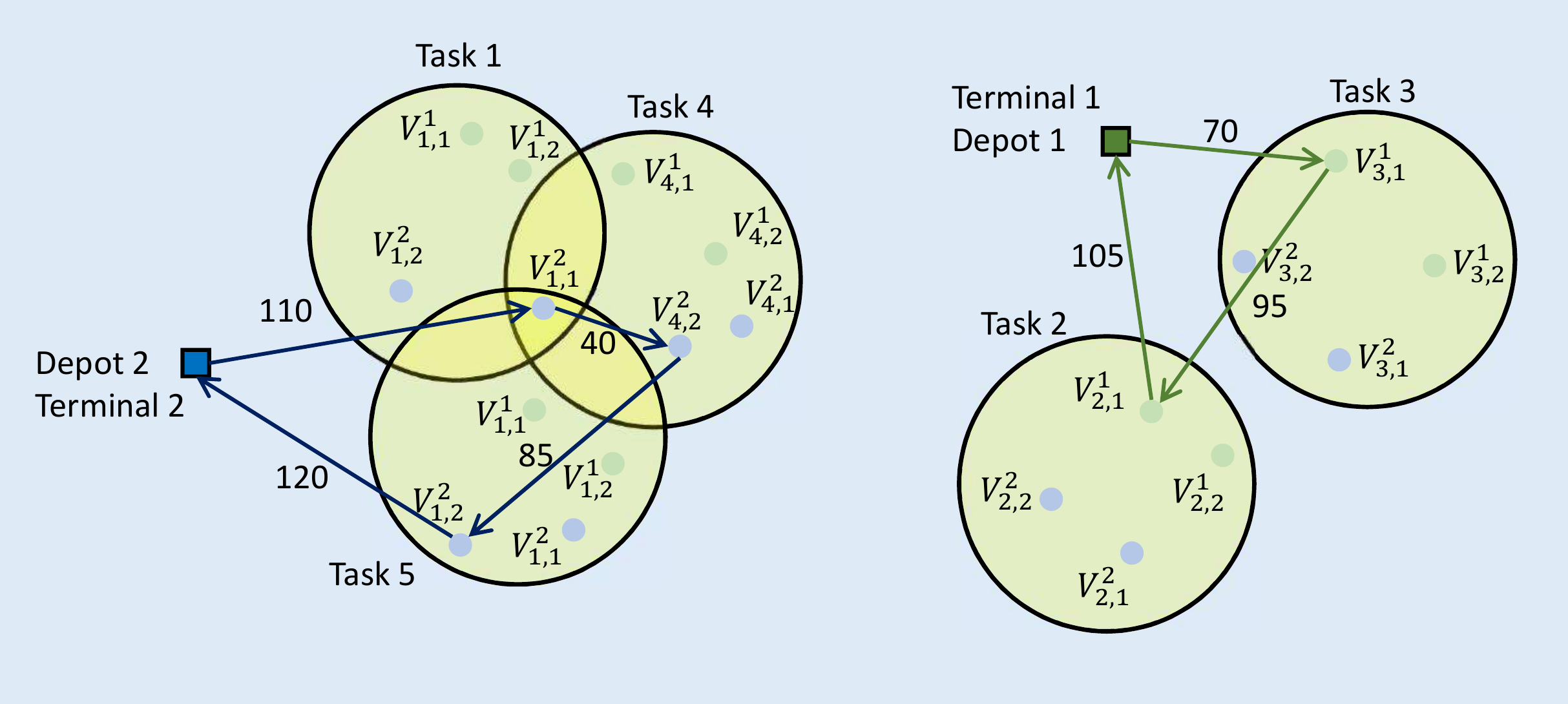}}
		\end{minipage}
		\begin{minipage}{\linewidth}\centering
			\subfloat[Feasible solution for the scenario, with the necessarily intersecting nodes.]{\label{fig:trans5}\includegraphics[frame,width=.495\linewidth]{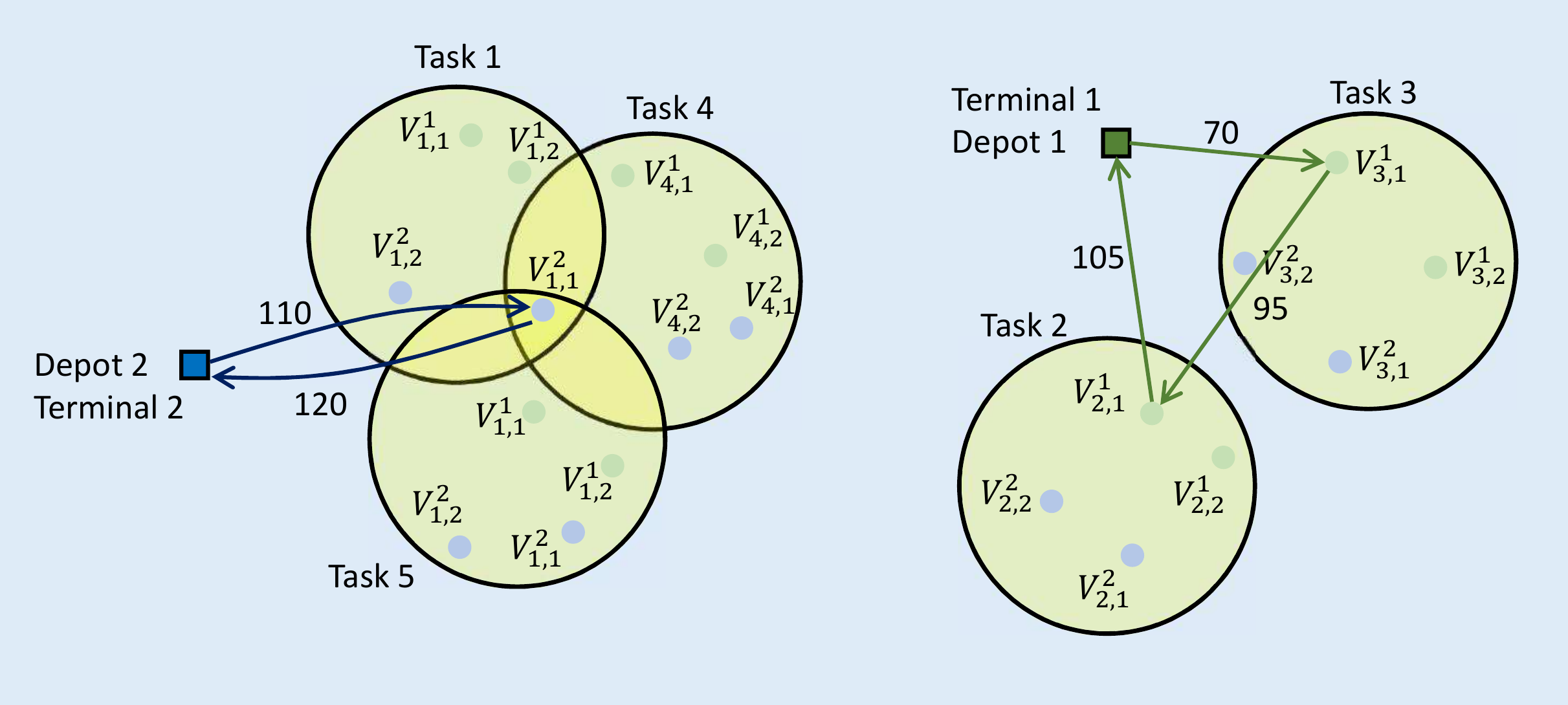}}
			\subfloat[Feasible solution for the scenario, with the necessarily intersecting nodes.]{\label{fig:trans6}\includegraphics[frame,width=.495\linewidth]{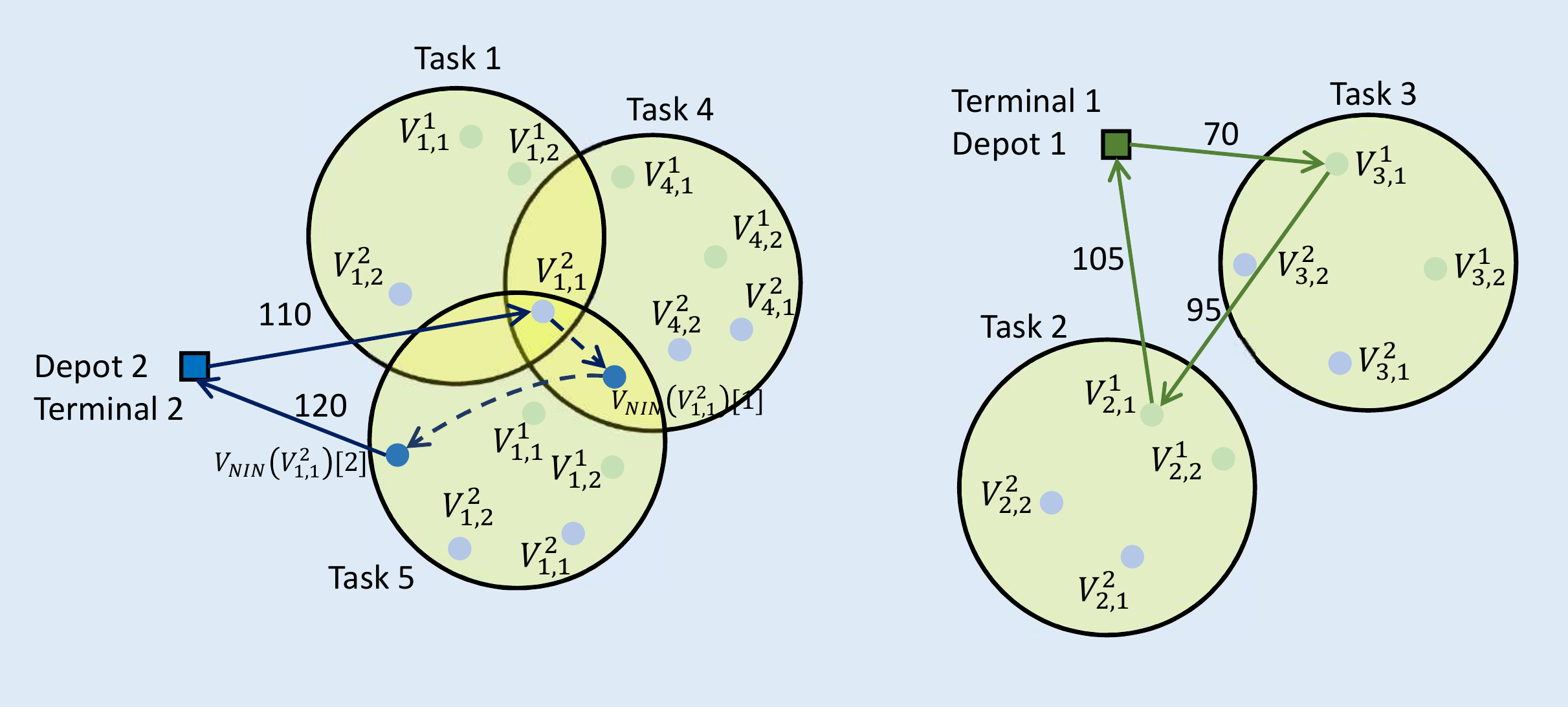}}
		\end{minipage}
		\begin{minipage}{\linewidth}\centering
			\subfloat[Construction of a transformation from GHMDATSP to ATSP: Adding 0-cost edges connecting the nodes within each group.]{\label{fig:trans7}\includegraphics[frame,page=7,width=.495\linewidth]{transFigs7}}
			\subfloat[Construction of a transformation from GHMDATSP to ATSP: Corresponding solution to the ATSP.]{\label{fig:trans8}\includegraphics[frame,width=.495\linewidth]{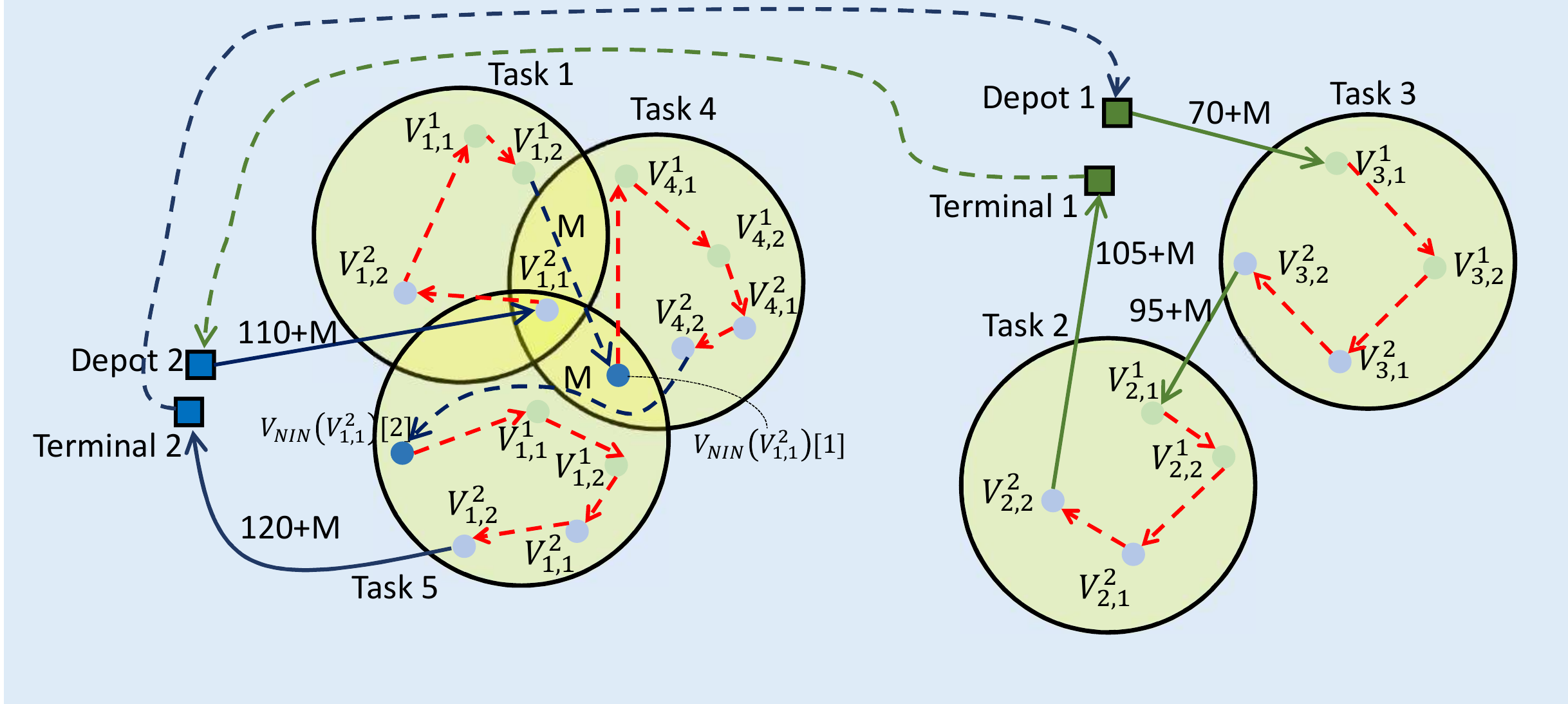}}
		\end{minipage}
		\caption{Procedure for constructing the ATSP solution from the GHMDATSP instance.}
		\label{fig:trans}
	\end{figure*}

	For illustrative purposes, consider a simple scenario where a set of two vehicles must visit 5 tasks.
	We assume that each task, depot, and terminal has two sample nodes, and moving costs are already calculated between any sample nodes.
	The schematic for this scenario is shown in Figures \ref{fig:trans1}, \ref{fig:trans2}, and \ref{fig:trans3}, and a feasible solution is shown in Figure \ref{fig:trans4}.
	In Figure \ref{fig:trans4}, vehicle 1 returns to the terminal after passing tasks 3 and 2, and vehicle 2 traverses the tasks in the order 1, 4, and 5 and then returns to the terminal.
	
	The objective in a given scenario is to find a set of tours whose vehicles have the least cost for traversing all mission areas.
	If we apply the necessarily intersecting node in the scenario, the overall moving cost can be reduced because vehicle 2 passes tasks 4 and 5 spontaneously just by visiting node $ V_{1,1}^2 $ in task 1.
	Therefore, a feasible solution can be obtained as shown in Figure \ref{fig:trans5} by applying the concept.
	To construct the problem with consideration of the necessarily intersecting node, virtual sample nodes for each sample node are generated and added to the corresponding task clusters.
	Referring to Figure \ref{fig:trans6}, in the case of node $ V_{1,1}^2 $, since it originally belongs to task 1 but naturally passes through tasks 4 and 5, the virtual nodes $ V_{NIN}(V_{1,1}^2)[1] $ and $ V_{NIN}(V_{1,1}^2)[2] $ are generated and assigned respectively.
	Then, the nodes from $ V_{1,1}^2 $ to $ V_{NIN}(V_{1,1}^2)[1] $, and from $ V_{NIN}(V_{1,1}^2)[1] $ to $ V_{NIN}(V_{1,1}^2)[2] $ are connected with zero-cost edges.
	The zero-cost edges are represented by dashed lines in Figure \ref{fig:trans6}.\footnotetext{For the sake of simplicity, the illustration only shows serial connections, but it is actually a bit more complicated, which differs from Isaacs or Jang's methods \cite{isaacs2013dubins, jang2016optimal}. The virtual nodes $ V_{nin}(s)[1], \cdots, V_{nin}(s)[p] $ are connected in both directions. See conditions 8 to 10 in Table \ref{table:rule} for details on edge generation related to the virtual nodes.}
	Finally, the starting point of the edge from node $ V_{1,1}^2 $ to the terminal is changed to $ V_{NIN}(V_{1,1}^2)[2] $ while the value of the cost is maintained.
	In this case, the tour of vehicle 2 visits only one task, but it can be assumed that tasks 4 and 5 are also visited via $ V_{NIN}(V_{1,1}^2)[1] $ and $ V_{NIN}(V_{1,1}^2)[2] $, respectively.
	
	To convert the problem into the form of the ATSP, zero-cost edges are created between the sample nodes and $ V_{NIN} $ nodes for each task, depot, and terminal cluster.
	For convenience, assuming that the nodes belonging to each task cluster are arranged according to a user-given rule, the zero-cost edges are generated in an ordered sequence to connect the nodes.
	In this study, the starting node (source node) of a zero-cost edge heading to node $ s $ is referred to as $ s^- $.
	In the ATSP, all nodes must be included in a tour by any arbitrary vehicle.
	Thinking simply with this proposition, the following idea can be formulated:
	Based on the feasible solution shown in Figures \ref{fig:trans6} and \ref{fig:trans7}, all the nodes are traversed along the zero-cost edges after entering one of the nodes in a task cluster, and then come back to the first-entered node and leave the cluster to visit the remaining nodes.
	However, because each node should be visited exactly once and the number of vehicles is assumed to be one in the ATSP, a slight modification to the above idea is necessary.
	For example, for sample nodes $ s_1 $ and $ s_2 $ which belong to different task clusters, the edge from $ s_1 $ to $ s_2 $ in the GHMDATSP instance should be changed because the starting node of the edge should be $ s_1^- $ rather than $ s_1 $ when transforming the instance into an ATSP.
	In addition, to replace the problem from multiple vehicles to a single vehicle, the sample nodes in the terminal cluster and the depot cluster of the next vehicle are connected with zero-cost edges for every vehicle.
	To solve the ATSP-formatted problem, a moderately large number $ M $ is added to each edge that connects the different clusters.
	The ATSP-formatted feasible solution following the above conditions is shown in Figure \ref{fig:trans8}.	
	After transforming the GHMDATSP instance to the form of an ATSP, the number of nodes in the ATSP graph equals the number of all the sample nodes in the GHMDATSP graph.
	The rule for adding edges when transforming the instance of the GHMDATSP into the ATSP is shown in Table \ref{table:rule}.

	\begin{table}[t]
		\centering
		\caption{Transformation rules for assigning weights to edges from the GHMDATSP to the ATSP.}
		\label{table:rule}
		\begin{adjustbox}{max width=\textwidth}
			\begin{tabular}{c|ll|c}
				\hthickline
				\textbf{\#} & \multicolumn{2}{c|}{\textbf{Edge type}} & \textbf{Cost}                   \\\hthickline
				
				\multirow{2}{*}{$ 1 $} & $ s_1^- \in V^k_{(n+1)} \rightarrow s_2 \in V^k_{(n+2)} $ & $ \forall k\in K $ & \multirow{2}{*}{$ M $} \\
				&\multicolumn{2}{l|}{sample nodes in depot(k) $ \rightarrow $ sample nodes in terminal(k)} & \\\hline
				
				\multirow{2}{*}{$ 2 $} & $ s_1^- \in V^k_{(n+2)} \rightarrow s_2 \in C_{k+1,n+1} $ & $ \forall k\in K\notin \{m\}$,\quad $ m $ : last vehicle & \multirow{2}{*}{M} \\
				&\multicolumn{2}{l|}{sample nodes in terminal(k) $ \rightarrow $ sample nodes in depot(k+1)} & \\\hline
				
				\multirow{2}{*}{$ 3 $} & $ s_1^- \in V^m_{(n+2)} \rightarrow s_2 \in C_{1,n+1} $ &           & \multirow{2}{*}{M} \\
				&\multicolumn{2}{l|}{terminal$ (m) \rightarrow $ depot$ (1) $ } & \\\hline
				
				\multirow{2}{*}{$ 4 $} & $ s_1^- \in V^k_{(n+1)} \rightarrow s_2 \in V^k_t $ & $ \forall k \in K, t \in T $ & \multirow{2}{*}{cost$ (s_1,s_2) + M $} \\
				&\multicolumn{2}{l|}{previous node of $ s_1 $ in depot $ (k) \rightarrow $ sample nodes in cluster $ V^k_t $} & \\\hline
				
				\multirow{2}{*}{$ 5 $} & $ s_1^- \in V^k_t \rightarrow s_2 \in V^k_{(n+2)} $ & $ \forall k \in K, t \in T $ & \multirow{2}{*}{cost$ (s_1,s_2) + M $} \\
				&\multicolumn{2}{l|}{previous node of $ s_1 $ in cluster $ V^k_t \rightarrow $ sample nodes in terminal $ (k) $} & \\\hline
				
				\multirow{2}{*}{$ 6 $} & $ s_1^- \rightarrow s_1 \in V^k_t $ & $ k\in K, t\in T \cup D $ & \\
				&\multicolumn{2}{l|}{every nearby nodes for each task cluster, in an ascending order} & \multirow{-2}{*}{0} \\\hline
				
				\multirow{2}{*}{$ 7 $} & $ s_1^- \in V^k_{t_1} \rightarrow s_2 \in V^k_{t_2} $ &  $ k \in K, t_1, t_2 \in T, t_1 \neq t_2 $ &  \\
				&\multicolumn{2}{l|}{previous node of $ s_1 $ in $ V^k_{t_1} \rightarrow $ sample nodes in cluster $ V^k_{t_2} $} &  \multirow{-2}{*}{cost$ (s_1,s_2) + M $}  \\\hline
				
				\multicolumn{4}{c}{}\\
				\multicolumn{4}{c}{If $ s_1 \in V^k_{t_1} $ has $ p $ NINs, ~~~ $ k \in K, ~~t_1, t_2\in T, t_1 \neq t_2, ~~i_1, i_2 \in \{1,\cdots,p\}, i_1 \neq i_2 $} \\\hline
				\multirow{2}{*}{$ 8 $} & $ s_1^- \rightarrow V^{NIN}_{s_1}[i] $ & & \\
				&\multicolumn{2}{l|}{previous node of $ s_1 $ in $ V^k_{t_1} \rightarrow $ NIN nodes of $ s_1 $} & \multirow{-2}{*}{M} \\\hline
				
				\multirow{2}{*}{$ 9 $} & $ V^{NIN}_{s_1}[i_1]^- \rightarrow V^{NIN}_{s_1}[i_2] $ &  &  \\
				&\multicolumn{2}{l|}{edge between NIN nodes of $ s_1 $} & \multirow{-2}{*}{M}                  \\\hline
				
				\multirow{2}{*}{$ 10 $} & $ V^{NIN}_{s_1}[i]^- \rightarrow s_2 \in V^k_{t_2} $ &  &  \\
				&\multicolumn{2}{l|}{previous node of every NIN nodes of $ s_1 \rightarrow $ sample nodes in cluster $ V^k_t $} & \multirow{-2}{*}{cost$ (s_1,s_2) + M $}\\\hthickline      
			\end{tabular}
		\end{adjustbox}
	\end{table}
	
	\subsection{Solving the ATSP and Inversely Transforming it into the GHMDATSP}\label{subsec:proc3}
	
	After converting the instance into the ATSP, we solve the instance.
	To get a high-performance solution in a short time, we exploited the well known state-of-the-art TSP solver which is based on the Lin-Kernighan-Helsgaun (LKH) heuristic \cite{helsgaun2015solving}.
	The Lin-Kernighan (LK) algorithm is the algorithm for solving the symmetric TSP by swapping pairs of sub-tours to make a new tour based on a k-opt which is a generalization of 2-opt and 3-opt.
	The LKH heuristic is a modified version of the LK algorithm which is highly effective in the sense of a computational complexity and the quality of the solution.	
	The solution of the ATSP obtained through the solver is the Hamiltonian path starting from the depot of the first vehicle, going through all the sample nodes, the virtual sample nodes, and the terminal point of the last vehicle, and then into the depot of the first vehicle.
	
	Then the obtained solution is inversely transformed to the form of the GHMDATSP.
	The cost of an edge forming the solution of the ATSP is 0, $ M $, or `cost$(s_1, s_2) + M $', and the meaningful edges are those whose cost is not zero or $ M $.
	In addition, since the destination part of the edges generated in the ATSP transformation process corresponds to the actual nodes, the nonzero-cost edges' destination nodes in the ATSP solution are the actual visiting nodes in the GHMDATSP.
	Finally, the inverse transformation process ends by first disconnecting the edges from the terminal cluster to the depot cluster to separate the solution and second assigning it to each vehicle.
	
	The rest of this section describes the following theorem.\footnotetext[1]{In this section, the form of proof is referenced from \cite{oberlin2009transformation}.}
	
	\begin{thm}\label{thm:opt}
		Given the optimal tour $ tour_{ATSP}^* $ on the transformed ATSP graph, a set of tours $ tour^{1*}, \cdots, tour^{m*} $ can be constructed which is the optimal for the GHMDATSP.
	\end{thm}
	
	Without loss of generality, $ tour_{ATSP}^* $ starts from the first sample node of vehicle 1's depot cluster, $ V^1_{(n+1),1} $ and each vehicle starts from the first sample node of its depot cluster, i.e. $ V^k_{(n+1),1} $ for vehicle $ k $.
	For the proof of Theorem \ref{thm:opt}, a list of facts about $ tour_{ATSP}^* $ is given in the following lemma.
	
	\begin{lemma}\label{lem:1}
		The optimal tour $ tour_{ATSP}^* $ on the transformed ATSP graph satisfies the following conditions.
		\begin{itemize}
			\item[a.] For each cluster $ V_t ~~ \forall t\in T\cup D $, both the in-degree and out-degree are 1. In other words, there is only a single edge each for entering and leaving the cluster.
			\item[b.] Let $ V^k_{t,i} $ be the first visited sample node in cluster $ V_t $. After $ V^k_{t,i} $, $ tour_{ATSP}^* $ visits all the remaining nodes in cluster $ V_t $ before leaving it.
			\item[c.] For vehicle $ k $, there is no sample node of any other vehicle's depot and terminal cluster between the first sample node of cluster $ V^k_{n+1} $ and the last sample node of cluster $ V^k_{n+2} $.
		\end{itemize}
	\end{lemma}

	\begin{proof}
		The cost of any edge edge between different clusters has an additional cost $ M $ on it. 
		Suppose that $ V^k_{t,i} $ is chosen among the nodes in cluster $ V_t $ as a sink node of the edge toward cluster $ V_t $.
		From condition 6 in Table \ref{table:rule}, $ tour_{ATSP}^* $ visits every remaining sample node along zero-cost edges.
		After reaching all of the nodes, $ tour_{ATSP}^* $ leaves cluster $ V_t $.
		If $ tour_{ATSP}^* $ leaves cluster $ V_t $ before visiting every node in it, then the remaining nodes in $ V_t $ can only be visited by one of the edges towards $ V_t $ with the cost at least greater than or equal to $ M $.
		Since $ M $ is a large number, $ tour_{ATSP}^* $ should contain the least number of edges connecting different clusters.
		Therefore claim \textit{a.} and \textit{b.} hold.
		
		The sample nodes in vehicle $ k $'s depot cluster do not have any edge towards any other vehicle's depot and terminal clusters. 
		The sample nodes in vehicle $ k $'s terminal cluster have edges only towards the next vehicle's depot cluster.
		This holds for every vehicle; hence claim \textit{c.} is true.
	\end{proof}
	
	\begin{remark}\label{remark}
		From claim \textit{c}. in Lemma \ref{lem:1}, one of the edges between vehicle $ k $'s terminal cluster and the next vehicle's depot cluster is selected.
		By changing each selected edge's sink node (or destination node) to the first sample node of vehicle $ k $'s depot cluster in $ tour_{ATSP}^* $, we obtain $ m $ disconnected directed tours denoted as $ tour_{ATSP}^{1*}, \cdots, tour_{ATSP}^{m*} $.
		Furthermore, $ tour^k $ is a tour for vehicle $ k $ in the form of a GHMDATSP which is inversely transformed from $ tour_{ATSP}^{k*} $.
	\end{remark}
	
	\begin{lemma}\label{lem:2}
		Given the optimal tour $ tour_{ATSP}^* $ for the transformed graph, the cost of $ tour_{ATSP}^* $, denoted as $ \textrm{Cost}(tour_{ATSP}^*) $, equals  $ \sum_{k=1}^{m} \textrm{Cost}^k + M(n+2m) $ where $ \sum_{k=1}^{m} \textrm{Cost}^k $ is the cost sum of every vehicle.
	\end{lemma}
		
	\begin{proof}
		Let $ n^k $ denote the number of distinct task clusters visited by vehicle $ k $.
		If $ n^k > 0 $, $ tour^k $ can be represented as $ \left\{ V^k_{(n+1), i_0}, V_{t_1, i_1}, V_{t_2, i_2}, \cdots, V_{t_{n^k}, i_{n^k}}, V^k_{(n+2),i_{n^k+1}} \right\} $ which starts from the depot and ends at the terminal.
		The cost relationship between $ tour_{ATSP}^{k*} $ and $ tour^k $ is as follows.
		\begin{align}\label{eq:costRel1}
			\textrm{Cost}\left( tour_{ATSP}^{k*} \right) = & \left( \underbrace{c\left(V^k_{(n+1), i_0}, V^k_{t_1, i_1}\right) + M}_\text{from depot to task} \right)
												+ \left( \underbrace{\sum_{j=1}^{n^k-1} c(V^k_{t_{j},i_{j}}, V^k_{t_{j+1},i_{j+1}}) + M}_\text{from task to task} \right) \nonumber \\
		   						   & + \left( \underbrace{ c\left(V^k_{t_{n^k},i_{n^k}}, V^k_{(n+2),i_{n^k+1}}\right) + M}_\text{from task to terminal} \right)
		   						   	 + \underbrace{M}_\text{from terminal to depot} \nonumber \\
								 = & \textrm{Cost}\left( tour^k \right) + M\left( n^k + 2 \right) \nonumber \\
								 = & \textrm{Cost}^k + M\left( n^k + 2 \right)
		\end{align}		
		Otherwise, if $ n^k $ is 0, $ tour^k $ consists of two directed edges $ \left(V^k_{(n+1),i_1}, V^k_{(n+2),i_2}\right) $ and $ \left(V^k_{(n+2),i_2}, V^k_{(n+1),i_1}\right) $ where each cost is $ M $, respectively.
		
		As the optimal tour $ tour_{ATSP}^* $ visits each of the sample nodes exactly once, any task $ t \in T $ belongs to exactly one of the vehicles.
		Therefore a set of tours $ \{tour^1, \cdots, tour^m\} $ is a feasible solution to the GHMDATSP.
		The cost relationship between the solution of ATSP and GHMDATSP as follows.
		\begin{align}\label{eq:costRel2}
			\textrm{Cost}\left( tour_{ATSP}^{*} \right) = & \sum_{k=1}^{m} \textrm{Cost}\left( tour_{ATSP}^{k*} \right) \nonumber\\
						= & \sum_{k=1}^{m} \begin{bmatrix}
							\left( c\left(V^k_{(n+1), i_0}, V^k_{t_1, i_1}\right) + M\right)+ \left( \sum_{j=1}^{n^k-1} c(V^k_{t_{j},i_{j}}, V^k_{t_{j+1},i_{j+1}}) + M \right) \\
							+ \left( c\left(V^k_{t_{n^k},i_{n^k}}, V^k_{(n+2),i_{n^k+1}}\right) + M \right) + M
						\end{bmatrix} \nonumber\\
						= & \sum_{k=1}^{m} \textrm{Cost}^k + M(n+2m)
		\end{align}	
	\end{proof}

	\begin{lemma}\label{lem:3}
		Given a set of optimal tours $ \{tour^{1*}, \cdots, tour^{m*}\} $ for the GHMDATSP, a feasible solution $ tour_{ATSP} $ for the transformed graph can be constructed such that $ \sum_{k=1}^{m} \textrm{Cost}^{k*} $ equals $ \textrm{Cost}(tour_{ATSP}) - M(n+2m) $.
	\end{lemma}

	\begin{proof}
		If vehicle $ k $ does not visit any sample node, let $ tour_{ATSP}^k $ consist of every sample node in cluster $ V^k_{n+1} $ and $ V^k_{n+2} $ such as $ \left\{ V^k_{(n+1),1}, \cdots, V^k_{(n+1),end}, V^k_{(n+2),1},\right.$  $\left. \cdots, V^k_{(n+2),end} \right\} $ where $ V^k_{(n+1),end} $ and $ V^k_{(n+2),end} $ are the last sample node in each cluster.
		Otherwise, if vehicle $ k $ visits $ n^k $ sample nodes such that $ tour^k $ is  $ \left\{ V^k_{(n+1), i_0}, V_{t_1, i_1}, V_{t_2, i_2}, \cdots, V_{t_{n^k}, i_{n^k}}, V^k_{(n+2),i_{n^k+1}} \right\} $, construct $ tour_{ATSP}^k $ as follows:
		\begin{enumerate}
			\item Set $ tour_{ATSP}^k $ as an empty set.
			\item Add $ V^k_{(n+1), i_0} $ to $ tour_{ATSP}^k $.
			\item For every node in the depot cluster except $ V^k_{(n+1), i_0} $, add and connect to the last element of $ tour_{ATSP}^k $ with a zero-cost edge.
			\item Add $ V_{t_1, i_1} $ and connect to the last element of $ tour_{ATSP}^k $ with cost \\$ c\left( V^k_{(n+1), i_0}, V_{t_1, i_1} \right) + M $.
			\item For $ j $ from 1 to $ n^k-1 $  \begin{enumerate}
				\item For every node in cluster $ V_{t_{j}} $ except $ V_{t_{i},i_{j}} $, add and connect to the last element of $ tour_{ATSP}^k $ with a zero-cost edge.
				\item Add $ V_{t_{i+1},i_{j+1}} $ and connect to the last element of $ tour_{ATSP}^k $ with cost $ c\left( V_{t_{i},i_{j}}, V_{t_{i+1},i_{j+1}} \right) + M $.
			\end{enumerate}
			\item For every node in cluster $ V_{t_{n^k}} $ except $ V_{t_{n^k}, i_{n^k}} $, add and connect to the last element of $ tour_{ATSP}^k $ with a zero-cost edge.
			\item Add $ V^k_{(n+2),i_{n^k+1}} $ and connect to the last element of $ tour_{ATSP}^k $ with cost $ c\left( V_{t_{n^k}, i_{n^k}}, V^k_{(n+2),i_{n^k+1}} \right) + M $.
			\item For every node in the terminal cluster, add and connect to the last element of $ tour_{ATSP}^k $ with a zero-cost edge starting from $ V^k_{(n+2),i_{n^k+1}} $.
		\end{enumerate}
		Then, $ tour_{ATSP} $ consists of $ tour_{ATSP}^1 $ to $ tour_{ATSP}^m $, and edges with cost $ M $ are added between each vehicle's ATSP tour which represent edges from terminal cluster to the next vehicle's depot cluster.
		By using the same arguments of Lemma \ref{lem:2}, $ \sum_{k=1}^{m} \textrm{Cost}^{k*} = \textrm{Cost}(tour_{ATSP}) - M(n+2m) $.
	\end{proof}

	\begin{proof}[Proof of \textbf{Theorem \ref{thm:opt}}]
		Let the set of GHMDATSP tours, $ \{tour^1, \cdots, tour^m\} $, be constructed from the optimal solution of transformed graph, $ tour_{ATSP}^* $, as in Remark \ref{remark}.
		\begin{align}\label{eq:costRel3}
			\sum_{k=1}^{m} \textrm{Cost}^k = & \textrm{Cost}\left( tour_{ATSP}^{*} \right) - M(n+2m) & (\textrm{Lemma \ref{lem:2}}) \nonumber\\
										\leq & \textrm{Cost}\left( tour_{ATSP} \right) - M(n+2m) & \nonumber\\
										   = & \sum_{k=1}^{m} \textrm{Cost}^{k*} & (\textrm{Lemma \ref{lem:3}})
		\end{align}	
		With the relationship above, the set of tours, $ \{tour^1, \cdots, tour^m\} $, is the optimal for the GHMDATSP.
	\end{proof}

	\subsection{Path Refinement}\label{subsec:proc4}
	
	\begin{algorithm*}[]
		\caption{Path Refinement}
		\small
		\begin{algorithmic}[1]
			\Procedure{PathRefinement (\textbf{\textit{path}})}{}
			\State $ unvisited = \{1,\cdots,n\}$
			\For{$ k := 1 $ to $ m $}
			\State $ visitOrder^k \leftarrow \emptyset $, $ visitState^k \leftarrow \emptyset $
			\If {$ path^k = \emptyset $}
			\State continue \textbf{for} loop
			\EndIf			
			
			\Algphase{Phase 1 - Find visiting order and states of given path}
			\State $ visitOrder^k \leftarrow $ \{n+1\} \COMMENT {Save depot information}
			\State $ visitState^k \leftarrow \{path^k[1]\} $
			\ForEach {state $ \textbf{s} $ in $ path^k $} \COMMENT{Path from depot to terminal}
			\ForEach{$ t $ in $ unvisited $}
			\If {footprint touches task $ t $ when vehicle$ ^k $ at $ \textbf{s} $}
			\State $ unvisited \leftarrow unvisited \setminus \{t\} $
			\State $ visitOrder^k \leftarrow visitOrder^k \cup \{t\} $
			\State $ visitState^k \leftarrow visitState^k \cup \{\textbf{p}\} $
			\EndIf
			\EndFor
			\If {$ unvisited = \emptyset $}
			\State break \textbf{for each} loop
			\EndIf
			\EndFor
			\State $ visitOrder^k \leftarrow $ \{n+2\} \COMMENT {Save terminal information}
			\State $ visitState^k \leftarrow \{path^k[\textrm{end}]\} $
			\State $ numState^k = |visitState^k| $	\COMMENT{Number of states to be optimized}
			\Algphase{Phase 2 - Find locally optimal states}
			\Do
			\For {$ i := 1$:$2$:$numState^k, 2$:$2$:$numState^k $}
			\State $ t \leftarrow visitOrder^k[i] $
			\State $ \textbf{p}_{\textrm{prev}} \leftarrow visitState^k [i$-1$] $
			\State $ \textbf{p}_{\textrm{next}} \leftarrow visitState^k [i$+1$] $
			\If {$ t = $ n+1}	\COMMENT{Depot}						
			\State $ \textbf{p}_{\textrm{opt}} \leftarrow localOpt(\textrm{depot}^k, \textrm{vehicle}^k, \textbf{p}_{\textrm{next}}) $
			\ElsIf {$ t = $ n+2}	\COMMENT{Terminal}
			\State $ \textbf{p}_{\textrm{opt}} \leftarrow localOpt(\textrm{terminal}^k, \textrm{vehicle}^k, \textbf{p}_{\textrm{prev}}) $
			\Else	\COMMENT{other tasks}
			\State $ \textbf{p}_{\textrm{opt}} \leftarrow localOpt(\textrm{terminal}^k, \textrm{vehicle}^k, (\textbf{p}_{\textrm{prev}}, \textbf{p}_{\textrm{next}})) $
			\EndIf
			\State $ visitState^k[i] \leftarrow \textbf{p}_{\textrm{opt}} $
			\EndFor
			\State $ path^k \leftarrow $ continuous path from $visitState^k$
			\doWhile {cost of $ path^k $ not converged}
			\EndFor
			\RETURN \textit{\textbf{path}}
			\EndProcedure
		\end{algorithmic}
		\label{alg:PathRefinement}
	\end{algorithm*}

	As an additional step of the proposed procedure, this section suggests an approach to refine the paths for each vehicle to improve the quality of the solution.
	The output of the sampling-based methods for the GHMDATSP is obtained through a limited number of crudely discretized samples in the 3-dimensional space.
	In other words, even if the optimal solution is obtained for a given roadmap, there is some quality difference from the actual optimum solution with given conditions.
	Therefore, to reduce the difference, local optimization is applied to the outputs from the suggested method.
	When performing the local optimization, the parameters are optimized in the continuous state space while assuming that the newly refined path follows the sequence of visiting each task from the given solution.
	The improvement of the solution quality through this step might be limited since the visiting sequence of the vehicle does not change.
	However, the global optimal solution of the instance can be obtained if the quality of the given solution from the algorithms above is sufficiently high.
	In addition, there is an advantage in that the total calculation time can be greatly reduced compared to simply increasing the total number of samples in the instance.
	
	Path refinement proceeds as follows. Since the output of the previous step may not have sample nodes for all tasks due to the concept of the necessarily intersecting region, the visiting sequence and states when the sensor's footprint touches tasks are sequentially generated for each vehicle based on the set of paths in the given solution.
	Considering the states of the depots and the terminals, the total number of states to be optimized is $ n+2m' $ where n is the number of tasks, and $ 2m' $ corresponds to the number of initial and final states for each vehicle assigned at least a single task.
	For the local optimization of each task state, the constraints of the state to be optimized and the neighboring states are required.
	Each is optimized in the direction of decreasing the solution cost where the neighboring states are fixed.\footnotetext[2]{The `fmincon' function in MATLAB is used for the local optimization.}
	Similarly, the depot state is optimized fixing the next state and the terminal state fixing the previous state.
	Optimization is repeated in an alternating order as follows: odd-numbered states are optimized while others are fixed and then even-numbered states are optimized, and the iteration repeats until the cost of the vehicle converges.
	In the simulation, iteration was performed until the difference between the previous cost and the next cost was less than 0.01\%.
	The pseudo-code of this step is shown in Algorithm \ref{alg:PathRefinement}.

	\section{Numerical Experiments}\label{sec:sim}
	
	The aim of this section is to answer the following questions:
	\begin{enumerate}
		\item How does the performance differ by applying the necessarily intersecting node and the path refinement process?
		\item How does the result depend on the orientation and the size of the sensor?
		\item How does the computational time depend on the size of the instance?
		\item When operating multiple vehicles, how do the results differ when the characteristics of the vehicle are equal or not equal?
	\end{enumerate}

	Rather than comparing simulation results for various instances at once, we change the key factors or parameters that affect the problem one by one to understand the tendency of the results of path generation.
	We compare three methods:
	\begin{enumerate}
		\item noNIN: Path generation, without considering the \textit{necessarily intersecting node} (from Section \ref{subsec:proc1} to \ref{subsec:proc3}).
		\item NIN: Path generation, applying the \textit{necessarily intersecting node} (from Section \ref{subsec:proc1} to \ref{subsec:proc3}).
		\item NINPR: Path generation and refinement on the solution of the NIN (Section \ref{subsec:proc4}).
	\end{enumerate}
	The algorithms were implemented in MATLAB environment on a PC with an Intel(R) Core$ ^{\textrm{TM}} $ i7-6700K 4.00GHz CPU and 16.0GB RAM.
	The ATSP instances were solved by LKH\footnotetext[1]{The C code of the LKH heuristic was mex-compiled to use in MATLAB environment.}, and the computational times are expressed in seconds.
	When the instance is for a single vehicle, the noNIN method equals the algorithm proposed in \cite{obermeyer2012sampling}.
	
	\subsection{noNIN, NIN, and NINPR with different sensors}\label{subsec:simul1}

	\begin{figure*}[]
		\centering
		\captionsetup{justification=centering}
		\begin{minipage}{\linewidth}\centering
			\subfloat[Omni-directional, noNIN.]{\label{fig:simul1-1}\includegraphics[width=.325\linewidth]{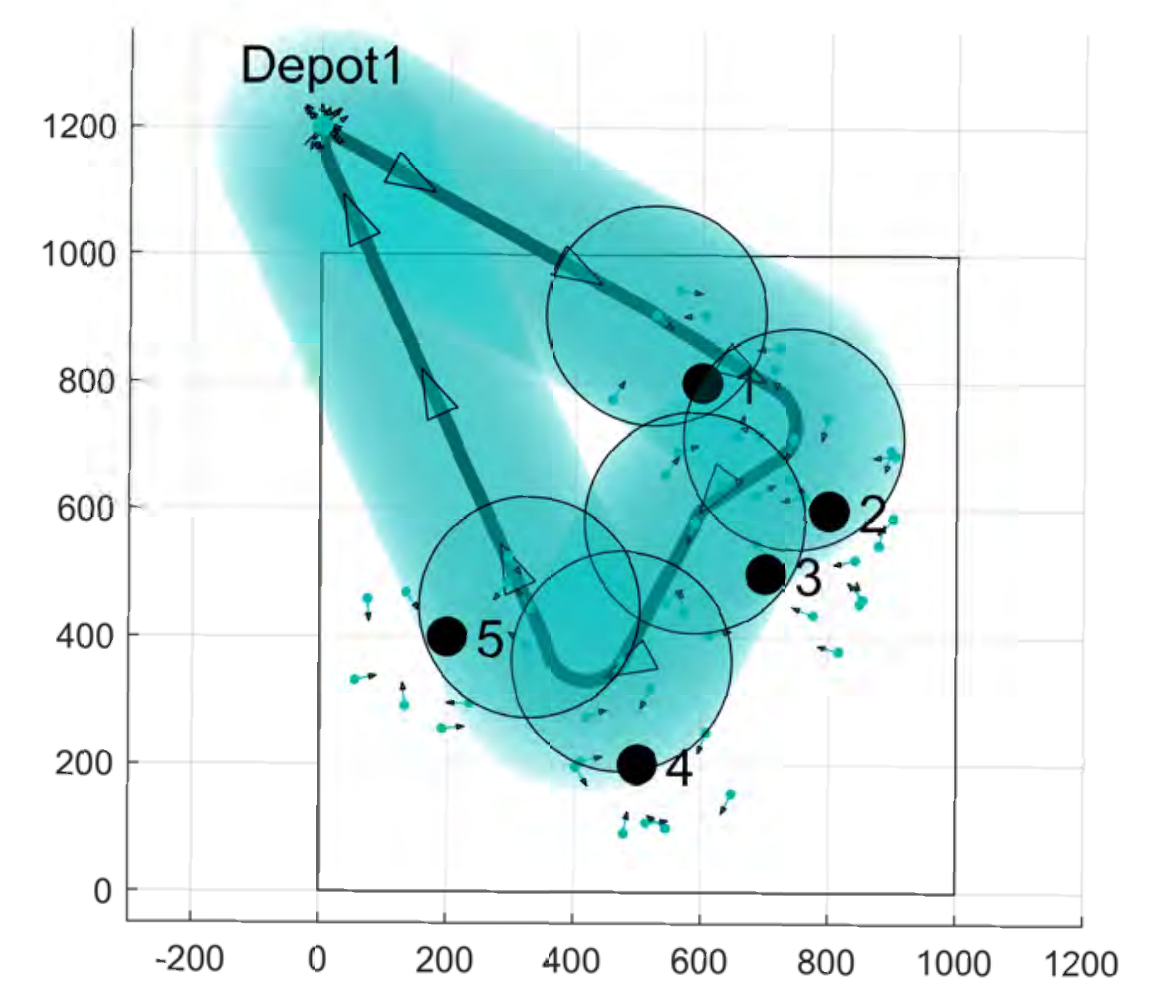}}
			\subfloat[Omni-directional, NIN.]{\label{fig:simul1-2}\includegraphics[width=.325\linewidth]{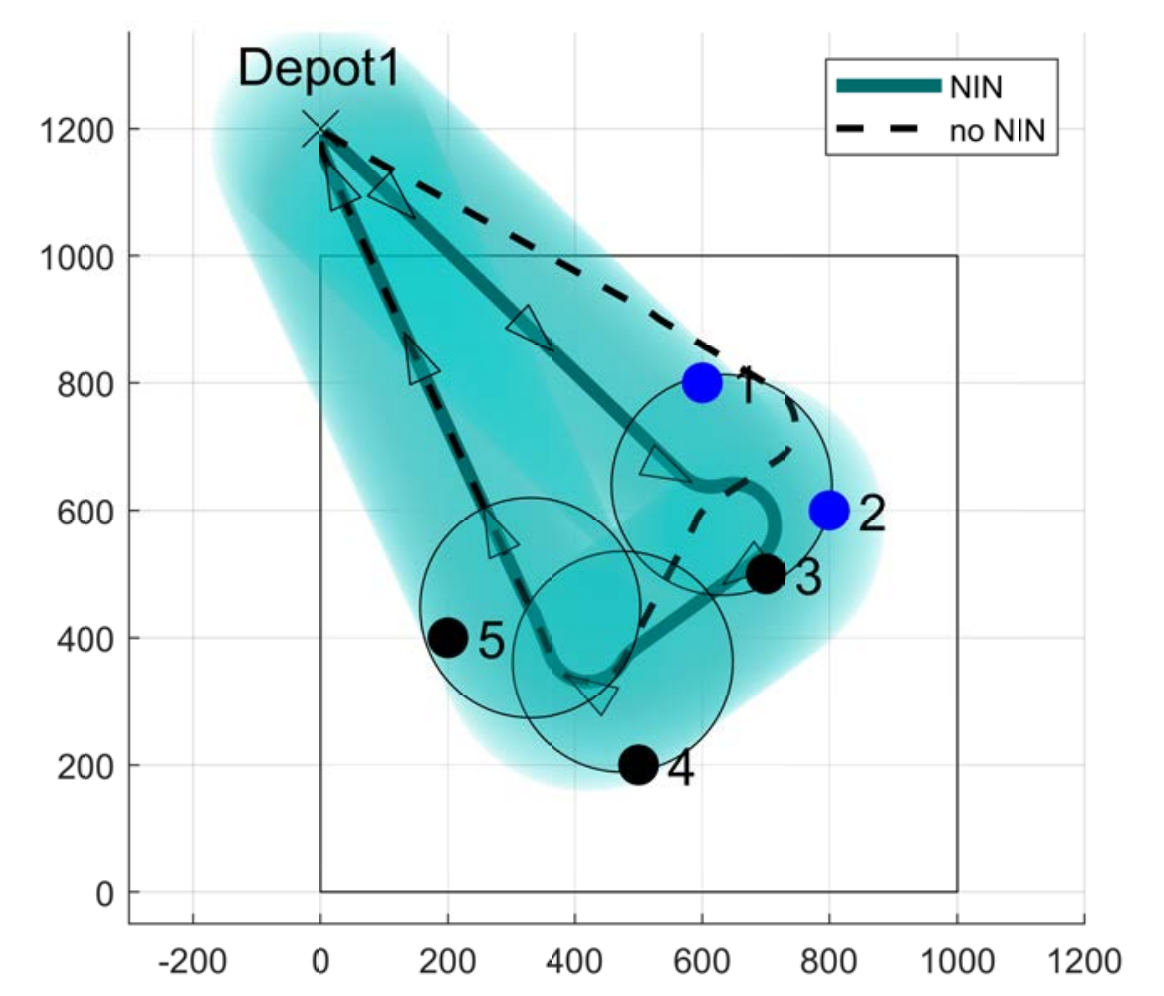}}
			\subfloat[Omni-directional, NIN-PR.]{\label{fig:simul1-3}\includegraphics[width=.325\linewidth]{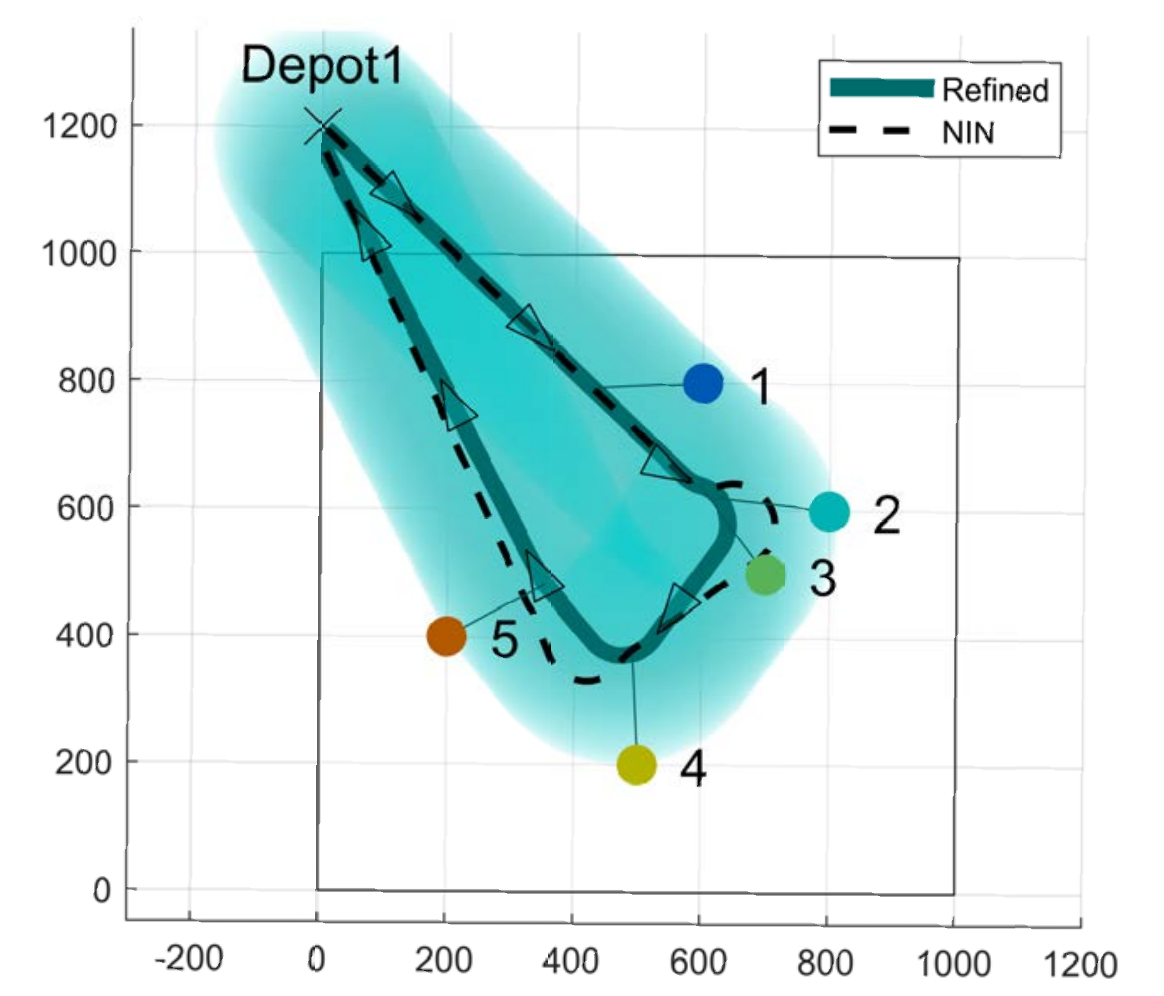}}
		\end{minipage}
		\begin{minipage}{\linewidth}\centering
			\subfloat[Forward, noNIN.]{\label{fig:simul1-4}\includegraphics[width=.325\linewidth]{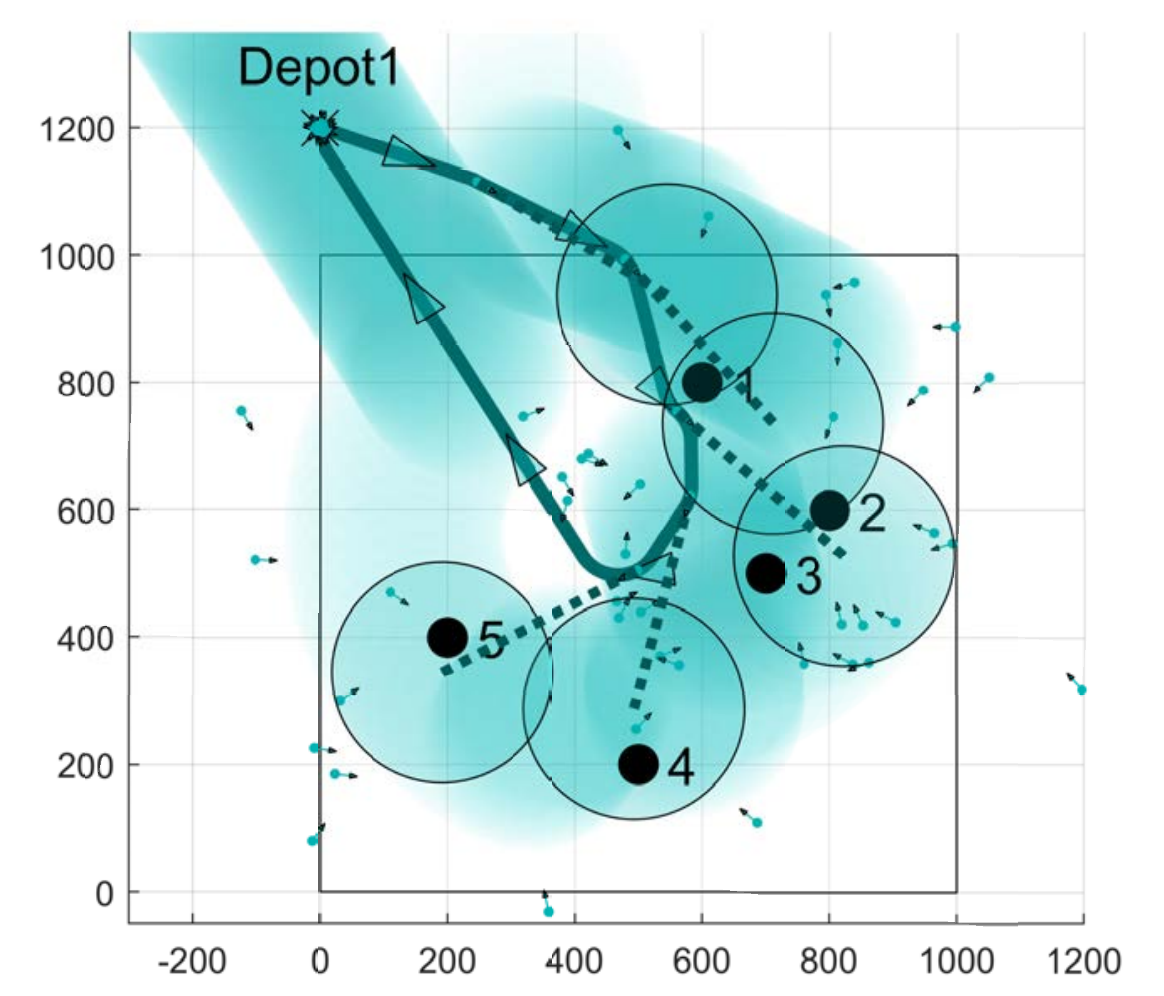}}
			\subfloat[Forward, NIN.]{\label{fig:simul1-5}\includegraphics[width=.325\linewidth]{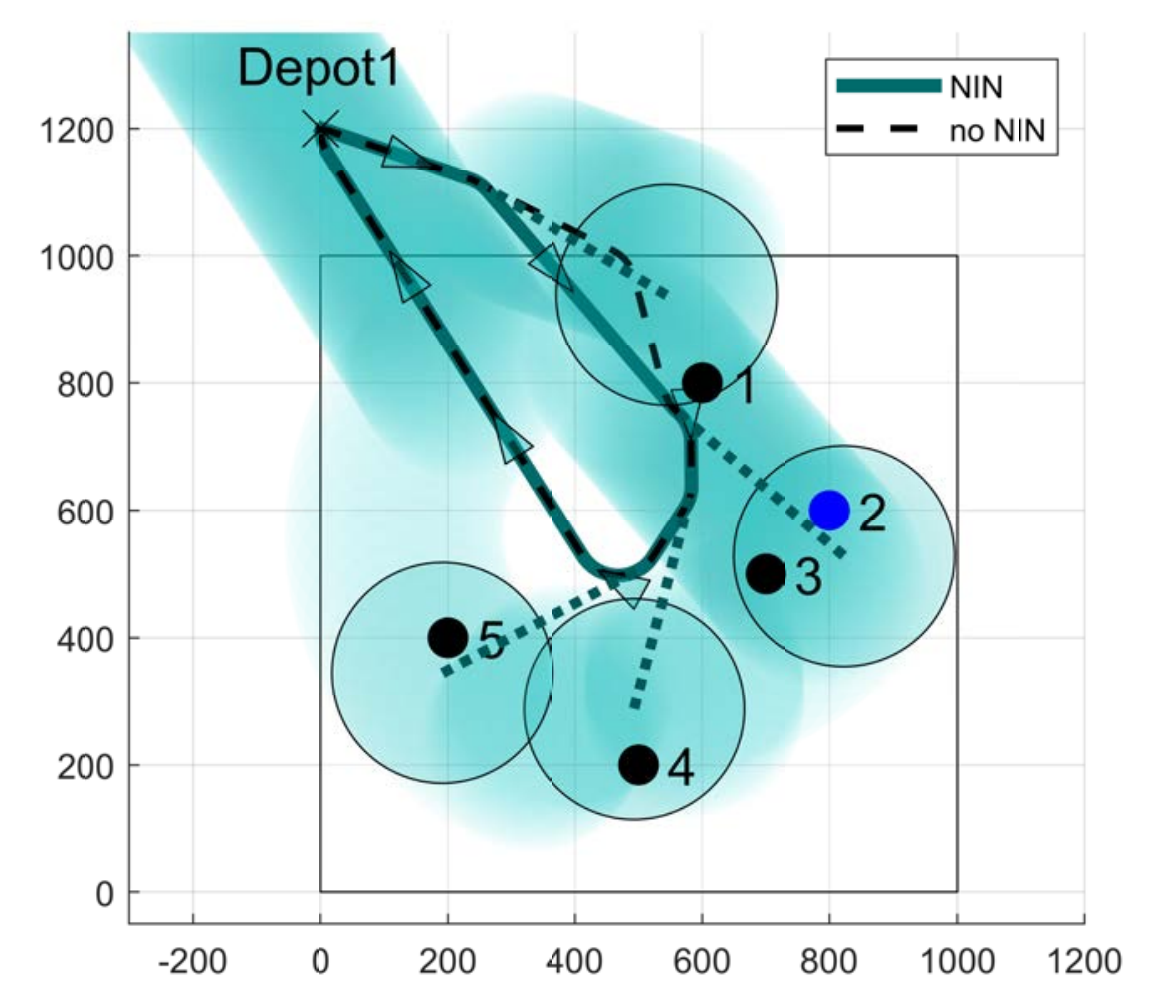}}
			\subfloat[Forward, NIN-PR.]{\label{fig:simul1-6}\includegraphics[width=.325\linewidth]{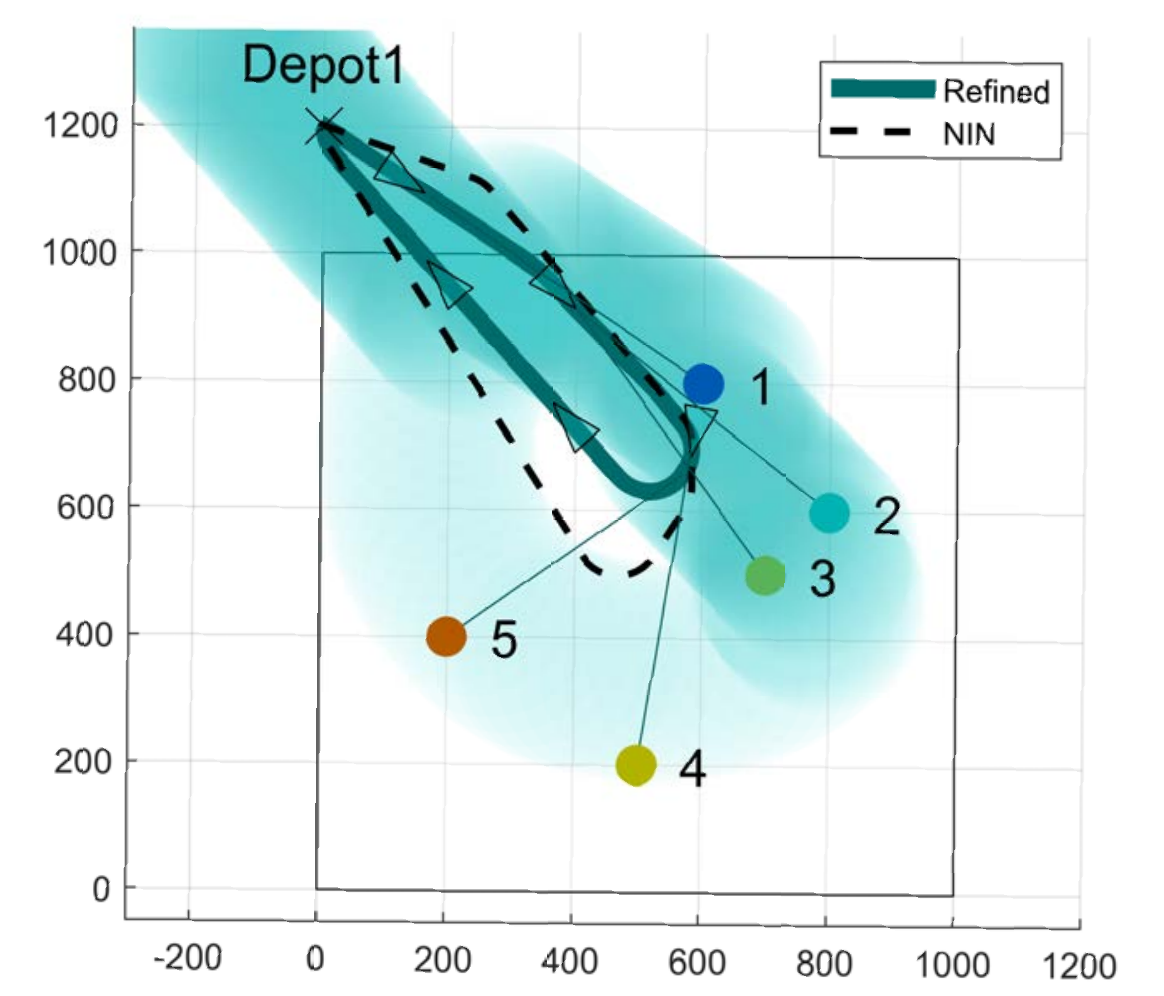}}
		\end{minipage}
		\begin{minipage}{\linewidth}\centering
			\subfloat[Rightward, noNIN.]{\label{fig:simul1-7}\includegraphics[width=.325\linewidth]{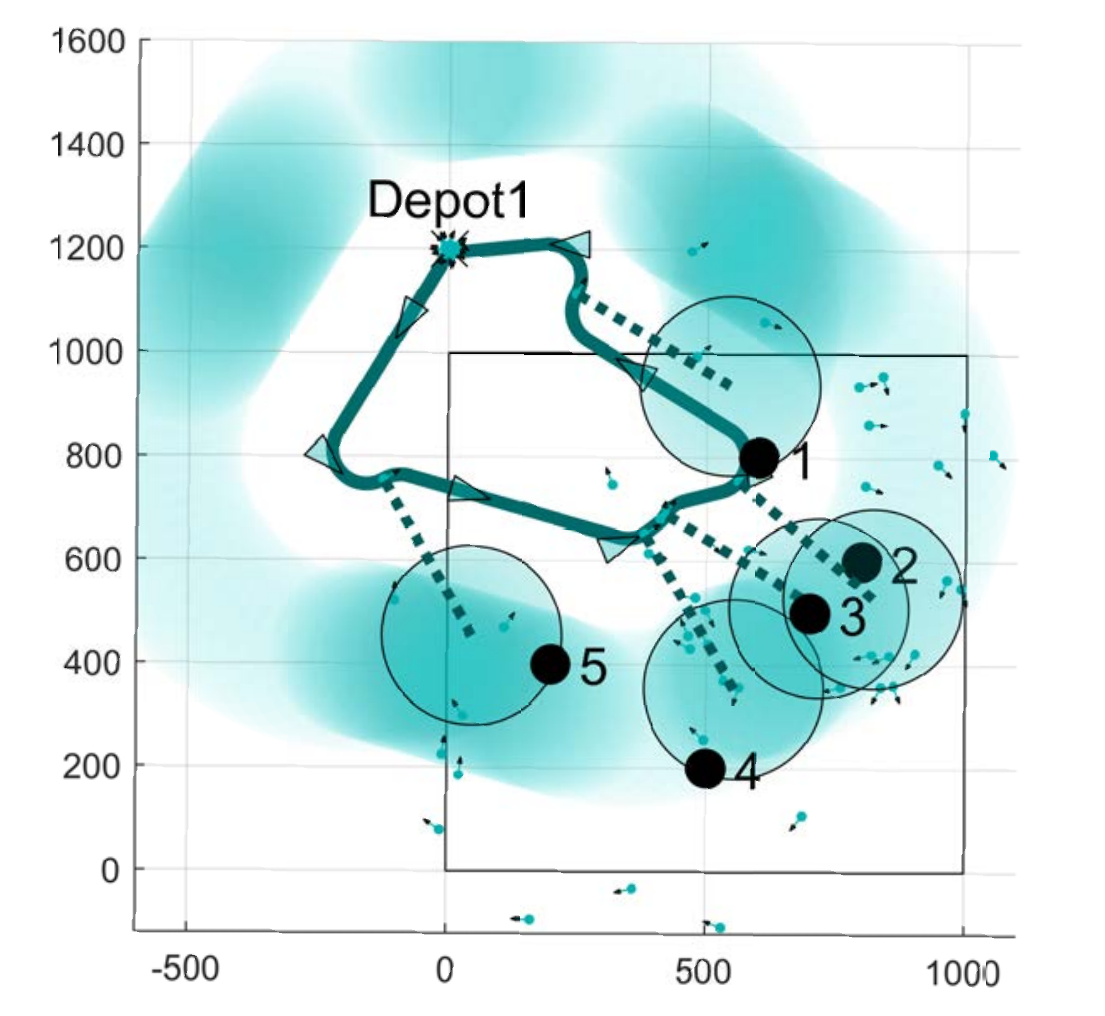}}
			\subfloat[Rightward, NIN.]{\label{fig:simul1-8}\includegraphics[width=.325\linewidth]{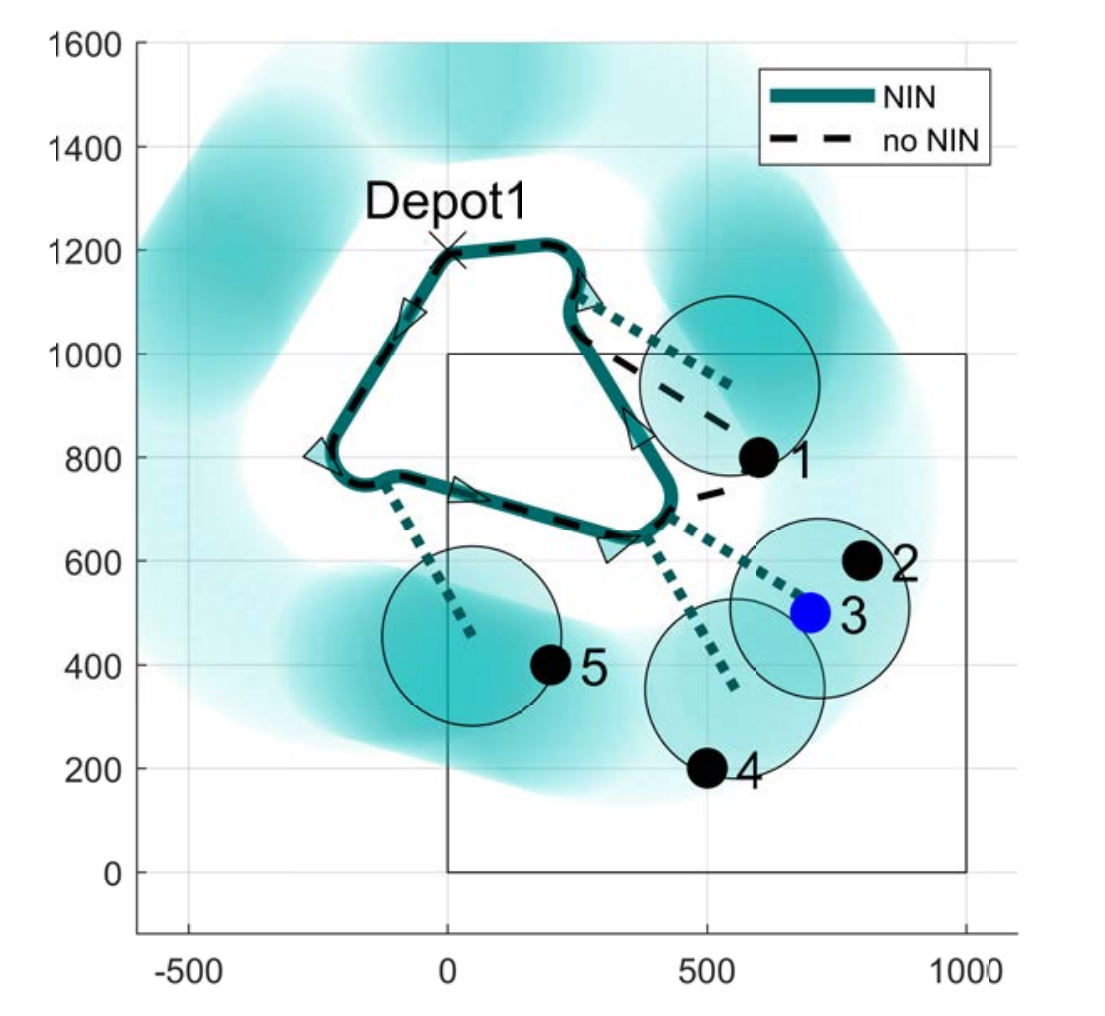}}
			\subfloat[Rightward, NIN-PR.]{\label{fig:simul1-9}\includegraphics[width=.325\linewidth]{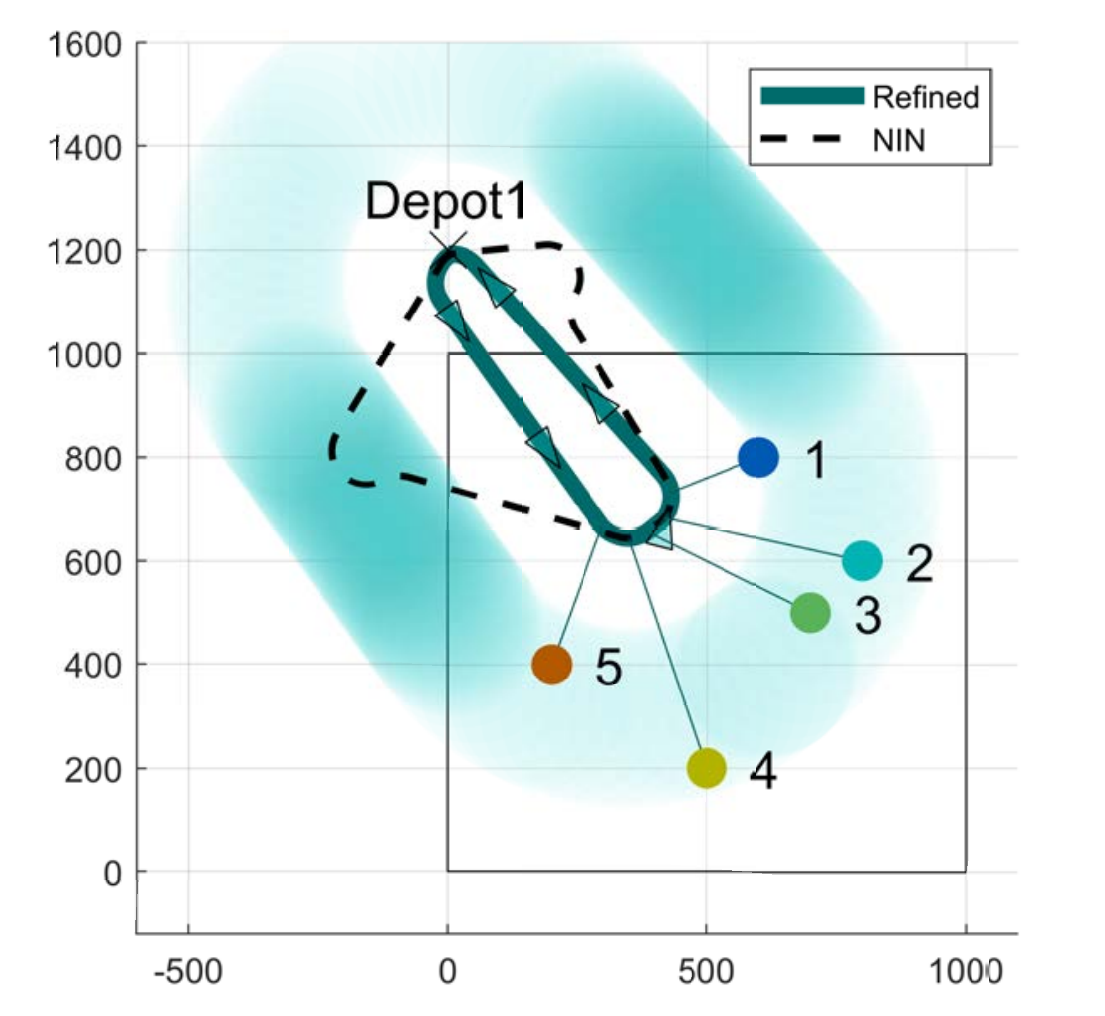}}
		\end{minipage}
		\caption{Comparison among the results of noNIN, NIN, and NINPR.}
		\label{fig:simulComp1}
	\end{figure*}
	
	Each subfigure in Figure \ref{fig:simulComp1} deals with the result of the problem to find the optimal tour where the aerial vehicle departs from the depot, traverses five tasks, and returns to the original location.
	The parameters and assumptions for each instance in Figure \ref{fig:simulComp1} are as follows. 
	The speed of the vehicle is assumed to be 50$ m/s $ and the minimum turning radius is about 66$ m $ when the load factor is fixed as 4. 
	When the operating altitude is 300$ m $ and the minimum and maximum angles between the nadir and vector from the sensor to the footprint, simply a footprint, of the sensor are 30 and 60 degrees, respectively, the minimum and maximum distances from the nadir point of the vehicle to the footprint are approximately 173m and 520m, respectively.
	To simplify the problem, the shape of the sensor foot is assumed to be a circle, so the radius of the footprint is about 173m. For the consistency of the instance, the radius of the sensor foot of the omni directional sensor is also assumed to be 173$ m $.
	The depot is located at $ (0, 1200) $ and the region of interest in a two-dimensional plane is $ [0,1000] \times [0,1000] $.
	
	Figures \ref{fig:simul1-1} to \ref{fig:simul1-3} contain the results of the vehicle equipped with the omni directional sensor.
	Figure \ref{fig:simul1-1} shows the solution obtained by the noNIN method for the instances where 10 sample nodes that have the vehicle's location and heading information are generated for each task, depot, and terminal based on the Halton sequences (length: 2,404).
	The Halton sequences are known to be deterministic, but they are of low-discrepancy that the sequences are more evenly distributed compared to the outputs generated based on the normal distribution.
	Black dots indicate the position of randomly generated tasks within the region of interest, and each sample node is indicated by a small arrow. 
	The five large circles in the figure represent the footprint when the vehicle passes the selected sample node in the solution for each task cluster, and the lightly shaded region along the path represents the trace of the footprint.
	The order of task visits is 1-2-3-4-5.
	The instance in Figure \ref{fig:simul1-2} has the same set of sample nodes as the one in Figure \ref{fig:simul1-1}, and the path obtained through the NIN is indicated by a bold solid line, while the path from the noNIN is indicated by a dashed line (cost: 2,404 $ \rightarrow $ 2,340).
	Only three footprints are shown in Figure \ref{fig:simul1-2}, because the necessarily intersecting region of the sample node of task 3 selected in the solution contains the locations of task 1 and task 2.
	Thus, the vehicle in Figure \ref{fig:simul1-2} can perform all tasks along the generated path without directly visiting the sample nodes of task 1 and task 2.
	Figure \ref{fig:simul1-3} shows the refined path from the path in Figure \ref{fig:simul1-2} indicated by a bold solid line where the solution quality is improved by applying the procedure in Section \ref{subsec:proc4} (cost: 2,340 $ \rightarrow $ 2,131).
	The path in Figure \ref{fig:simul1-2} is shown as a dashed line for comparison.
	In addition, thin solid lines are drawn between the vehicle position and the task position at the moment each task meets the footprint.
	Figure \ref{fig:simul1-3} shows that the result from the NINPR efficiently reduces the cost while visiting all given tasks and the constraints caused by the sensor and the vehicle dynamics.
	
	Figures \ref{fig:simul1-4} to \ref{fig:simul1-6} show the simulation results when the mounted sensor detects the forward direction of the vehicle.
	As in the case of the omni-directional sensor, 10 sample nodes are generated based on the Halton sequences for each task, depot and terminal cluster, and the result obtained through the noNIN is shown in Figure \ref{fig:simul1-4}.
	In the same way as above, the order of task visits is 1-2-3-4-5 (cost: 1,952).
	The path in Figure \ref{fig:simul1-5} is the result from the NIN.
	Because the chosen sample node of task 3 contains the location of task 2 in its necessarily intersecting region, the vehicle skips the sample nodes in task 2 and visits only the sample nodes of task 1 and task 3 (cost: 1,952 $ \rightarrow $ 1,913).
	As shown in Figure \ref{fig:simul1-3}, the path obtained by the NINPR from the result through the NIN (dashed line) is shown by a bold solid line, and it is confirmed that the path is refined with a satisfactory performance (cost: 1,913 $ \rightarrow $ 1,665).
	
	Figures \ref{fig:simul1-7} to \ref{fig:simul1-9} represent the results when the sensor detects in a direction perpendicular to the right side of the vehicle's heading.
	In the same way as above, the result from the noNIN is shown in Figure \ref{fig:simul1-7} after generating 10 sample nodes for each cluster based on Halton sequences (cost: 2,192).
	The order of task visits in path is 5-4-3-2-1 which is the opposite of that of the omni-directional and the forward sensor, and this is because the detection range is relatively wide.
	The reason the vehicle's path is constructed counterclockwise is that the union of the footprint is relatively wider than when it is clockwise.
	The path generated by the NIN is shown in Figure \ref{fig:simul1-8}, and the efficiency comes from the fact that the chosen sample node of task 2 contains task 3 in the necessarily intersecting region (cost: 2,192 $ \rightarrow $ 1,973).
	The refined path through the NINPR is shown in Figure \ref{fig:simul1-9} (cost: 1,973 $ \rightarrow $ 1,469).

	\subsection{Number of tasks and sample nodes per cluster}\label{subsec:simul2}

	\begin{figure*}[t]
		\centering
		\captionsetup{justification=centering}
		\begin{minipage}{\linewidth}\centering
			\subfloat[Cost, computational time, and number of total nodes versus number of nodes per cluster. The number of task is fixed at 15.]{\label{fig:simulComp2-1}\includegraphics[width=.49\linewidth]{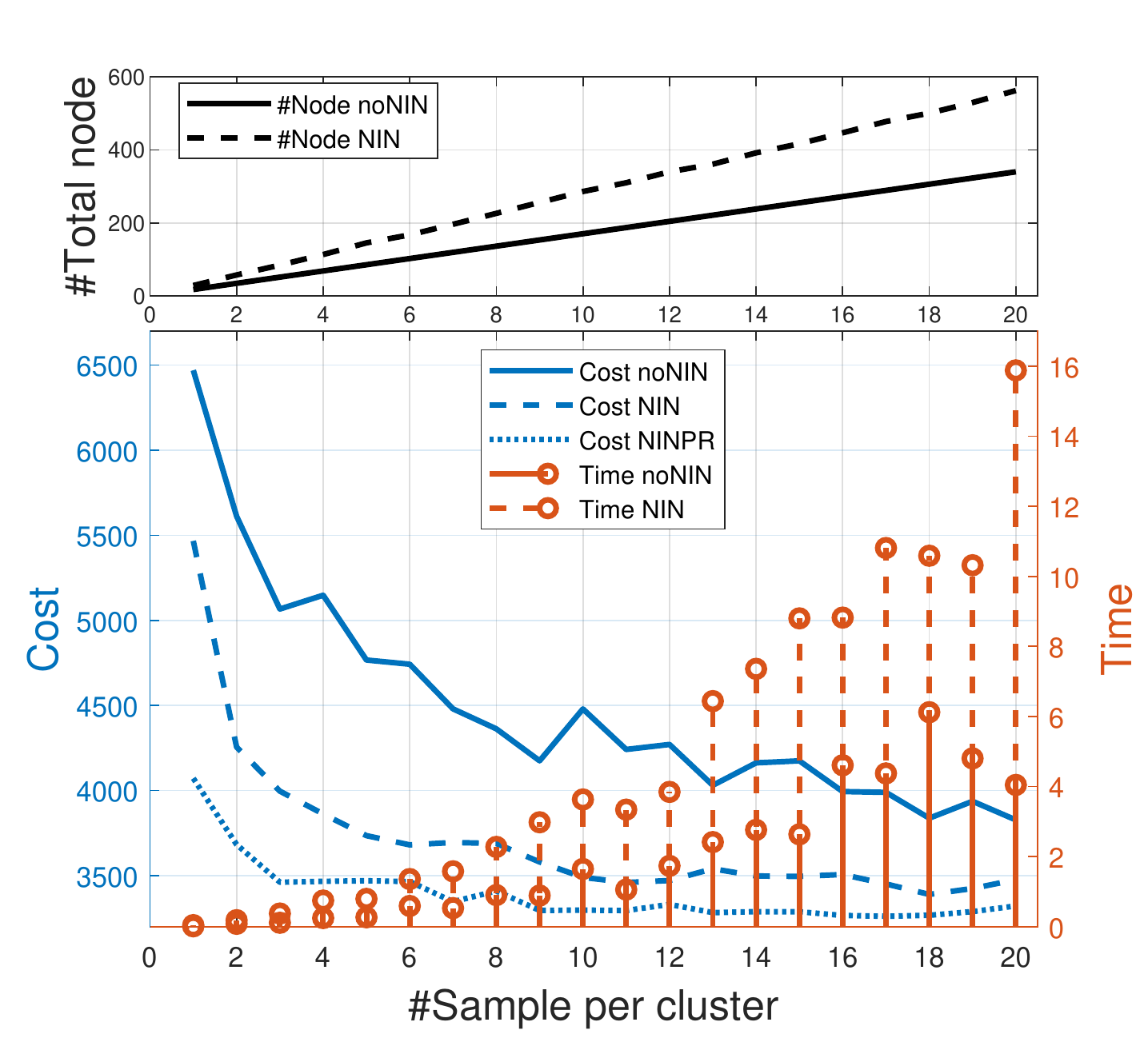}}
			\subfloat[Cost, computational time, and number of total nodes versus number of nodes per cluster. The number of nodes per cluster is fixed at 10.]{\label{fig:simulComp2-2}\includegraphics[width=.49\linewidth]{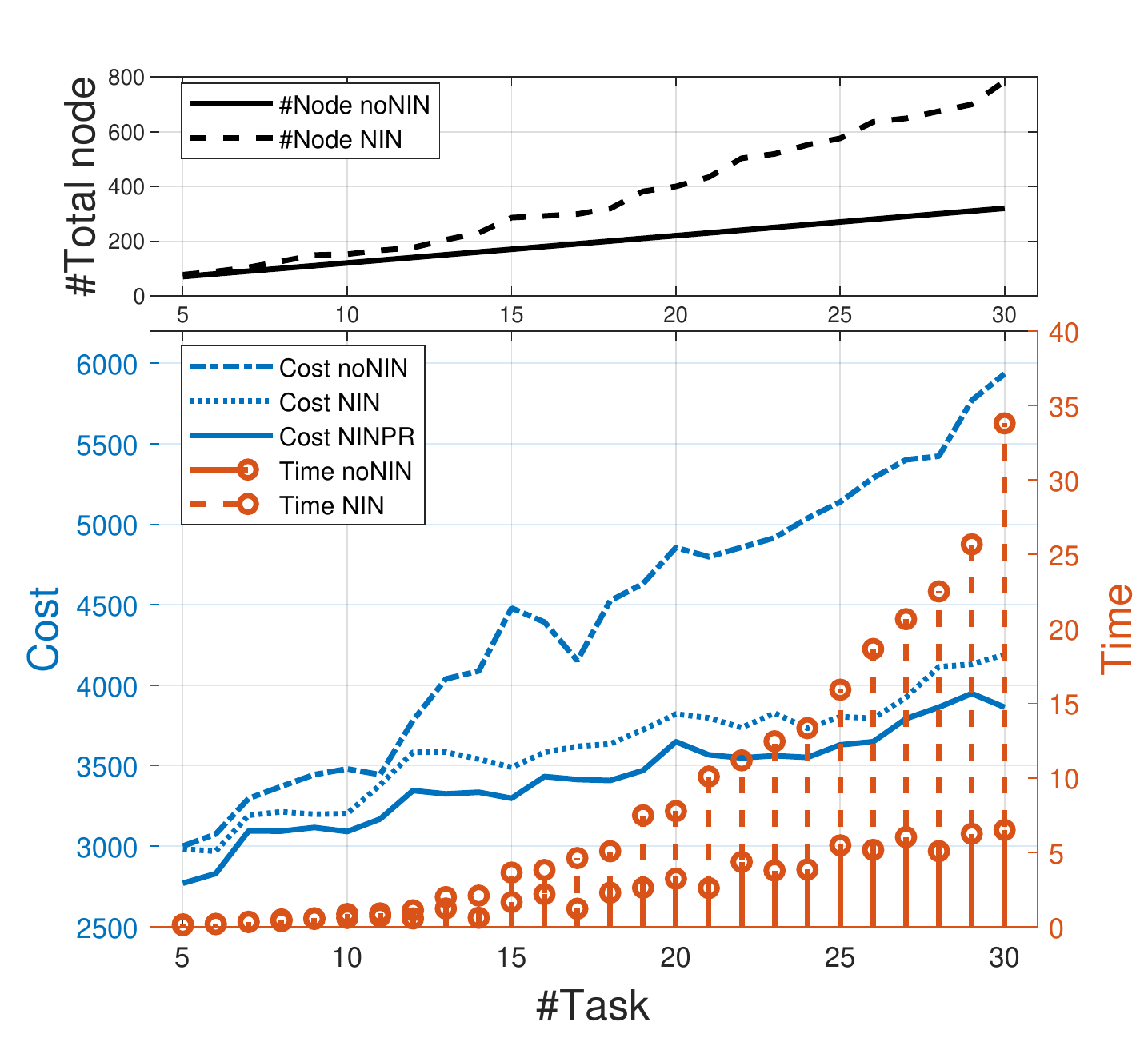}}
		\end{minipage}	
		\caption{Comparison for each method. Single vehicle, altitude: 200$ m $, footprint radius: 115$ m $.}
		\label{fig:simulComp2}
	\end{figure*}
	
	\begin{figure*}[t]
		\centering
		\captionsetup{justification=centering}
		\begin{minipage}{\linewidth}\centering
			\subfloat[Cost, computational time, and number of total nodes versus number of nodes per cluster. The number of task is fixed at 15.]{\label{fig:simulComp3-1}\includegraphics[width=.49\linewidth]{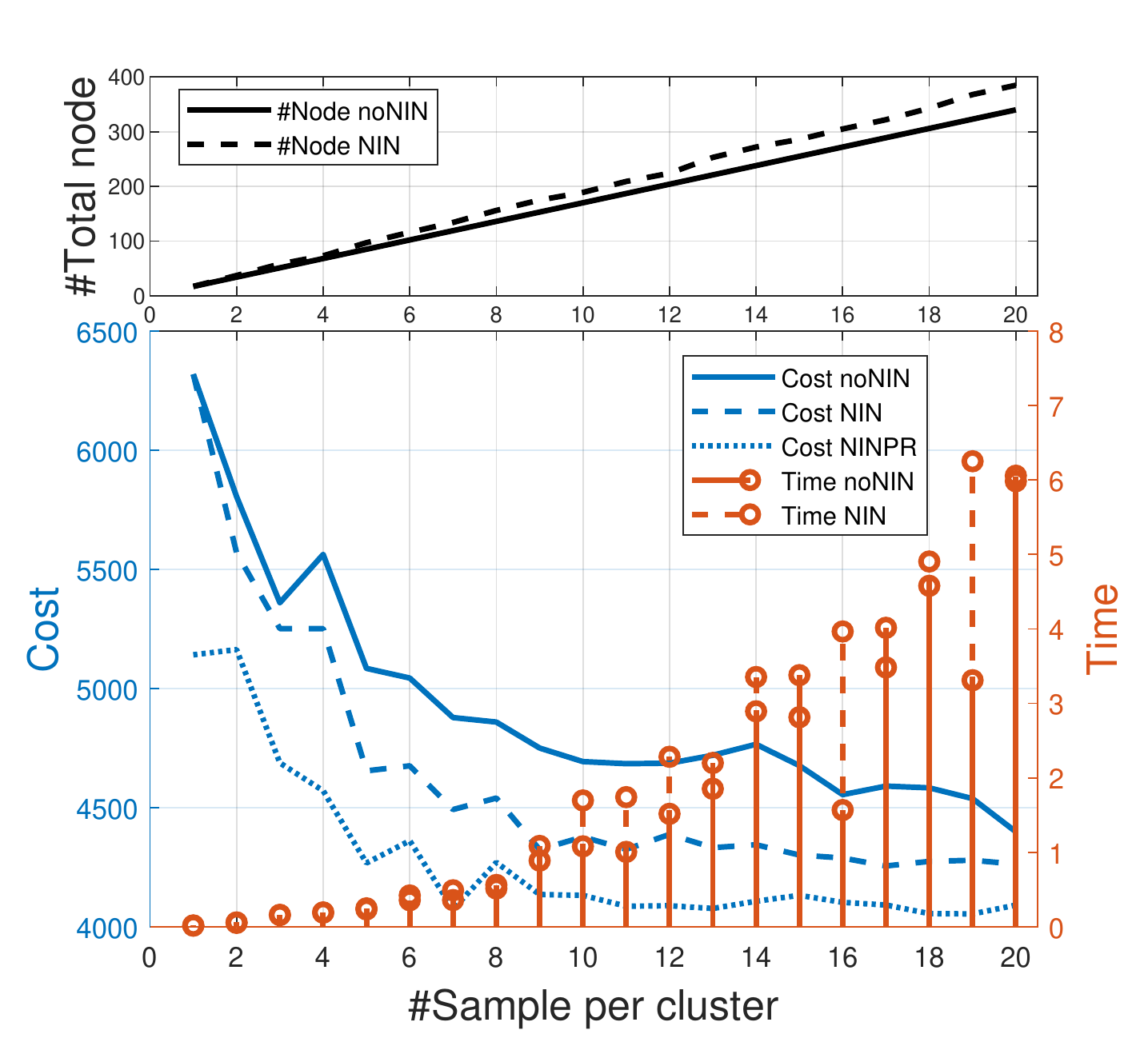}}
			\subfloat[Cost, computational time, and number of total nodes versus number of nodes per cluster. The number of nodes per cluster is fixed at 10.]{\label{fig:simulComp3-2}\includegraphics[width=.49\linewidth]{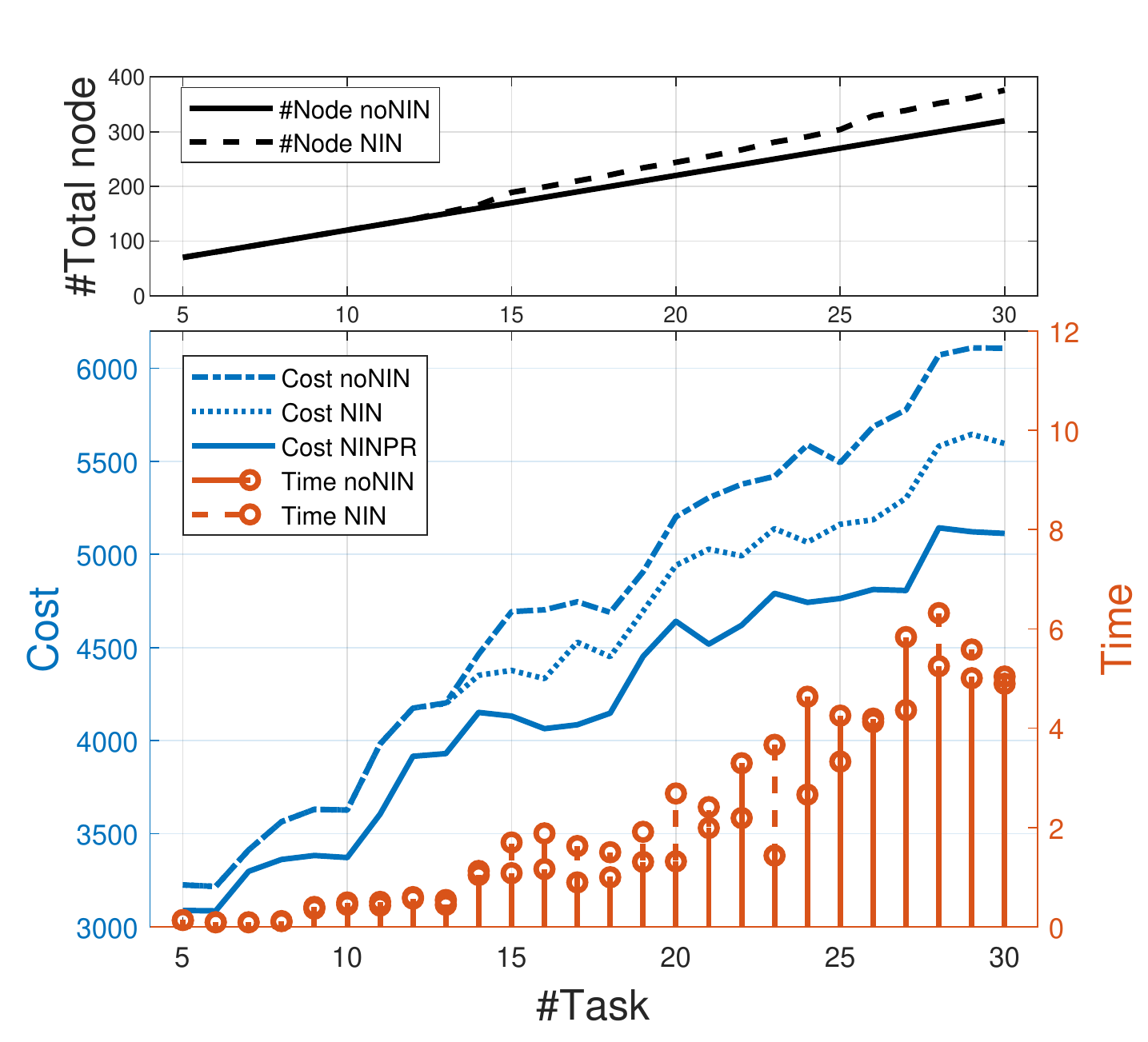}}
		\end{minipage}	
		\caption{Comparison for each method. Single vehicle, altitude: 50$ m $, footprint radius: 29$ m $.}
		\label{fig:simulComp3}
	\end{figure*}

	For each method, the trends of the number of total nodes, solution cost, and the computational time according to the number of sample nodes for each cluster are shown in Figure \ref{fig:simulComp2}.
	To obtain more meaningful characteristics from the results, the parameters of the instances are the same as in Figure \ref{fig:simulComp1} except for the altitude which is changed from 300$ m $ to 200$ m $, and the radius of the footprint is approximately 115$ m $.
	The bottom plot in Figure \ref{fig:simulComp2-1} shows that the cost decreases in the order of the noNIN, NIN, and NINPR given the same number of sample nodes for each cluster, and the cost values from each method decreases as the number of sample nodes increases.
	The cost from the NINPR converges to about 3,200 for 9 or more sample nodes per cluster, and it can be inferred that the path from the NIN has the same task visit sequence.
	Computational time is indicated only for the noNIN and NIN methods because the NINPR is based on the result of the NIN and the time complexity of the path refinement process is $ O(n) $ for $ n $ tasks, and calculation is completed within 1 to 2 seconds.
	The total number of nodes in the instance is proportional to the number of sample nodes in each cluster, and the calculation time of the NIN is larger than the noNIN in every case, which is confirmed in Figure \ref{fig:simulComp2-1}.	
	The virtual nodes are added to the corresponding task clusters if the tasks are included in the necessarily intersecting region of a sample node, so the graph for the NIN has more nodes than the one from the noNIN.
	The time complexity of the LKH heuristic is known as $ O(n^{2.2}) $ where $ n $ is the number of nodes in the graph, and the computational time in the figure roughly follows the above complexity.
	Figure \ref{fig:simulComp2-2} shows the trend of cost and computational time versus the number of tasks for each method while the number of sample nodes for each cluster is fixed as 10.
	The cost values are improved in the order of noNIN, NIN, and NINPR as in the preceding case.
	In addition, the cost efficiency of the noNIN gets worse than that of the other two methods as the number of tasks increases, which suggests a limitation of the method itself.
	For the NIN, the number of total nodes in the instance nonlinearly increases and the computational time grows faster than that of the noNIN method.
	
	In a way similar to above, Figure \ref{fig:simulComp3} shows the results with instance using the same parameters except the operational altitude is changed from 200$ m $ to 50$ m $, and the radius of the footprint is approximately 29$ m $.
	As the sensor footprint becomes smaller and the vehicle needs to get closer to the task, the overall cost value is increased compared to the result in Figure \ref{fig:simulComp2}.
	Moreover, the differences in the cost and computational time between the noNIN and NIN are reduced, since the number of virtual nodes generated through the concept of the necessarily intersecting region is reduced.
	In particular, in Figure \ref{fig:simulComp3-2}, there is no difference in the results of the noNIN and NIN for cases where the number of tasks is 13 or less, because there is no necessarily intersecting node among all the nodes in the graph.

	\subsection{Multiple Vehicles}\label{subsec:simul3}
	
	To find the answer to the last question mentioned above, we use NINPR to find paths for instances of multiple vehicles and analyze the results. 
	The instances with multiple vehicles can be divided into the homogeneous and heterogeneous case in terms of their characteristics.
	Going one step further, the heterogeneity can be divided into structural and functional heterogeneity.
	The vehicles are called structurally heterogeneous if the dynamics or design of the vehicles are different, which leads to differences in motion constraints, speed, or moving cost.
	And a fleet of vehicles is called functionally heterogeneous if different types of sensors are mounted on the vehicles and each sensor can handle specific types of tasks.
	In the following, the results are analyzed in the order of a homogeneous case and heterogeneous case.
	
	\subsubsection{Homogeneous Case}
	
	\begin{figure*}[t]
		\centering
		\captionsetup{justification=centering}
		\begin{minipage}{\linewidth}\centering
			\subfloat[Path output when tasks are close to one of the depots.]{\label{fig:simulComp4-1}\includegraphics[width=.49\linewidth]{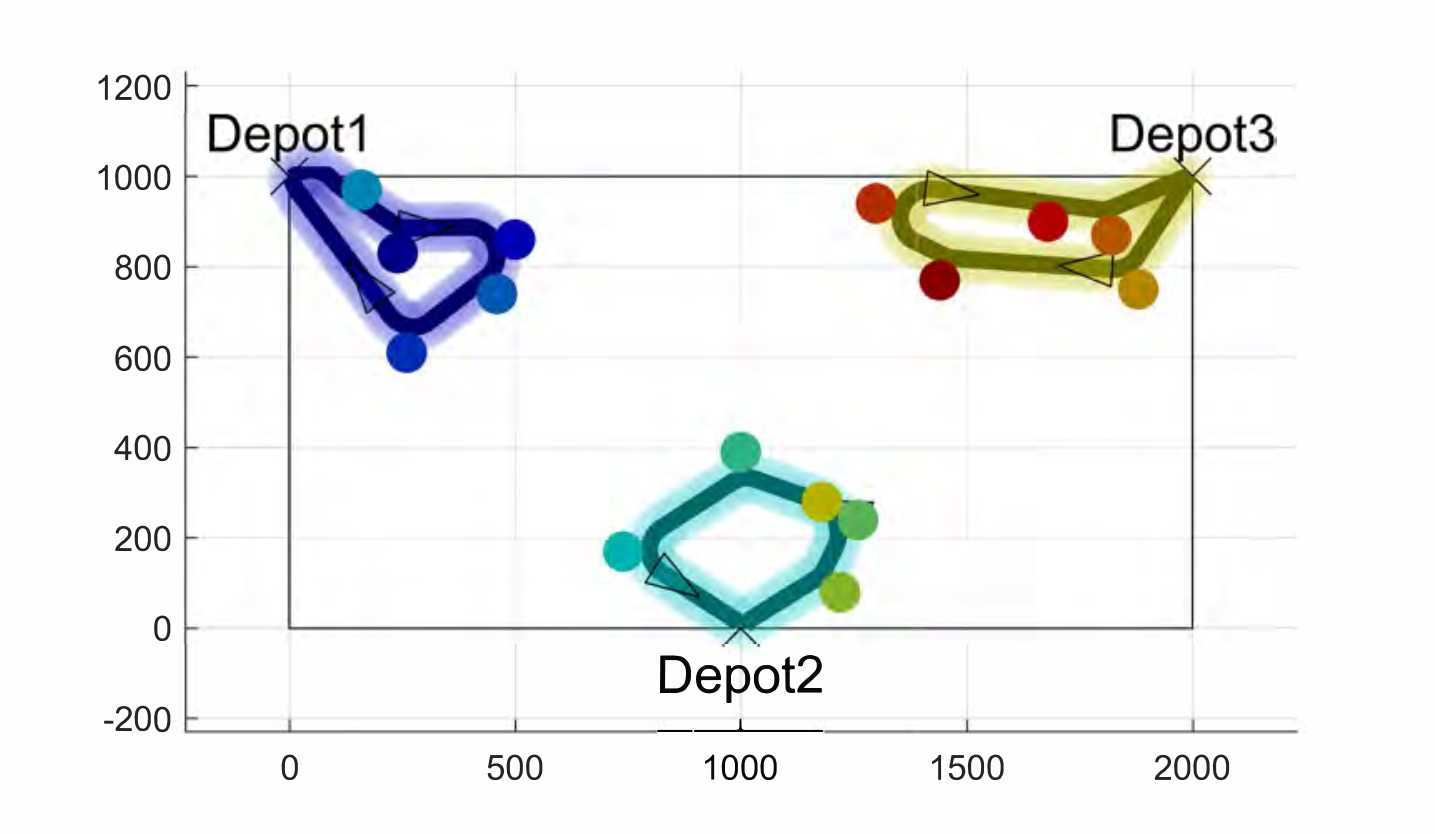}}
			\subfloat[Path output when tasks are dispersed.]{\label{fig:simulComp4-2}\includegraphics[width=.49\linewidth]{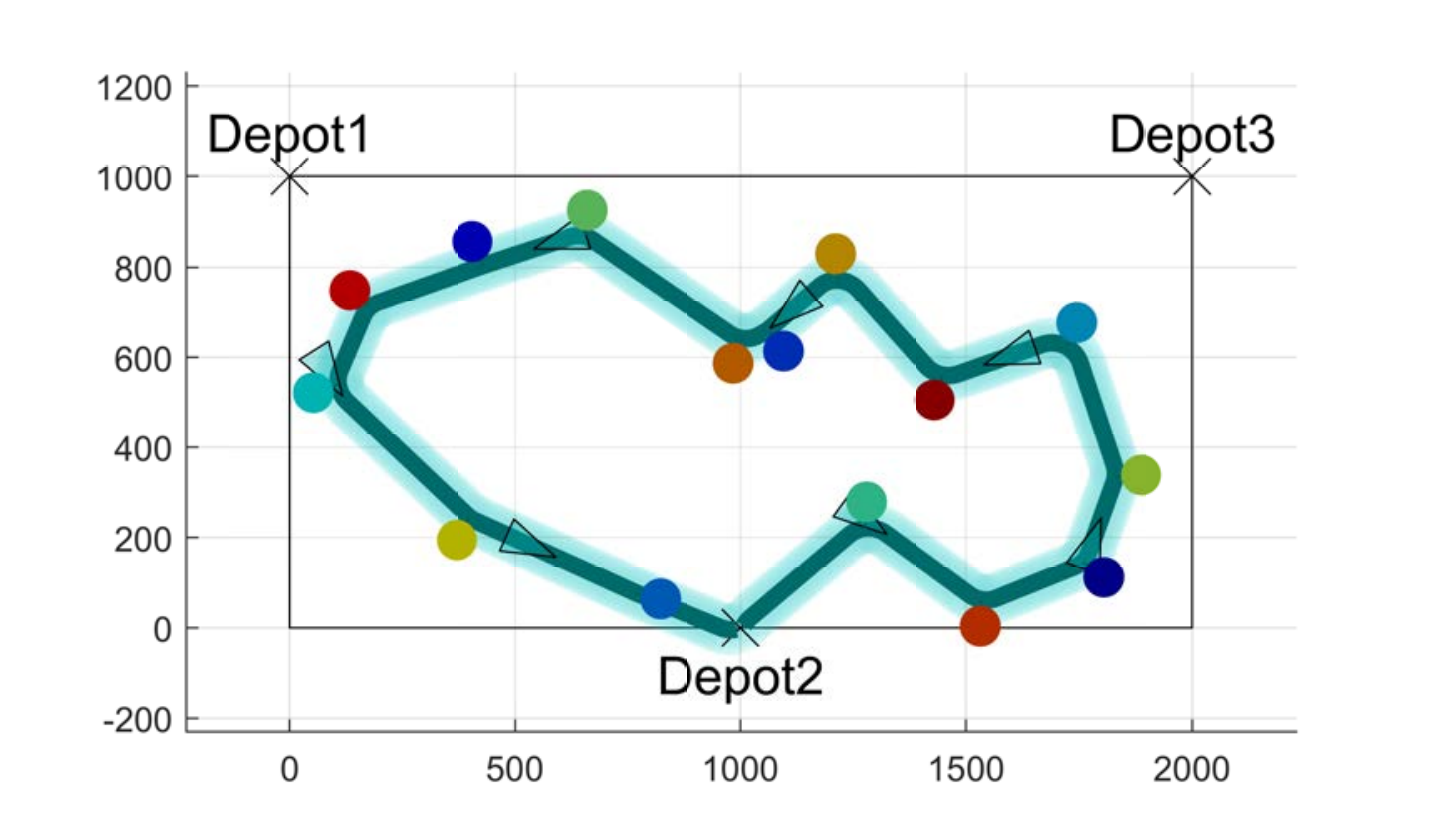}}
		\end{minipage}	
		\caption{Path results from NINPR. Homogeneous vehicles, omni-directional sensor.}
		\label{fig:simulComp4}
	\end{figure*}
	Assume omni-directional sensors are mounted on each vehicle and the radius of the footprint is set identically to be 60$ m $.
	In the instance, the depot locations of each vehicle are (0,1000), (1000,0), and (2000,1000), and the region of interest is $ [0,2000]\times[0,1000] $.
	Figure \ref{fig:simulComp4-1} shows the path generation result when the tasks are close to one of the depots, and Figure \ref{fig:simulComp4-2} shows the result when the tasks are randomly located in the ROI.
	All vehicles participate in fulfilling the mission in Figure \ref{fig:simulComp4-1}, while a path is created only for vehicle 2, which visits every task in Figure \ref{fig:simulComp4-2}.
	Since the objective function of the problem minimizes the sum of the total cost and the upper bound of the tour cost is not specified such as the constrained vehicle routing problem (CVRP), the solution in Figure \ref{fig:simulComp4-2} is considered reasonable considering the problem formulation.
	In a way similar to the result in Figure \ref{fig:simulComp4-2}, all tasks are assigned to only one of the vehicles in most cases when tasks are randomly placed in the ROI.
		
	\subsubsection{Heterogeneous Case}
	
	\begin{figure*}[t]
		\centering
		\captionsetup{justification=centering}
		\begin{minipage}{\linewidth}\centering
			\subfloat[Paths for 30 tasks, functional heterogeneity.]{\label{fig:simulComp5-1}\includegraphics[width=.49\linewidth]{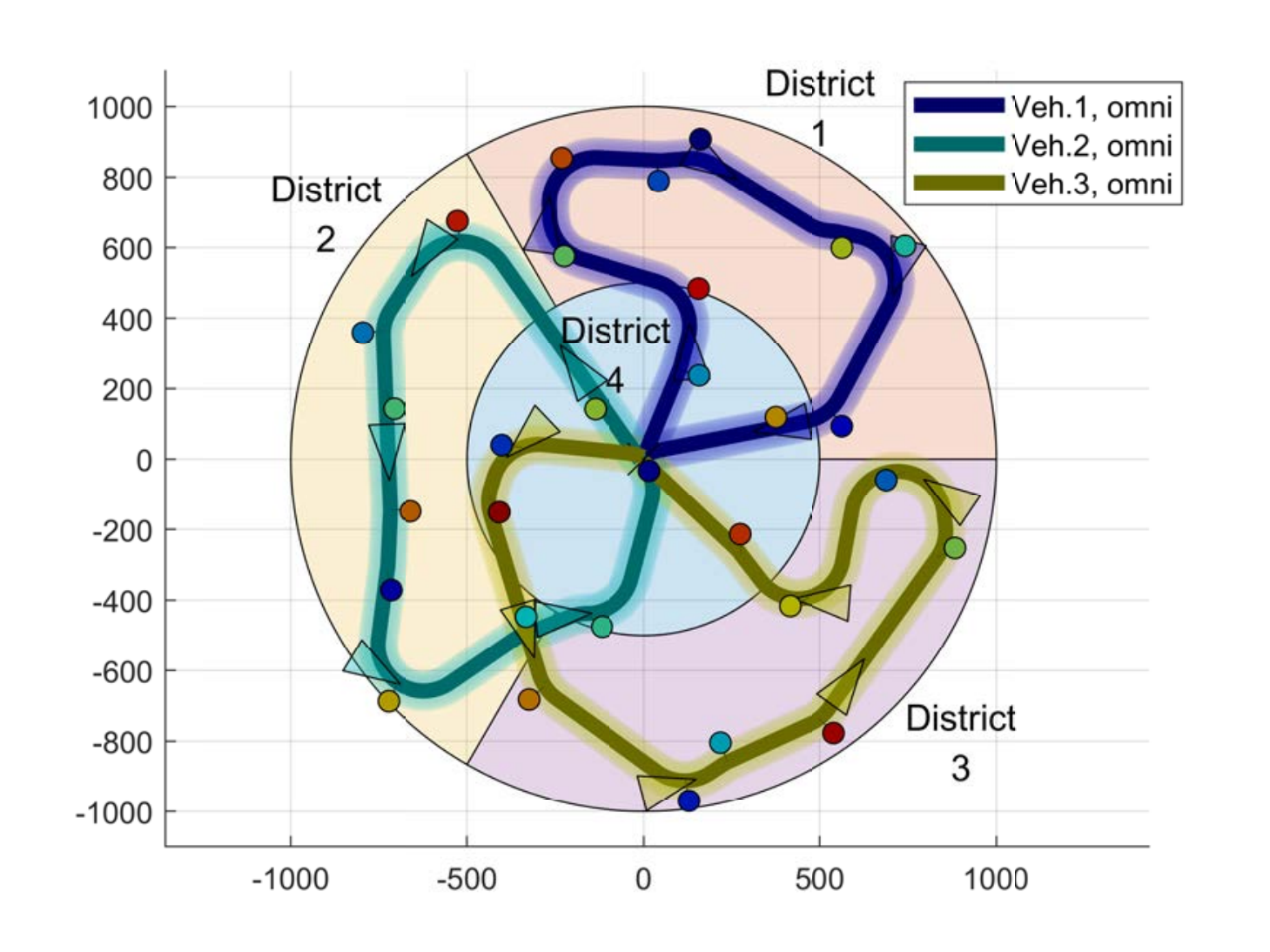}}
			\subfloat[Paths for 100 tasks, functional heterogeneity.]{\label{fig:simulComp5-2}\includegraphics[width=.49\linewidth]{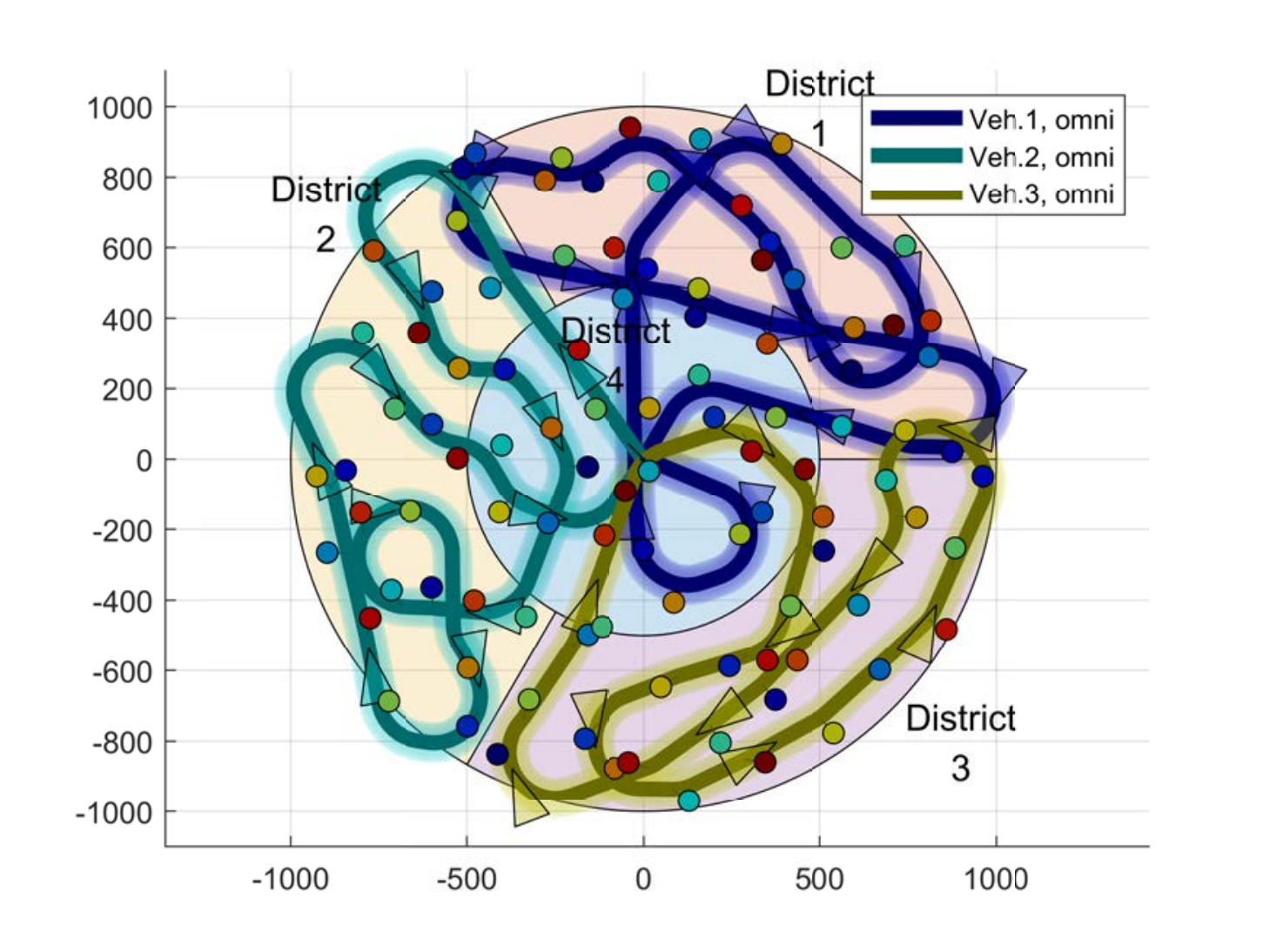}}
		\end{minipage}	
		\begin{minipage}{\linewidth}\centering
			\subfloat[Paths for 30 tasks, functional and structural heterogeneity.]{\label{fig:simulComp5-3}\includegraphics[width=.49\linewidth]{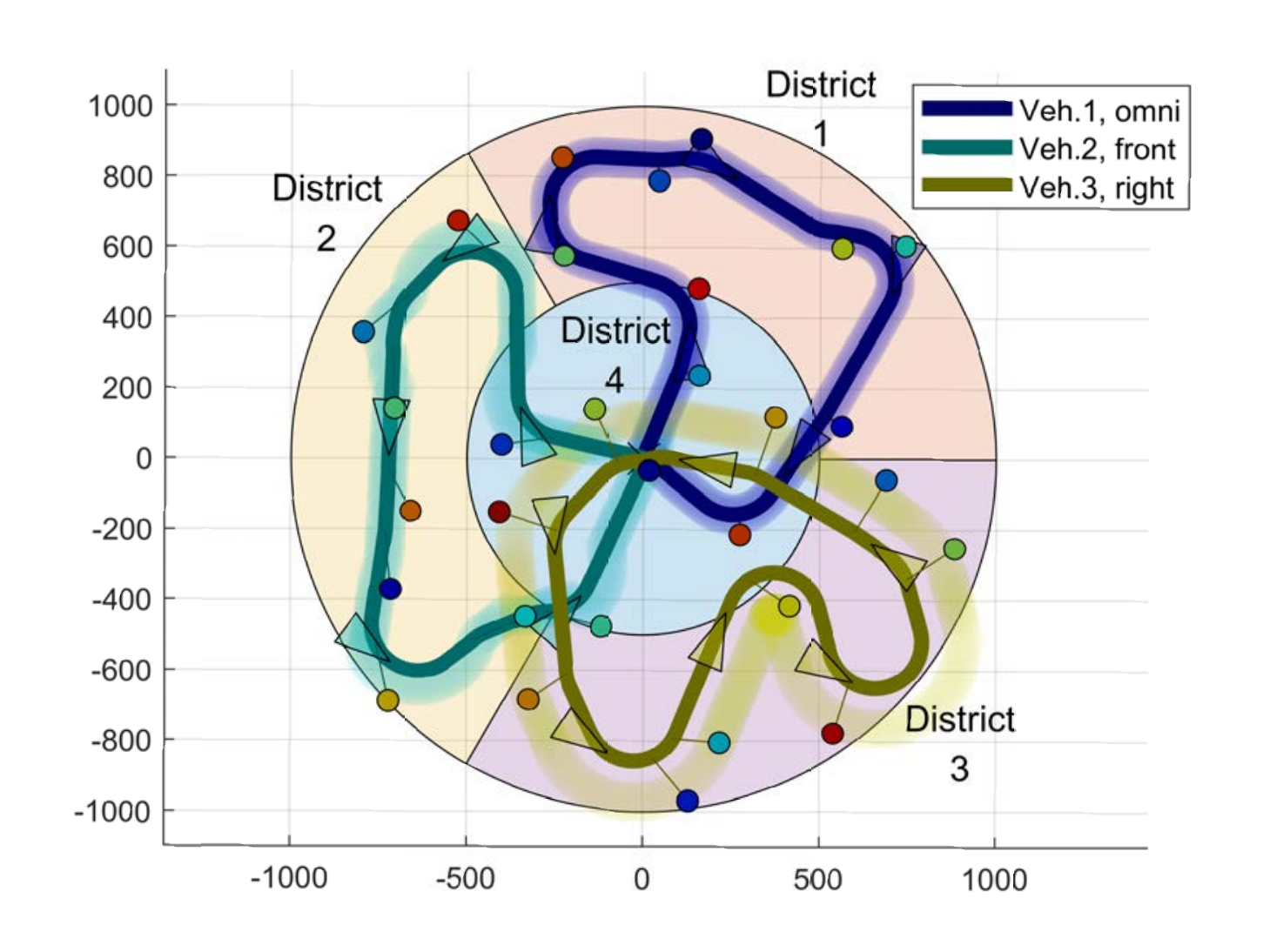}}
			\subfloat[Paths for 100 tasks, functional and structural heterogeneity.]{\label{fig:simulComp5-4}\includegraphics[width=.49\linewidth]{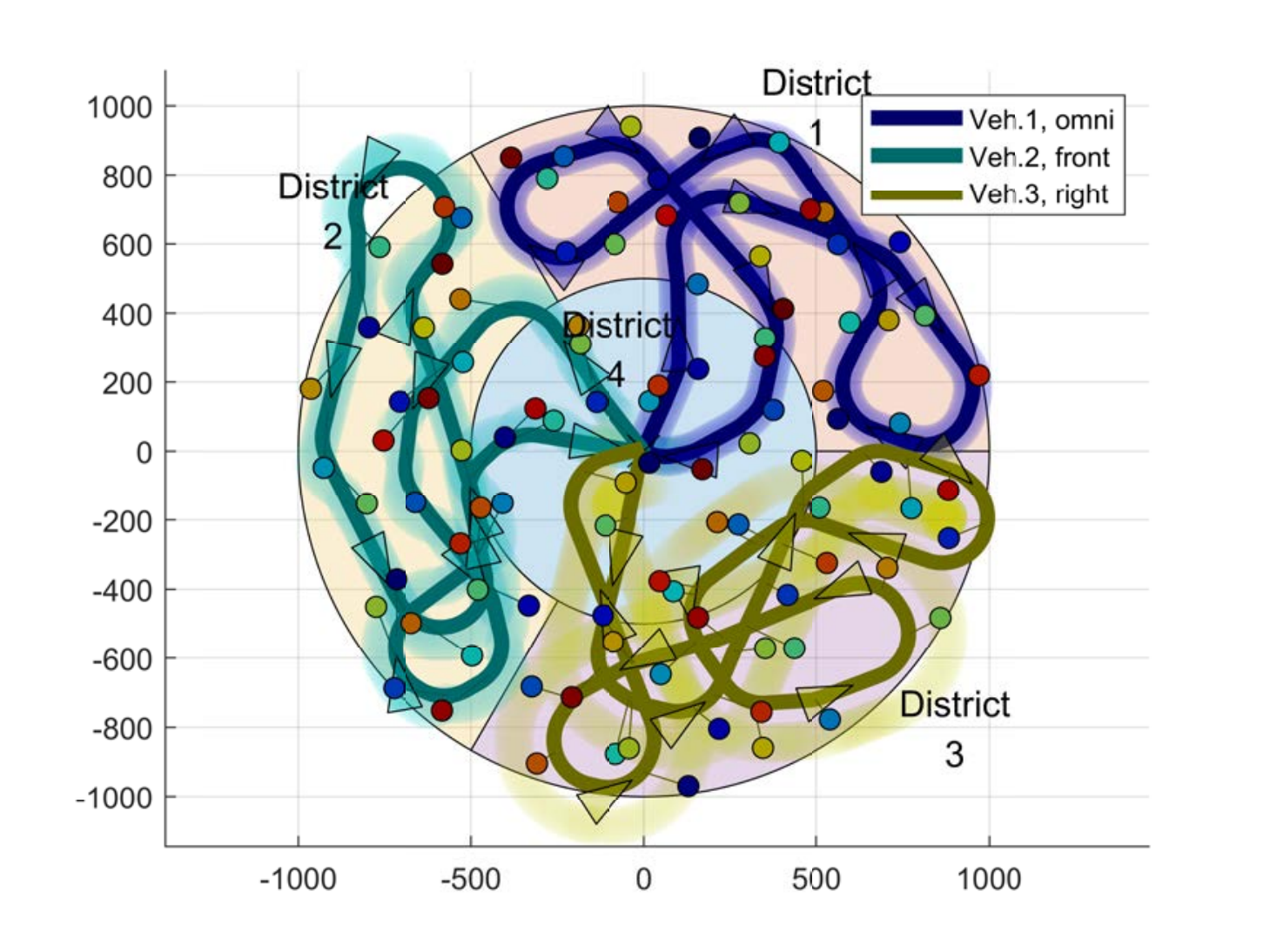}}
		\end{minipage}	
		\caption{Path results from NINPR, heterogeneous case.}
		\label{fig:simulComp5}
	\end{figure*}

	To illustrate functional heterogeneity, we create instances as follows.
	In a way similar to above, three vehicles are under operation and the parameter values for the vehicles remain the same. 
	The changed parts are the locations of depots and the ROI; the depot of each vehicle is equally located at the origin (0,0), and the ROI is a circle with a radius of 1,000$ m $ centered at the origin.
	Tasks are randomly generated in the ROI and each of them belongs to one of the districts based on the location.
	For integer $ i $ from 1 to 3, tasks in district $ i $ can be processed only by vehicle $ i $, and tasks in district 4 are set to be executable by all vehicles.
	A simple and straightforward way to create a functionally heterogeneous instance is that sample nodes are not generated for the tasks that cannot be handled by certain vehicles.
	Figures \ref{fig:simulComp5-1} and \ref{fig:simulComp5-2} show the results of path generation using NINPR when the number of tasks is 30 and 100, respectively.
	Each vehicle travels through the tasks in its district, and the tasks in district 4 are assigned to each vehicle in a way that minimizes the overall cost.

	In order to investigate the characteristics of structural heterogeneity, it is assumed that the direction of the sensor is mounted in the order of omni-directional, forward, and rightward for vehicles 1 to 3.
	The path generation results are shown in Figures \ref{fig:simulComp5-3} and \ref{fig:simulComp5-4} when the number of tasks is 30 and 100, respectively.
	The results confirm that the paths are generated in a way that satisfies the constraints for each vehicle's sensor.

	\section{Conclusion}\label{sec:conc}
	
	This paper described a Mixed-Integer Linear Programming formulation for the multi-vehicle path generation problem as a variant of the traveling salesman problem, which is called the Generalized Heterogeneous Multiple Depot Traveling Salesmen Problem (GHMDATSP).	
	The concept called the necessarily intersecting neighborhood was extended from the omni-directional sensor model to more general cases and was used to increase the performance in terms of path length, especially when the tasks are densely located.
	The key part of the proposed procedure is the transformation method which transforms the GHMDATSP graph to the form of the ATSP to exploit the LKH heuristic, which is the state-of-the-art ATSP solver.
	Local optimization helped improve the performance of the solutions that come from the procedure described in Section \ref{sec:proc}.
	Numerical studies showed the results of various observations of the characteristics of the GHMDATSP with Dubins vehicles.
	Future work will focus on handling different objective functions and more constraints related to capacity or accessible time, and additionaly on development of effective algorithms for larger-scale problems.

	\section*{Appendix - Necessarily intersecting region, frontward sensor direction}\label{app:nir}
	
	\begin{figure}[]
		\centering
		\includegraphics[height=6cm]{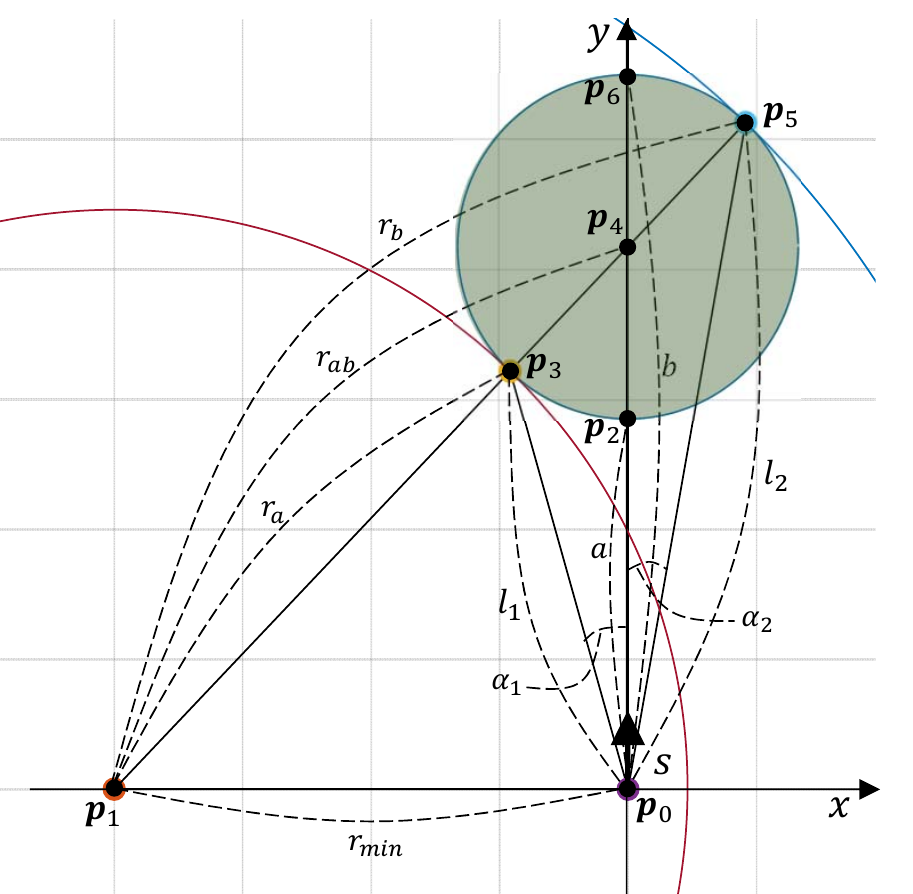}
		\caption{Necessarily intersecting nodes with a directional sensor, facing front-side.}\label{fig:ninFrontSol1}
	\end{figure}	

	Given $ a, b, r_{min}, $ and a sample node $ s $ located on $ \textbf{p}_0 $ (the arrow starting from $ \textbf{p}_0 $), we need to calculate the lengths of line segments $ \overline{\textbf{p}_0 \textbf{p}_3} $ and $ \overline{\textbf{p}_0 \textbf{p}_5} $ which are $ l_1 $ and $ l_2 $, and the angles $ \angle \textbf{p}_3 \textbf{p}_0 \textbf{p}_4 $ and $ \angle \textbf{p}_5 \textbf{p}_0 \textbf{p}_4 $ which are $ \alpha_1, \alpha_2 $ to get an equation of the colored area in Figure \ref{fig:ninFront}.
	For ease of computation, let us assume $ \textbf{p}_0 = (0,0) $ and a heading of sample node $ s $ is directly towards y-axis $ \textbf{p}_1 = (-r_{min},0) $, .
	Before obtaining the equations of the above parameters, let $ r_a $, $ r_{ab} $, and $ r_b $ the length of $ \overline{\textbf{p}_1 \textbf{p}_3} $, $ \overline{\textbf{p}_1 \textbf{p}_4} $, and $ \overline{\textbf{p}_1 \textbf{p}_5} $, respectively.
	Since $ \triangle \textbf{p}_0 \textbf{p}_1 \textbf{p}_4 $ is a right triangle and the length of $ \overline{\textbf{p}_0 \textbf{p}_4} $ is $ \dfrac{a+b}{2} $, $ r_{ab} = \sqrt{r_{min}^2 + \left(\dfrac{a+b}{2}\right)^2} $.
	Using $ r_{ab} $, $ r_a = r_{ab} - \left( \dfrac{b-a}{2} \right) $ and $ r_b = r_{ab} + \left( \dfrac{b-a}{2} \right) $.
	We solve the system of equations in Eq. \eqnref{eq:appSys1} to get the coordinates of point $ \textbf{p}_3 = \left( x_{p_3}, y_{p_3} \right) $.
	\begin{equation}\label{eq:appSys1}
		\left\{\begin{matrix}
			\left( x - x_{p_1} \right)^2 + \left( y - y_{p_1} \right)^2 = r_a^2 \\
			\left( x - x_{p_4} \right)^2 + \left( y - y_{p_4} \right)^2 = \left( \dfrac{b-a}{2} \right)^2
		\end{matrix}\right.
	\end{equation}
	Then the coordinates are as follows:
	\begin{align}
		\left( x_{p_3}, y_{p_3} \right) = \left( \dfrac{r(a-b)}{\sqrt{*}}, \dfrac{(a+b)\left( (a-b)\sqrt{*}+* \right)}{2*} \right) \\
		\textrm{where } * = a^2+2ab+b^2+4r^2. \nonumber
	\end{align}
	From $ \left( x_{p_3}, y_{p_3} \right) $ we obtain the expressions of $ \alpha_1 $ and $ l_1 $ using $ a,b, $ and $ r $ easily by using $ tan^{-1} $ operator and the Pythagorean theorem.
	\begin{eqnarray}
		\alpha_1 = & tan^{-1}\left( \dfrac{2\sqrt{X} \cdot r(b-a)}{(a+b)\left( X+(a-b)\sqrt{X} \right)} \right) \\
		\l_1 = & \sqrt{ \dfrac{r^2(a-b)^2}{*} + \left( \dfrac{(a+b)\left( (a-b)\sqrt{*}+* \right)}{2*} \right)^2 }
	\end{eqnarray}

	In a similar way, we can obtainthe expressions of $ \alpha_2 $ and $ l_2 $ using $ a,b, $ and $ r $ starting from the coordinates of point $ \textbf{p}_5 = \left( x_{p_5}, y_{p_5} \right) $.
	The coordinates are obtained with the following system of equations:
	\begin{equation}\label{eq:appSys2}
		\left\{\begin{matrix}
			\left( x - x_{p_1} \right)^2 + \left( y - y_{p_1} \right)^2 = r_b^2 \\
			\left( x - x_{p_4} \right)^2 + \left( y - y_{p_4} \right)^2 = \left( \dfrac{b-a}{2} \right)^2
		\end{matrix}\right.
	\end{equation}
	Then the coordinates are as follows:
	\begin{equation}
		\left( x_{p_5}, y_{p_5} \right) = \left( \dfrac{r(a-b)}{\sqrt{*}}, \dfrac{(a+b)\left( (a-b)\sqrt{*}+* \right)}{2*} \right)
	\end{equation}
	Finally $ \alpha_2 $ and $ l_2 $ are as follows:
	\begin{eqnarray}
		\alpha_1 = & tan^{-1}\left( \dfrac{2\sqrt{X} \cdot r(b-a)}{(a+b)\left( X+(b-a)\sqrt{X} \right)} \right) \\
		\l_1 = & \sqrt{ \dfrac{r^2(a-b)^2}{*} + \left( \dfrac{(a+b)\left( (b-a)\sqrt{*}+* \right)}{2*} \right)^2 }
	\end{eqnarray}

	\begin{figure}[]
		\centering
		\includegraphics[height=6cm]{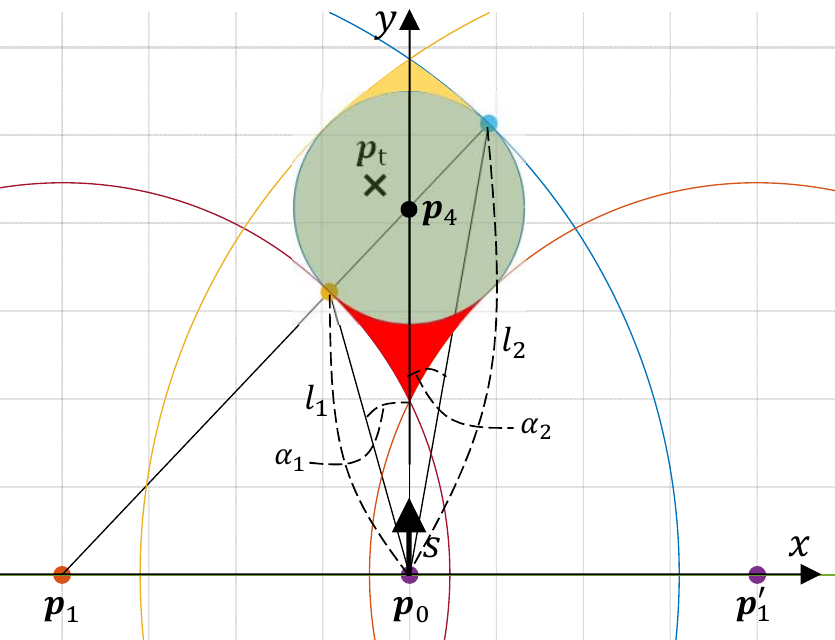}
		\caption{Necessarily intersecting nodes with a directional sensor, facing front-side.}\label{fig:ninFrontSol2}
	\end{figure}	
	From the above results, the necessarily intersecting region by sample node $ s $ is a region which satisfies one of the inequalities in Eq. \eqref{eq:ineq}.
	\begin{subequations}\begin{align}
		r_a < \left|\left| \textbf{p}_t - \textbf{p}_0 \right|\right| < l_1, & \qquad |\angle(\textbf{p}_t-\textbf{p}_0) - \angle(\textbf{p}_4-\textbf{p}_0) | < \alpha_1 \label{eq:ineq1}\\
		l_2 < \left|\left| \textbf{p}_t - \textbf{p}_0 \right|\right| < r_b, & \qquad |\angle(\textbf{p}_t-\textbf{p}_0) - \angle(\textbf{p}_4-\textbf{p}_0) | < \alpha_2 \label{eq:ineq2}\\
		\left|\left| \textbf{p}_t - \textbf{p}_4 \right|\right| < r_{sen} \span \label{eq:ineq3}
	\end{align}\label{eq:ineq}\end{subequations}
	In Figure \ref{fig:ninFrontSol2}, Eqs. \eqref{eq:ineq1} and \eqref{eq:ineq2} correspond to the red-colored and yellow-colored areas, respectively, and Eq. \eqref{eq:ineq3} corresponds to the green-colored area which is a footprint of the sensor when the vehicle is at $ \textbf{p}_0 $ and heading towards the direction of the positive y-axis.

	
	\bibliographystyle{aiaa}
	\bibliography{bibtex_database}
	
\end{document}